\documentclass{article} 
\PassOptionsToPackage{sort&compress}{natbib}
\usepackage{iclr2027_conference,times}
\setcitestyle{numbers,square,comma}

\usepackage{amsmath,amsfonts,bm}

\def\eqref#1{equation~\ref{#1}}

\def\1{\bm{1}}

\DeclareMathAlphabet{\mathsfit}{\encodingdefault}{\sfdefault}{m}{sl}
\SetMathAlphabet{\mathsfit}{bold}{\encodingdefault}{\sfdefault}{bx}{n}

\usepackage{hyperref}
\usepackage{url}

\usepackage[utf8]{inputenc}   
\usepackage[T1]{fontenc}      
\usepackage{booktabs}         
\usepackage{amsfonts}         
\usepackage{nicefrac}         
\usepackage{microtype}        
\usepackage[table,dvipsnames]{xcolor} 
\usepackage{wrapfig}
\usepackage{subcaption}
\usepackage{amsthm}
\usepackage{float}
\usepackage{tabularx}
\usepackage{makecell}         
\usepackage{pifont}           
\usepackage{siunitx}
\usepackage{adjustbox}
\usepackage{graphicx}
\usepackage{subcaption}
\usepackage{caption}
\usepackage{listings}
\usepackage{amssymb}
\usepackage{comment}
\usepackage{xspace}
\usepackage{tikz}
\usetikzlibrary{calc}
\usepackage{algorithmicx}
\usepackage[ruled,vlined]{algorithm2e}
\usepackage{enumitem}
\usepackage{multirow}
\usepackage{mathtools}
\definecolor{best}{RGB}{180,220,140}
\definecolor{second}{RGB}{220,240,200}
\definecolor{royalgreen}{RGB}{0,102,51}
\definecolor{royalblue}{RGB}{0,53,148}
\definecolor{royalpurple}{RGB}{106,13,173}
\definecolor{royalorange}{RGB}{204,85,0}
\definecolor{royalred}{RGB}{178,24,43}
\definecolor{royalgold}{RGB}{184,134,11}
\definecolor{greenblue}{RGB}{0,128,140}

\definecolor{cAutomorphism}{RGB}{213,94,0}
\definecolor{cOrbit}{RGB}{0,114,178}
\definecolor{cSubgraphOrbit}{RGB}{0,158,115}
\definecolor{cEdgeOrbit}{RGB}{204,121,167}
\definecolor{cEdgeEquiv}{RGB}{230,159,0}
\definecolor{cEAR}{RGB}{86,180,233}

\colorlet{red}{black}

\usepackage[nameinlink,capitalize,noabbrev]{cleveref}
\newtheorem{definition}{Definition}[section]
\newtheorem{measure}{Measure}[section]
\newtheorem{Proposition}{Proposition}[section]

\newtheorem{theorem}{Theorem}[section]
\newtheorem{lemma}{Lemma}[section]
\crefname{figure}{Fig.}{Figs.}
\Crefname{figure}{Fig.}{Figs.}
\crefname{section}{Sec.}{Secs.}
\Crefname{section}{Sec.}{Secs.}
\crefname{subsection}{Sec.}{Secs.}
\Crefname{subsection}{Sec.}{Secs.}
\crefname{equation}{Eq.}{Eqs.}
\Crefname{equation}{Eq.}{Eqs.}
\crefname{theorem}{Thm.}{Thms.}
\Crefname{theorem}{Thm.}{Thms.}
\crefname{lemma}{Lem.}{Lems.}
\Crefname{lemma}{Lem.}{Lems.}
\crefname{appendix}{App.}{Apps.}
\Crefname{appendix}{App.}{Apps.}
\crefname{definition}{Def.}{Defs.}
\Crefname{definition}{Def.}{Defs.}
\Crefname{measure}{Measure}{Measures.}
\Crefname{measure}{Measure}{Measures.}
\crefname{algorithm}{Algo.}{Algos.}
\Crefname{algorithm}{Algo.}{Algos.}
\crefformat{figure}{#2Fig.~#1#3}
\Crefformat{figure}{#2Fig.~#1#3}
\crefname{table}{Tab.}{Tabs.}
\Crefname{table}{Tab.}{Tabs.}
\crefformat{table}{#2Tab.~#1#3}
\Crefformat{table}{#2Tab.~#1#3}
\crefformat{section}{#2Sec.~#1#3}
\Crefformat{section}{#2Sec.~#1#3}

\newcommand{\addcomment}[2]{\textcolor{#1}{#2}}
\newcommand{\chen}[1]{\addcomment{teal}{#1}}

\definecolor{purple2}{RGB}{153,0,153} %
\definecolor{green2}{RGB}{0,153,0} %

\newcommand{\stepone}{\textbf{{(S1)}}\xspace}
\newcommand{\steptwo}{\textbf{{(S2)}}\xspace}
\newcommand{\stepthree}{\textbf{{(S3)}}\xspace}

\definecolor{darkgreen}{RGB}{0,153,0}
\newcommand{\cmark}{{\color{darkgreen}\ding{51}}}
\newcommand{\xmark}{\ding{55}}

\newcommand{\V}[1]{\mathbf{#1}}

\newcommand{\graph}{\mathcal{H}}
\newcommand{\vertexSet}{\mathcal{V}}
\newcommand{\edgeSet}{\mathcal{E}}

\newcommand{\matA}{\mathbf{A}}

\newcommand{\matX}{\mathbf{X}}

\newcommand{\matW}{\mathbf{W}}

\newcommand{\matL}{\mathbf{L}}

\newcommand{\neighNoSelfLoop}{\bar{N}}

\newcommand{\decoder}{\textsc{decoder}}
\newcommand{\metrics}{\textnormal{automorphism}\xspace}
\newcommand{\Designone}{Automorphism-aware Dropout}
\newcommand{\method}{\textsc{EO-GNN}\xspace}
\newcommand{\Designtwo}{Subgraph Orbit-Biased Aggregation}
\newcommand{\EAR}{EAR}
\newcommand{\NAP}{node automorphism problem}
\newcommand{\fullname}{\textsc{Edge-Orbit Equivariant Graph Neural Network}\xspace}
\newcommand{\multiset}[1]{\left\{\!\left\{#1\right\}\!\right\}}
\definecolor{modifyred}{RGB}{224,0,0}
\newcommand{\modify}[1]{#1}

\title{Edge-Level Automorphism in GNNs: A Quantitative Framework
       and Effective Designs for Link Prediction}
\author{Chen Shao \\
Karlsruhe Institute of Technology \\
Karlsruhe, Germany \\
\texttt{chen.shao2@kit.edu}
\And
Donald Loveland \\
Department of Computer Science \\
University of Michigan \\
Ann Arbor, MI, USA \\
\texttt{loveland@umich.edu}
\And
Tobias Käfer \\
Karlsruhe Institute of Technology \\
Karlsruhe, Germany \\
\texttt{tobias.kaefer@kit.edu}
\And
Danai Koutra\thanks{Corresponding author.} \\
Computer Science \& Engineering \\
University of Michigan \\
Ann Arbor, MI, USA \\
\texttt{dkoutra@umich.edu}
}

\iclrfinalcopy 
\begin{document}

\maketitle

\begin{abstract}
Graph Neural Networks (GNNs) are effective for learning node and link embeddings
through permutation-equivariant aggregation. However, standard GNNs collapse
automorphic nodes, i.e., those with identical structural roles (or orbits), into
indistinguishable representations, leading to the \emph{\NAP}. This collapse
limits their expressive power and degrades link prediction performance.
Existing approaches to characterize GNN expressiveness rely primarily on
Weisfeiler-Lehman (WL) analyses, but these methods are typically qualitative
and often misaligned with empirical results. To address this gap, we begin by
introducing a novel quantitative framework to assess GNN expressiveness for link
prediction. We first formalize edge-level automorphism through \emph{edge
orbits}, which capture the set of structural role pairs for nodes that share a
link. Then, we introduce the edge automorphism ratio (EAR), a scalar metric that
quantifies a GNN's ability to distinguish links in a given graph. We empirically
demonstrate that EAR correlates strongly with performance, validating its
practical benefit. Building on this insight, we design \fullname (\method), a
GNN architecture that addresses automorphism collapse while preserving
equivariance and incurring minimal computational overhead. \method accomplishes
this through two core designs combined with WL-based node hashes: (i)~automorphism-aware
dropouts and (ii)~subgraph orbit-biased aggregation. Empirical evaluations on
synthetic and real graphs show improvements of up to~\num{42.36}\% and~\num{28.44}\%,
respectively, in predicting links in scenarios with high automorphism.
\end{abstract}

\section{Introduction}

Link prediction (LP) estimates the likelihood of connections between node pairs in a graph and is fundamental to applications such as recommendation systems, protein interaction analysis, and knowledge graph completion~\citep{10.1145/1065385.1065415,10.1093/nar/gky1131}. 
Most GNN-based LP methods follow an encoder-decoder paradigm, where a permutation-equivariant GNN produces node embeddings that are subsequently decoded into link probabilities~\citep{zhang_link_2018,Zhang2020LabelingTA}. While permutation equivariance ensures invariance to node ordering, it also causes structurally equivalent (automorphic) nodes to receive identical representations, even when they play distinct semantic and positional roles in the graph~\citep{bloem2020probabilistic,morris2024orbitequivariant,Ma2024ACP}. \begin{wrapfigure}{r}{0.52\linewidth}
\vspace{-0.42cm}
\centering
\includegraphics[width=0.9\linewidth]{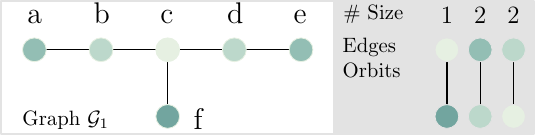}
\vspace{-0.22cm}
\caption{A standard GNN assigns identical scores to the automorphic candidate links $(a,b)$ and $(a,d)$, 
illustrating the Node Automorphism Problem (NAP).}
\label{fig:motive}
\vspace{-0.5cm}
\end{wrapfigure}
Informally, two nodes are \emph{structurally equivalent (automorphic)} if some symmetry of the graph, i.e., a relabeling of vertices that preserves every edge, maps one node onto the other; because a GNN's output cannot depend on how nodes are labeled, such nodes are architecturally indistinguishable, regardless of their intended semantic role, as illustrated in \Cref{fig:motive}. We refer to this as the \emph{Node Automorphism Problem (NAP)}, which leads to degraded link discriminability~\citep{chamberlain_graph_2023}.

Existing expressiveness analyses for GNNs primarily rely on the vertex-level Weisfeiler--Lehman (WL) hierarchy, which characterizes how well GNNs distinguish nodes using graph isomorphism tests \citep{zhang2024weisfeilerlehmanquantitativeframeworkgnn}.
However, existing expressiveness frameworks suffer from several fundamental limitations. 
\textbf{(L1) Level mismatch}: WL-based analyses primarily operate at the node or subgraph level, while link prediction inherently requires fine-grained edge-centric reasoning over complex multi-node interactions, resulting in a substantial gap between theoretical expressivity characterization and practical link discrimination. 
\textbf{(L2) Lack of quantitative characterization}: Current frameworks remain largely qualitative, providing little insight into the extent of structural automorphism or its direct impact on link prediction performance, thereby limiting systematic analysis of model behavior under symmetry. 
\textbf{(L3) Missing empirical foundations}: Existing studies lack synthetic benchmarks with controllable automorphism levels, leaving theoretical claims insufficiently validated and preventing rigorous, reproducible investigation of symmetry-induced ambiguity in link prediction.

To bridge these gaps, we establish a principled edge-level framework for analyzing link prediction under structural automorphism. 
First, we introduce the notion of \textit{Weisfeiler Lehman-induced Edge Orbit} to explicitly characterize edge ambiguity induced by automorphisms. 
Building upon this formulation, we propose the \textit{Edge Automorphism Ratio} (\EAR), a rigorous quantitative metric that measures the proportion of theoretically indistinguishable edges in a graph. We address L3 by constructing synthetic benchmarks  in which \EAR\xspace is an explicit, tunable generation parameter rather than an incidental property of a fixed graph. Sweeping this parameter lets us test whether
performance degradation is attributable to automorphism itself (\Cref{subsec:synthetic benchmark}).Guided by these empirical findings, we develop \fullname\ (\method), an automorphism-aware GNN architecture designed to preserve equivariance while substantially improving edge discriminability in highly symmetric graphs. 
Together, our contributions are summarized as follows:

\color{black}
\begin{itemize}[leftmargin=*,itemsep=0pt,parsep=0pt,topsep=0pt]

\item \textbf{Quantitative Theoretical Framework:} We introduce a structural concept, the \emph{WL-induced Edge Orbit}, that refines discrimination from the vertex to the edge level. Building on it, we characterize a GNN's expressive power through its distinguishable edge set $\modify{\mathcal{E}^{\mathcal{M}}}$, letting different models be compared directly by set inclusion and set difference. We summarize this set as a single scalar, the \emph{Edge Automorphism Ratio}, the proportion of theoretically indistinguishable edges i.e., the complement of $\modify{\mathcal{E}^{\mathcal{M}}}$, giving a principled, model-agnostic measure of expressivity. (\Cref{sec:measure and framework})
\item \textbf{Automorphism-Aware GNN with Theoretical Guarantees:} We propose
  \fullname\ (\method), which augments standard GNN layers with two key designs:
  (D1)~an automorphism-aware dropout and (D2)~subgraph orbit-based aggregation. We prove that \method exceeds the expressivity of the standard GNN. We make our code and data   available at
  \href{https://anonymous.4open.science/r/ANP4Link-6A70/syn_graph/lp_agcn_tri.py}{[Anon Repo Link]}.\color{black}
  (\Cref{sec:design})
\item \textbf{Extensive Empirical Validation:} We evaluate \method\ against
  state-of-the-art GNNs on synthetic graphs spanning low to high automorphism
  levels and on diverse real-world networks. Our ablation studies show up to a
  \num{42.36}\% improvement over the second-best model in high-symmetry settings and
  consistent gains of up to \num{28.44}\% across real-world benchmarks.
  (\Cref{sec:exp}).
\end{itemize}
\color{black}

\section{Related Work}
\label{sec:related work}
\noindent \textbf{Graph Expressivity Theory.} Our work relates to recent efforts on characterizing GNN expressivity~\citep{zhang2024weisfeilerlehmanquantitativeframeworkgnn}. Standard MPNNs are theoretically bounded by the 1-WL test in distinguishing non-isomorphic graphs~\citep{Xu2018HowPA,weisfeiler1968reduction}, motivating extensions based on the higher-order $k$-WL  hierarchy~\citep{barcelo2020logical,srinivasan_equivalence_2020,maron2019invariant,maron2019universality,morris2019weisfeiler}. 2D-WL further incorporates pairwise structural features for link prediction~\citep{hu2022twodimensionalweisfeilerlehmangraphneural}, while unrolling distance analyzes MPNN generalization through aligned computation trees. More recently, $k_\phi$-$k_\rho$-$m$-WL unifies message-passing link predictors by characterizing encoder expressivity and structural neighborhood radius.
\begin{wraptable}{t!}{0.67\linewidth}
\vspace{-0.4cm}
\centering
\scriptsize
\setlength{\tabcolsep}{2pt} 
\caption{Graph Expressivity Frameworks for GNNs.}
\vspace{-0.2cm}
\begin{tabularx}{\linewidth}{X ccccc}
\toprule
Method & Level & M-Agnostic & Invariance. & Quantification. & Alignment. \\
\midrule
$1$-WL~\citep{Xu2018HowPA} & node & \cmark & \cmark & \xmark & \xmark \\
$k$-WL~\citep{Bker2023FinegrainedEO} & node & \cmark & \cmark & \xmark & \xmark \\
2D-WL~\citep{hu2022twodimensionalweisfeilerlehmangraphneural} & edge & \xmark & \cmark & \xmark & \cmark \\
$k_\phi$-$k_\rho$-$m$-WL~\citep{lachi2025bridgingtheorypracticelink} & edge & \xmark & \xmark & \cmark & \xmark \\
Homomorphism.~\citep{zhang2024weisfeilerlehmanquantitativeframeworkgnn} & subgraph & \cmark & \cmark & \cmark & \xmark \\
Unrolling Distance.~\citep{Vasileiou2025UnderstandingGI} & node & \cmark & \cmark & \xmark & \xmark \\
\midrule
\textbf{Ours} & \textbf{edge} & \cmark & \cmark & \cmark & \cmark \\
\bottomrule
\end{tabularx}
\label{tab:related_work_comparison}
\vspace{-0.6cm}
\end{wraptable}
As summarized in \Cref{tab:related_work_comparison}, our framework simultaneously ensures model-agnosticism,
permutation-invariance, quantitative characterization and direct alignment with the prediction task, as illustrated in \Cref{fig:framework}. We provide a more comprehensive discussion of related work, MPNNs for link prediction, and comparisons with other similar designs in \Cref{app:related work}.

\section{Quantifying Edge Discriminability: A Theoretical Framework}
\label{sec:measure and framework}
In this section, we provide key notation and present our theoretical framework before introducing EAR as a quantitative measure. We first review the fundamental concepts of graph automorphisms and node orbits. We then define the GCN-induced notion of structural equivalence, referred to as the subgraph orbit. Building on these concepts, we introduce the proposed WL-induced edge orbit and its associated edge-level equivalence relation. Finally, we define a measure for comparing the expressive power of different models based on their ability to distinguish structurally equivalent edges.

Our framework stems from a fundamental observation that the
representational power of a GNN model is effectively bounded by the set of links
it can structurally distinguish. We denote this set of distinguishable links as
$\mathcal{E}^\mathcal{M}$. Consequently, we can treat $\mathcal{E}^\mathcal{M}$
as a direct proxy for the model's expressivity given it fully characterizes which
edges will receive unique representations. Using $\mathcal{E}^\mathcal{M}$ is
also beneficial for model comparison as it enables the expressive power of
different models to be easily compared via set relations. 

\subsection{Notation and Preliminaries}
\label{sec:automorphic}
We summarize key notation in \Cref{tab:dfn} (\Cref{app:dfn}). Let $\mathcal{G} = (\vertexSet, \edgeSet)$ be an undirected, unweighted graph with adjacency matrix
$\matA \in \{0,1\}^{|\vertexSet| \times |\vertexSet|}$ and node feature matrix
$\matX \in \mathbb{R}^{|\vertexSet| \times F}$. We then consider a permutation-equivariant GNN as a mapping
$\phi_w : (\mathbf{X}, \matA) \to \mathbf{H} \in \mathbb{R}^{|\vertexSet| \times F}$,
where $w \in \mathcal{W}$ denotes the weight set of the
full link-prediction model, drawn from the weight space $\mathcal{W}$.
\begin{definition}[Isomorphism and Automorphism]
\label{dfn:graph-automorphism}
Two graphs $\mathcal{G}$ and $\mathcal{H}$ are said to be \emph{isomorphic}
($\mathcal{G} \cong \mathcal{H}$) if there exists a bijective mapping $\pi:V_\mathcal{G} \to V_\mathcal{H}$ such that
$(v, u) \!\in\! \mathcal{E}_\mathcal{G} \iff (\pi(v), \pi(u)) \!\in\!
\mathcal{E}_\mathcal{H}$. An \emph{automorphism} of $\mathcal{G}$ is the special case $\mathcal{H} = \mathcal{G}$: i.e., a relabeling of $\mathcal{G}$'s own vertices that leaves the edge set unchanged.
\end{definition}
\begin{definition}[Orbits]
\label{dfn:orbits}
The automorphisms of $\mathcal{G}$ form a group $\operatorname{Aut}(\mathcal{G})$
under composition. The \emph{orbit} of $v$ is
$\mathcal{O}(v) = \{\pi(v) : \pi \in \operatorname{Aut}(\mathcal{G})\}$,
i.e., the set of nodes automorphic to $v$. Orbits partition $\vertexSet$.
\end{definition}
\vspace{-7pt}
All automorphisms induce a partition of the vertex set $\mathcal{V}_\mathcal{G}$ into 
disjoint subsets called \emph{orbits}, where each orbit represents a unique structural role.
As illustrated in \Cref{fig:motive}, $\mathcal{G}_1$ has four distinct node \emph{orbits} ${a}$, ${b}$, ${c}$, and ${f}$ shown by color. Nodes in the same orbit play identical structural roles and are therefore automorphic. 

\subsection{Formalizing Edge Discrimination}
Classical \emph{orbits} are node equivalence classes of the global
automorphism group; in contrast, a
$k$-layer message-passing GNN observes the $k$-hop neighborhood of a node
and is bounded in expressivity by the 1-WL test
\citep{Xu2018HowPA, morris2019weisfeiler}. This creates a mismatch between
the global notion of automorphism and the local receptive field of GNNs
\citep{xu2021automorphic, zhang2023hierarchy}. Thus, global automorphisms is unsuitable for NAP and fails to capture the blind spots of GNNs regarding automorphic nodes (\Cref{fig:motive}).


\textit{We formalize a structural equivalence which captures the
structural role of a node as perceived by a GNN encoder. }
\begin{definition}[Subgraph Orbit]\label{dfn:subgraph-orbit}
With subgraph $\mathcal{S}^{(k)}(u)$ induced by all
nodes within distance $k$ and rooted at $u$, and representation $\phi_w(\mathcal{S}^{(k)}(u))$ of the root produced by a
$k$-layer message-passing GNN. Two nodes
$u,v\in\mathcal{V}$ are \emph{$k$-equivalent}, written $u\sim_k v$, if
\begin{equation}
\phi_w\big(\mathcal{S}^{(k)}(u)\big) = \phi_w\big(\mathcal{S}^{(k)}(v)\big)
\quad \text{for all } w .
\end{equation}
A $k$-layer message-passing GNN computes the root representation
from $\mathcal{S}^{(k)}(u)$. Throughout the paper, we use the equivalence
obtained after 1-WL stabilizes and omit the subscript $k$.
\modify{We write $\mathcal{O}_{\mathcal{S}}(u)$, or simply $\mathcal{O}_u$ when the
context is clear, for the resulting \emph{subgraph orbit} of $u$: its
equivalence class $\{v \in \mathcal{V} : v \sim u\}$ under this relation.}
\end{definition}

\textit{Based on subgraph orbits, we introduce an edge equivalence designed to characterize the limitations of GCN-based link prediction systems.}
\begin{definition}[WL-Induced Edge Orbit]\label{def:link-orbit}
For a node pair $e=\{u,v\}$, $u,v\in\mathcal{V}$, not necessarily in
$\mathcal{E}$, we define its \emph{WL-induced edge orbit} as the multiset of the subgraph orbits (\Cref{dfn:subgraph-orbit}) of its endpoints:
\begin{equation}\vspace{-2pt}
  \mathcal{O}_{\mathcal{E}}(e) \coloneqq
  \multiset{\mathcal{O}_{\mathcal{S}}(u),\,\mathcal{O}_{\mathcal{S}}(v)}.
\end{equation}
Two pairs $e,e'$ are \emph{WL-edge-equivalent}, written
$e\sim^{\mathcal{E}} e'$, iff
$\mathcal{O}_{\mathcal{E}}(e)=\mathcal{O}_{\mathcal{E}}(e')$. The multiset makes the definition invariant to endpoint order
while distinguishing pairs whose endpoints share a subgraph orbit
from pairs whose endpoints belong to different orbits.
\end{definition}
\begin{wrapfigure}{r}{0.32\linewidth}
\centering
\vspace{-0.7cm}
\includegraphics[width=\linewidth]{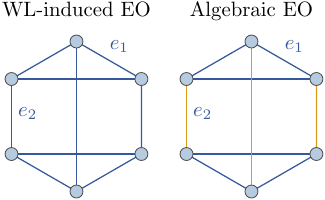}
\caption{Graph with one WL-induced edge orbit but two algebraic edge orbits.}
\label{fig:prism-edge-orbit}
\vspace{-1.2cm}
\end{wrapfigure}
\vspace{-0.3cm}

\textbf{Clarification}
We remark that the WL-induced edge orbit differs from classical \emph{edge orbits} in algebraic group theory.
Consider the triangular prism in
Fig.~\ref{fig:prism-edge-orbit}.
Because the graph is 3-regular, 1-WL assigns every vertex the same color,
placing all edges in a single WL-induced edge orbit. The graph automorphism
group, however, has two edge orbits: triangle edges and matching edges.
Automorphisms preserve triangle membership, so no automorphism maps a
triangle edge $e_1$ to a matching edge $e_2$.
WL-induced equivalence edge relation proposed in our paper is a model-induced equivalence determined by GCN.

We next formalize when a parameterized model
$\mathcal{M}_w$ discriminates a pair of such edges, and then compare
models by the sets of edge pairs they can distinguish.
\noindent
\begin{definition}[Edge Discrimination]\label{def:link-discrimination}
Let $e_1$ and $e_2$ belong to the same
\emph{WL-induced edge orbit}, i.e., $\mathcal{O}_{\mathcal{E}}(e_1) = \mathcal{O}_{\mathcal{E}}(e_2)$. A model
$\mathcal{M}_w$ \emph{discriminates} $e_1$ and $e_2$ if there
exists $w \in \mathcal{W}$ such that $\mathcal{M}_w(e_1) \neq \mathcal{M}_w(e_2)$,
written $e_1 \not\equiv_{\mathcal{M}} e_2$. We collect all such
pairs into the \emph{distinguishable edge set}
$\mathcal{E}^\mathcal{M} = \{\, \{e_1, e_2\} : e_1 \not\equiv_{\mathcal{M}} e_2 \,\}$.
\end{definition}
\begin{definition}[Higher Expressiveness]\label{def:more-expressive}
We order models by inclusion of their distinguishable edge sets. Given two models
$\mathcal{M}_1$ and $\mathcal{M}_2$, we say $\mathcal{M}_1$ is \emph{at least as
expressive} as $\mathcal{M}_2$, denoted $\mathcal{M}_1 \succeq \mathcal{M}_2$, if for every pair of edges $e_1, e_2 \in \mathcal{E}(\mathcal{G})$,
$e_1 \not\equiv_{\mathcal{M}_2} e_2 \;\Rightarrow\; e_1 \not\equiv_{\mathcal{M}_1} e_2$;
equivalently, $\mathcal{E}^{\mathcal{M}_2} \subseteq \mathcal{E}^{\mathcal{M}_1}$.
\end{definition}\vspace{-8pt}
\noindent The distinguishable edge set $\mathcal{E}^\mathcal{M}$ thus fully
characterizes a model's expressivity, capturing exactly the links it can
distinguish. We illustrate the framework in \Cref{fig:framework}.
\color{black}
\subsection{Quantifying Link Ambiguity: A Quantitative Measure}
\label{subsec:measure}
Building upon distinguishable edge set $\mathcal{E}^{\mathcal{M}}$, we introduce the \emph{EAR}, a quantitative measure of theoretically indistinguishable edges under a GNN encoder $\phi_w$. EAR will be used to empirically analyze the effect of structure automorphism and motivate our principled design choices.
\begin{figure}[t]
  \centering
  \includegraphics[width=0.75\textwidth]{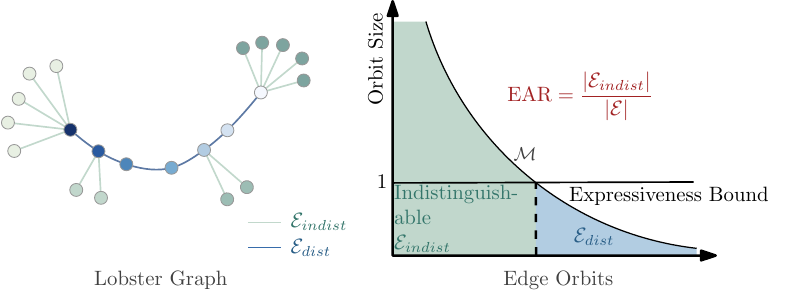}
  \vspace{-0.2cm}
  \caption{Illustration of EAR based on WL-induced edge equivalence. Green denotes indistinguishable edges $\mathcal{E}_{\text{indist}}$ and
   blue denotes distinguishable edges $\mathcal{E}_{\text{dist}}$. Expressiveness bound illustrates the theoretical discriminativity 
   limit of a GNN model. $\mathcal{E}^{\mathcal{M}}$ provides a formal characterization of the expressivity of model $\mathcal{M}$ 
   by capturing the set of edge pairs it can distinguish. A graph with more indistinguishable edges (high EAR) is inherently 
   more challenging for link prediction. } 
  \label{fig:framework}
  \vspace{-0.5cm}
\end{figure}
\begin{measure}[Edge Automorphism Ratio]
\label{dfn:automorphism-ratio}
Let $\mathcal{E}$ denote the set of all edges in the graph $\mathcal{G}$, \modify{and for an
edge $e=\{u,v\}$ let $\mathcal{O}_{\mathcal{E}}(e)$ be its WL-induced edge-orbit
signature (\Cref{def:link-orbit}). Denote by}
$$
\modify{[e]_{\mathcal{E}} \;=\; \left\{\, e' \in \mathcal{E} \;\middle|\; \mathcal{O}_{\mathcal{E}}(e') = \mathcal{O}_{\mathcal{E}}(e) \,\right\}}
$$
\modify{the \emph{edge orbit} of $e$, i.e., the set of edges sharing its signature
(equivalently, the equivalence class of $e$ under $\sim^{\mathcal{E}}$).}
\emph{EAR} is defined as the proportion of indistinguishable edges
$\mathcal{E}_{\text{indist}}$, written as:
$$
\mathrm{EAR} = \left( \frac{\left| \mathcal{E}_{\text{indist}} \right|}{|\mathcal{E}|} \right)^{\gamma}, \modify{\mathcal{E}_{\text{indist}} = \left\{ e \in \mathcal{E} \mid \left|[e]_{\mathcal{E}}\right| > 1 \right\}}.
$$
 where $\frac{|\mathcal{E}_{\mathrm{indist}}|}{|\mathcal{E}|}$ is the proportion of indistinguishable edges in the graph, and $\gamma\in(0,1]$ is a power-law scaling factor. For $\gamma<1$, the exponent increases EAR values between zero and one while preserving the endpoints. Indistinguishable edges are those belonging to non-singleton edge orbits, i.e., those whose \emph{orbit size} $|[e]_{\mathcal{E}}|>1$, because their WL-induced structural roles are not unique in the graph.
\end{measure}
\noindent \textbf{EAR Computation.} We utilize 1-orbit-WL (\Cref{algo:wl_forward}) adapted from~\citep{morris2019weisfeiler} to efficiently estimate edge orbits $\mathcal{O}_{\mathcal{E}}$. For example, in graph $\mathcal{G}_1$ of \Cref{fig:motive}, four out of five edges belong to non-singleton edge orbits, yielding $\mathrm{EAR}(\mathcal{G}_1)=\frac{4}{5}=0.8$. In contrast, all edges in $\mathcal{G}_2$ are indistinguishable, resulting in $\mathrm{EAR}(\mathcal{G}_2)=1$. The EAR measure ranges from $[0, 1]$ and applying a power-law transformation improves its spread across this range, in contrast to the raw ratio, which
typically falls within $[0, 0.4]$ for many real-world graphs. 
\subsection{Connecting EAR to GNN Performance}
\begin{wrapfigure}{r}{0.32\linewidth}
  \vspace{-0.49cm}
  \centering
  \begin{subfigure}{\linewidth}
    \centering
    \includegraphics[width=\linewidth]{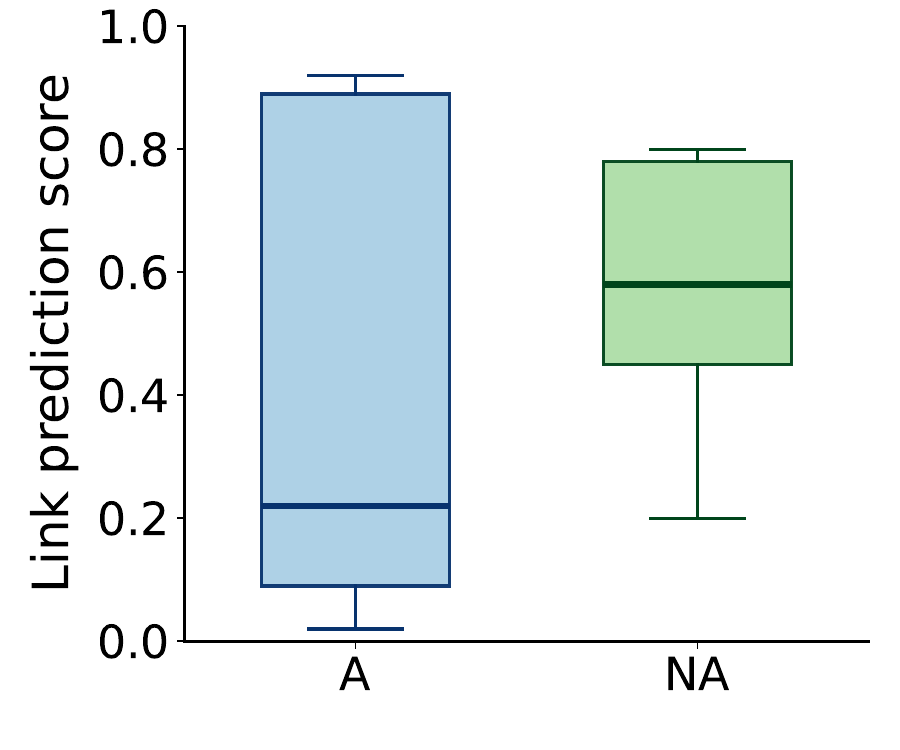}
    \caption{Automorphic (A) and non-automorphic (NA) edges.}
    \label{fig:lobster}
  \end{subfigure}
  \begin{subfigure}{\linewidth}
    \centering
    \includegraphics[width=\linewidth]{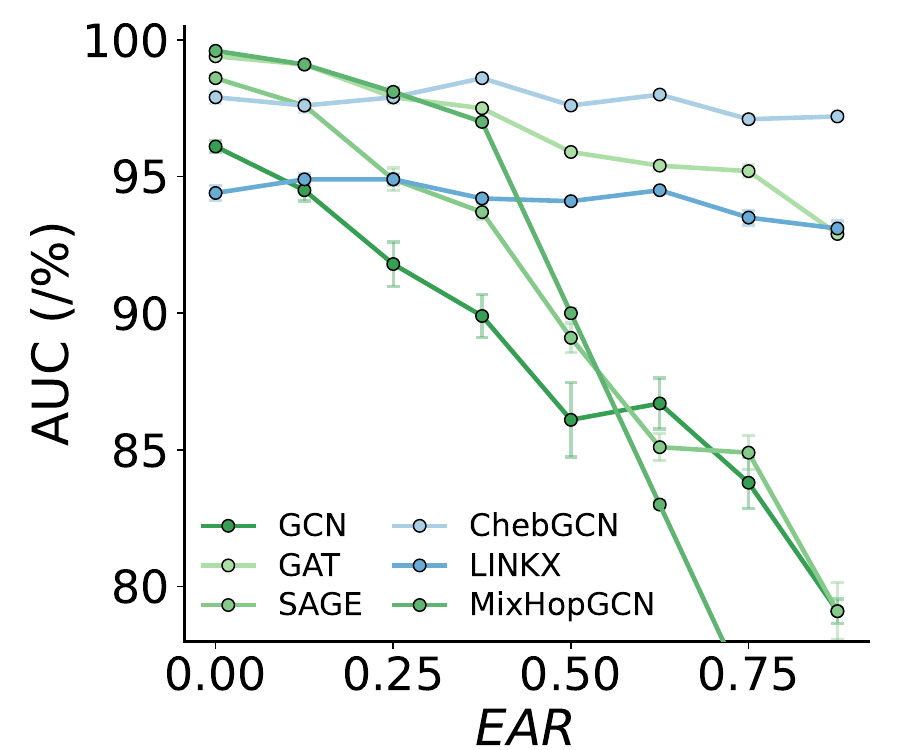}
    \caption{Performance decay on \textit{syn-triangle} as EAR increases.}
    \label{fig:motivate_example}
  \end{subfigure}
  \vspace{-0.65cm}
\end{wrapfigure}
As a motivating example, we demonstrate in \Cref{fig:lobster} that on a Lobster
graph (\Cref{fig:framework}) with automorphic edges detected by
\Cref{algo:wl_forward}. By design, the Lobster graph features a central path of
vertices with distinct orbits, while its leaves reside in identical orbits,
creating structural ambiguity. We show in \Cref{fig:lobster} that the automorphic edges exhibit significantly higher
variance in link likelihood scores for standard GNN \cite{Kipf2016SemiSupervisedCW},
with mean values falling below the detection threshold of 0.5. This indicates
that automorphic edges are likely to be incorrectly predicted as negative links.

To systematically investigate the impact of structural automorphism on link prediction, we construct synthetic graphs with controllable automorphism levels (\Cref{app:synthetic}) and report the mean AUC of representative models, including GNNs (SAGE~\citep{Hamilton2017InductiveRL}, MixHop~\citep{abuelhaija2019mixhophigherordergraphconvolutional}), the non-convolutional model LINKX~\citep{Lim2021LargeSL}, and the spectral method GCN-Cheby~\citep{he2024convolutionalneuralnetworksgraphs}, across increasing EAR values in \Cref{fig:motivate_example}. We observe that standard GNNs perform well under low automorphism but degrade substantially as structural automorphism increases. At $\EAR=0.88$, all GNN variants underperform LINKX by up to $15\%$, indicating severe loss of discriminability under high automorphism. Although MixHop partially alleviates this issue through higher-order neighborhoods, it still remains $15.3$\% below LINKX under strong symmetry. In contrast, GCN-Cheby, a
spectral method without equivariance constraints, maintains better robustness
across all levels, highlighting that GNN-based variants that work well under low
automorphism (EAR$=0.25$) are not appropriate for networks with medium/high
automorphism.

\section{Learning over Increasing Automorphism}
\label{sec:design}
Motivated by this limitation, we discuss and theoretically justify a set of key
design choices that, when appropriately incorporated in a GNN framework, can
maintain robust performance across the spectrum of automorphism values.
\subsection{D1: \Designone}
\noindent \textit{Intuition.} This design strategically perturbs the graph to
disrupt otherwise indistinguishable subgraph orbits, allowing GNNs to learn more
discriminative representations. \\
Given precomputed WL-hashes $\mathbf{O}\in\mathbb{Z}^{|\mathcal{V}|}$ from \Cref{algo:wl_forward}, we assign adaptive dropout probabilities according to normalized orbit sizes.
For each node $u$, we define
\[p_u
=
\min\left(
p_{\max},
\;
\alpha_p \cdot
\log\left(
1+
\frac{|\mathcal{O}_u|}{|\mathcal{V}|}
\right)
\right),
\]
where $|\mathcal{O}_u|$ denotes the cardinality of the subgraph orbit associated with node $u$, $\alpha_p>0$ is a scaling factor, and $p_{\max}<1$ bounds the dropout probability. The resulting edge and node dropout mechanisms are defined as
\begin{equation}
\label{equ:dropout}
\widetilde{A}_{uv}=
\begin{cases}
\displaystyle \frac{A_{uv}}{1-p_{uv}}, & \text{with probability }1-p_{uv},\\
0, & \text{with probability }p_{uv},
\end{cases}
\quad \forall (u,v)\in\mathcal{E},
\end{equation}
and
$
\widetilde{\mathbf{x}}_u
=
z_u \cdot \mathbf{x}_u,
\quad
z_u \sim \mathrm{Bernoulli}(1-p_u),
\; \forall u\in\mathcal{V}.
$
This design encourages the model to break structural automorphism and improve node distinguishability. Similarly, the edge dropout probability $p_{uv}$ is computed analogously from edge orbit sizes.
\begin{figure}[t!]
  \centering
  \includegraphics[width=0.85\linewidth]{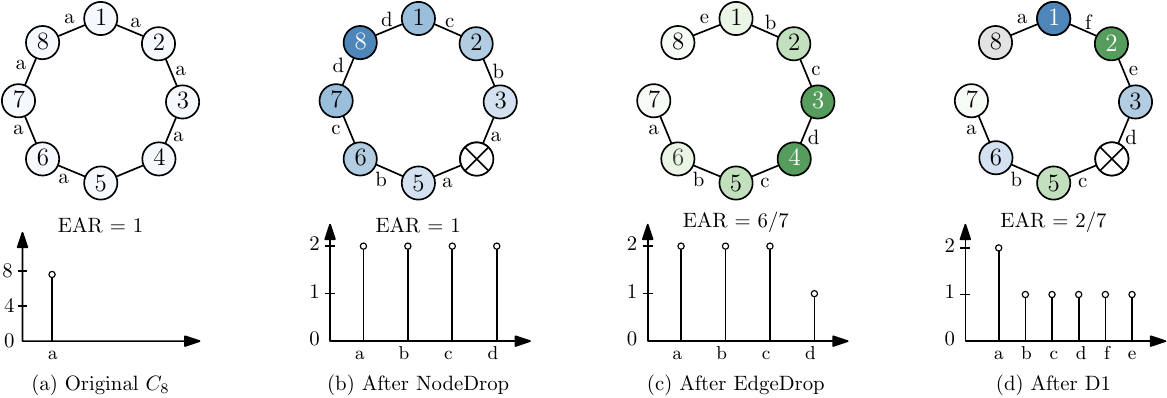}
  \vspace{-0.2cm}
\caption{Illustration of D1 \Designone\xspace on the cycle graph $C_8$, where node colors denote subgraph orbits under a 2-layer GNN. 
(a) All edges are automorphic ($\EAR=1$). 
(b) Dropping one edge breaks symmetry and reduces $\EAR$ to $\tfrac{6}{7}$. 
(c) Node dropout preserves $\EAR$ but makes half of the vertices distinguishable. 
(d) Combining node and edge dropout effectively improves GNN expressiveness (\Cref{theo:orbits-aware-dropout}).}
\label{fig:automorphism-aware dropout}
\vspace{-0.7cm}
\end{figure} \vspace{-0.2cm} 

\noindent \textit{Theoretical Justification.} We now formalize the role of D1 in
mitigating high edge automorphism. We show that GNN can better distinguish
subgraph orbits than a GNN without this design.
\begin{theorem}
\label{theo:orbits-aware-dropout}
Let $\phi_w:\mathcal{G}\rightarrow\mathbb{R}^{|\mathcal{V}|\times F}$ be a permutation-invariant GNN. If $\mathbb{E}_{w\sim\mathcal{W}}[\phi_w(\mathcal{S})]$ exists for all $\mathcal{S}\subseteq\mathcal{V}$, then GNN-D1 $\phi(p_a, p_n)$ with $p_a, p_n>0$, for any $\epsilon,\sigma>0$, it suffices to sample
\[
\mathbb{L}_\alpha=
\left\lceil
\frac{\log(\sigma)}
{|\mathcal{E}_{\mathrm{a}}|\log(1-p_{\mathrm{a}})
+
|\mathcal{E}_{\mathrm{n}}|\log(1-p_{\mathrm{n}})}
\right\rceil < \left\lceil \frac{\log(\sigma)}{-|\mathcal{E}|\log(1-p)} \right\rceil
\]
times such that $\mathbb{P}_{w\sim\mathcal{W}}
\!\left[
D(\mathcal{S},\mathcal{S}')
\geq \epsilon
\right]
\geq
1-\sigma,
$
where $\mathcal{S}'$ denotes the perturbed subgraph obtained by applying D1's node/edge dropout to $\mathcal{S}$, and $D(\cdot,\cdot)$ is a distance metric between their GNN embeddings (formally, \Cref{dfn:sufficient alteration,dfn:automorphic-node-problem}).
Here $\mathcal{E}_{\mathrm{a}}$ and $\mathcal{E}_{\mathrm{n}}$ denote the automorphic and non-automorphic edge sets, and $p_{\mathrm{a}}, p_{\mathrm{n}}$ the dropout probabilities applied to them.
$\mathbb{L}_\alpha$ is smaller than the sample complexity required by random dropout with a single rate $p$.
\end{theorem}
\vspace{-0.3cm}
The theorem shows that \textsc{GNN-D1}$\phi(p_a, p_n)$ can distinguish structurally equivalent nodes for any $p>0$, while requiring fewer repetitions than random dropout to break graph symmetries. Additional comparisons and analyses are provided in \Cref{app:related work}(\Cref{fig:comparison-dropout-mrr}).\\\vspace{-0.2cm}
\subsection{D2: Subgraph Orbit-Biased Aggregation}
\label{dfn:role-biased-aggregation}
\noindent \textit{Intuition.} D2 injects stochastic orbit-aware structural priors by perturbing identical structural roles with Gaussian noise, it introduces local asymmetry while preserving orbit-level information.\\
Let $\mathbf{O}\in\mathbb{Z}^{|\mathcal{V}|}$ denote the structural role identifiers obtained from $K$-iteration WL subtree labels, where $\mathbf{x}_v, \mathbf{o}_v$ represents raw embedding and orbit-aware identifier of node $v$. We construct a stochastic role embedding: \vspace{-0.5cm}
\begin{align}
\mathbf{h}_v^{(0)}
=
\mathrm{Emb}(\mathbf{o}_v)
+
\boldsymbol{\epsilon}_v,
\quad
\boldsymbol{\epsilon}_v
\sim
\mathcal{N}(0,\tau^2\mathbf{I}),
\label{eq:role-embedding}
\end{align}
where $\mathrm{Emb}(\cdot)$ is a learnable embedding table and $\tau$ controls the perturbation strength.
The noisy role embedding is concatenated with the initial node feature: $\mathbf{x}_v^{(0)}
=[\mathbf{h}_v^{(0)} \Vert \mathbf{x}_v],
\label{eq:concat-init} $
and injected into each message-passing layer:
\begin{align}
\mathbf{x}_v^{(k)}
=
\modify{\zeta^{(k)}}
\!\left(
\mathbf{x}_v^{(k-1)},
\square_{u\in\mathcal{N}(v)}
\phi^{(k)}
\!\left(
\mathbf{x}_u^{(k-1)},
\mathbf{x}_v^{(k-1)}
\right)
\right)
+
\mathbf{W}^{(k)}\mathbf{h}_v^{(0)},
\label{equ:subgraph-orbit aware aggregation}
\end{align}
where $\square$ denotes a permutation-invariant aggregation operator, $\phi,\modify{\zeta}$ are learnable functions, and $\mathbf{W}^{(k)}$ projects orbit-aware role embeddings into the hidden feature space.\\
\noindent \textit{Theoretical Justification.} To demonstrate D2's impact on automorphism, we first define an \textit{automorphism-dominant} vertex $v_a\in\mathcal{V}$ as a node whose neighbors are more likely to belong to the same orbit: $\mathbb{P}_{u\sim\mathcal{N}(v_a)}(\mathcal{O}_u=\mathcal{O}_{v_a})>\mathbb{P}_{u\sim\mathcal{N}(v_a)}(\mathcal{O}_u\neq \mathcal{O}_{v_a}),
$ where $\mathcal{O}_v$ denotes the subgraph orbit label of node $v$ (\Cref{dfn:subgraph-orbit}).
\begin{theorem}
\label{thm:automorphism-dominant}
With  a standard deterministic GNN with D2, the stochastic orbit-aware embedding increases their expected pairwise distance (\Cref{eq:role-embedding}):
$$\mathbb{E}\!\left[
\|\phi_{\mathrm{D2}}(u)-\phi_{\mathrm{D2}}(v_a)\|^2
\mid \mathcal{O}_u=\mathcal{O}_{v_a}
\right]
>
\mathbb{E}\!\left[
\|\phi(u)-\phi(v_a)\|^2
\mid \mathcal{O}_u=\mathcal{O}_{v_a}
\right]$$ Meanwhile, this distance remains bounded by the separation from non-automorphic neighbors when appropriately combined with GNN. Thus, D2 introduces local asymmetry within automorphic neighborhoods while preserving the structural separation from non-automorphic neighbors.
\end{theorem}

\subsection{\method: A Framework for Relaxing Equivariance in Link Prediction}
We now present Edge Orbit Equivariant Graph Neural Network (\method), which
integrates a standard GCN with our proposed designs D1 and D2 to address the
automorphic node problem across the full spectrum of low to high automorphism.
At a high level, \method consists of three stages: \stepone precomputes the
subgraph node orbit; \steptwo performs automorphism-aware message passing enhanced by D1--D2; and \stepthree performs link prediction.

The \textit{pre-compute stage} {\bf \stepone} uses WL (\Cref{algo:wl_forward})
to generate $\mathbf{O} \in \mathbb{Z}^{|\mathcal{V}|}$ subgraph orbit hash
labels. 
In the \textit{aggregation stage} {\bf \steptwo}, the generated automorphic embeddings are concatenated with original node features and repeatedly updated for $K$ rounds of message passing steps based on the aggregation in (\Cref{equ:subgraph-orbit aware aggregation}). We collect the final embeddings in $\mathbf{X}^{(K)}$, whose row $\mathbf{x}_u^{(K)}$ is the final embedding of node $u$. \\
In the \textit{prediction stage} {\bf \stepthree}, for a target edge $(u,v)$, we combine pairwise and structure features(common neighbors):
\vspace{-0.2cm}
\begin{equation}
\begin{aligned}
\mathbf{z}_{uv}
=
\mathrm{ReLU}\Big(
&\mathbf{W}_{cn}\,\mathrm{AGG}(\mathcal{N}(u)\cap\mathcal{N}(v)) \\
&+
\mathbf{W}_{h}
(\mathbf{x}_u^{(K)}\odot\mathbf{x}_v^{(K)})
+
\mathbf{W}_{o}
(\mathbf{o}_u^{(0)}\odot\mathbf{o}_v^{(0)})
\Big).
\end{aligned}
\end{equation}
\noindent \textbf{Time complexity.} The 1-WL algorithm with $K$ iterations requires $\mathcal{O}(K|\mathcal{E}|)$ time, since each iteration aggregates neighborhood information along graph edges~\citep{togninalli2019wasserstein}. D1 introduces no complexity. The feature propagation complexity is
$\mathcal{O}\!\left(L \cdot ( |\mathcal{V}| \cdot p^2 + |\mathcal{E}| \cdot p)\right)$,
where $p$ is the hidden feature dimension and $L$ is the number of GNN layers. Therefore, the overall complexity of \method is
$\mathcal{O}\!\left(
K|\mathcal{E}|
+
L \cdot ( |\mathcal{V}| \cdot p^2 + |\mathcal{E}| \cdot p )
\right).
$
A detailed analysis is provided in \Cref{app:time complexity}.

\color{black}
\begin{figure*}[t!]
  \centering
  \begin{subfigure}[t]{0.45\textwidth}
    \centering
    \includegraphics[width=\textwidth, clip]{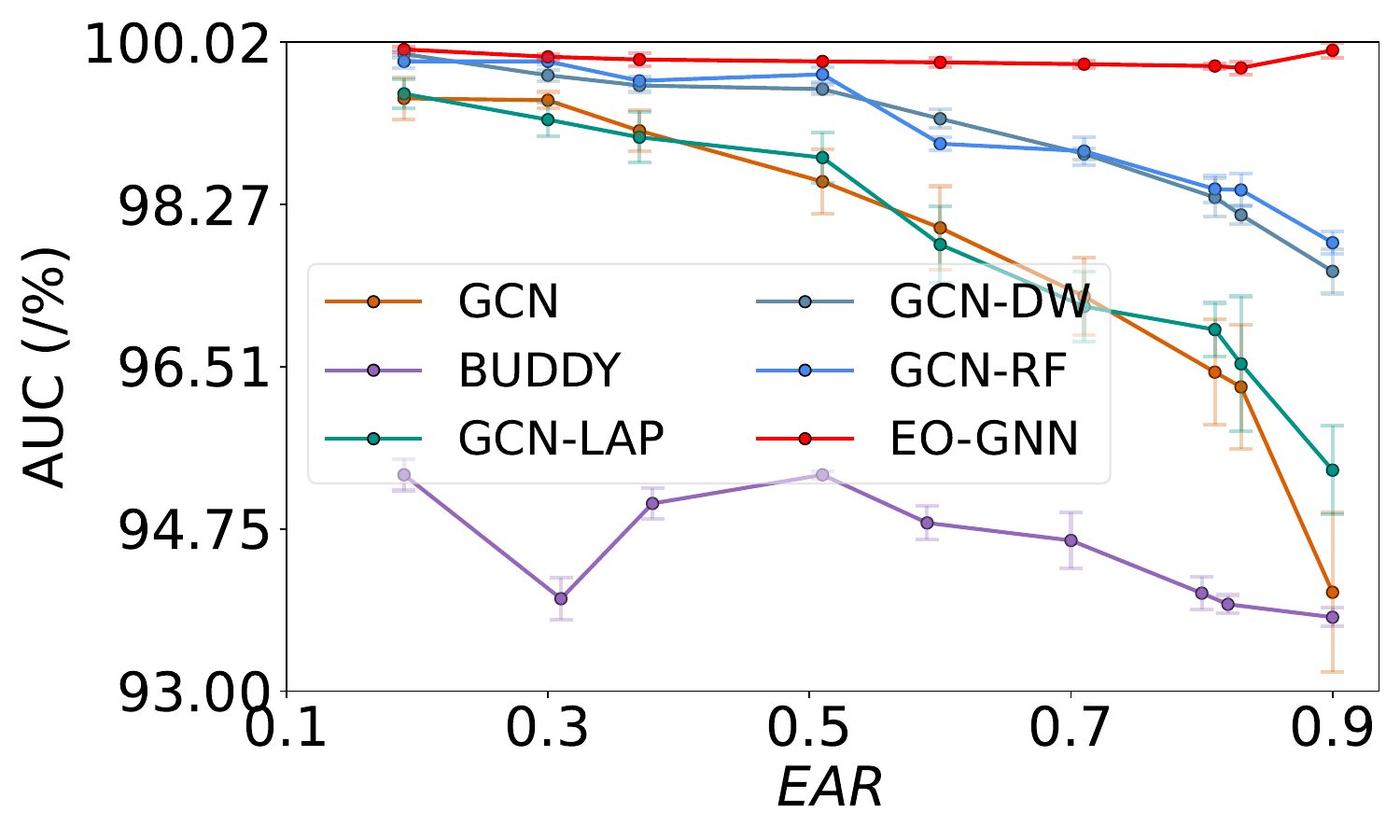}
    \vspace{-0.7cm}
    \caption{\texttt{syn-cora} (AUC)}
    \label{fig:syn-cora-auc}
  \end{subfigure}
  \begin{subfigure}[t]{0.45\textwidth}
    \centering
    \includegraphics[width=\textwidth, clip]{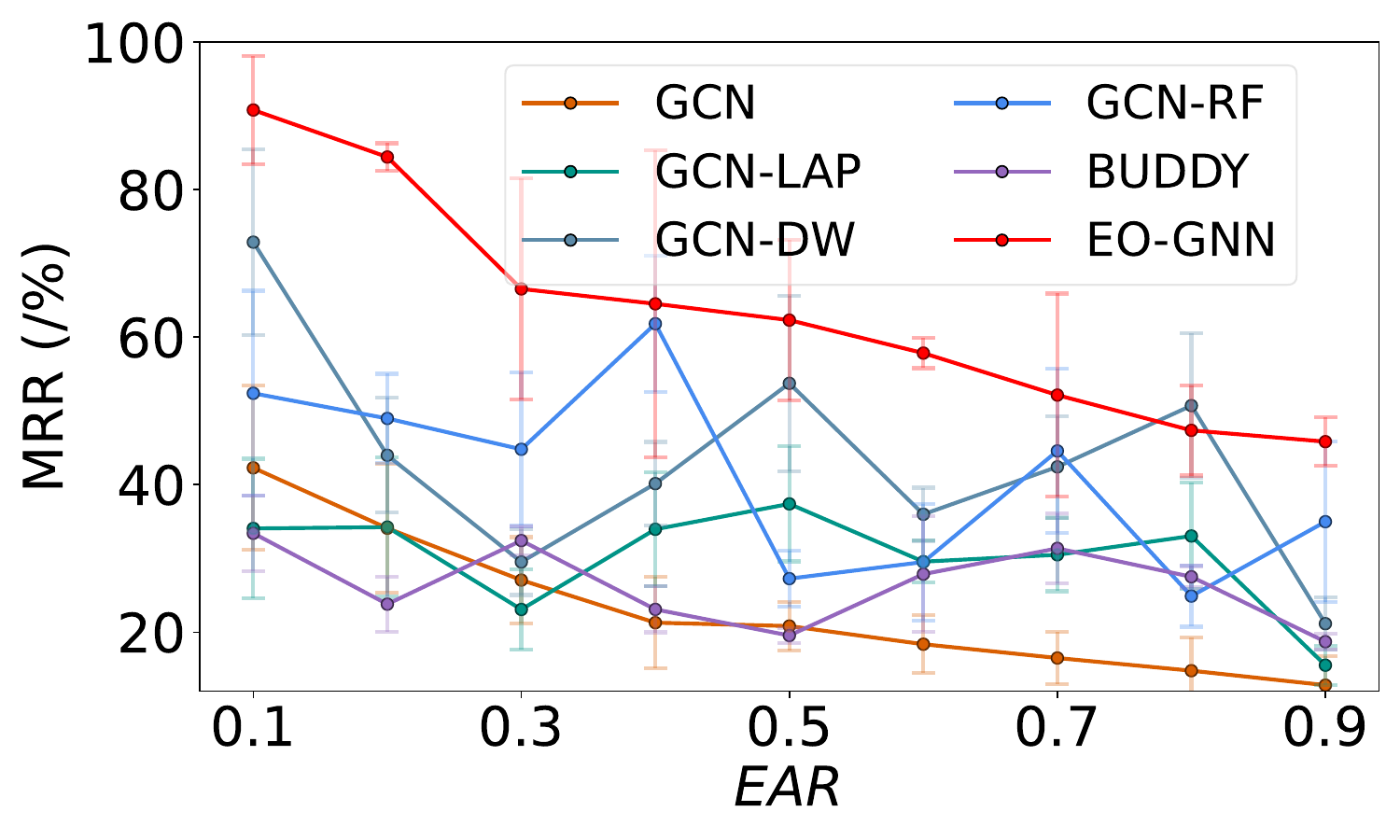}
    \vspace{-0.7cm}
    \caption{\texttt{syn-cora} (MRR)}
    \label{fig:syn-cora-mrr}
  \end{subfigure}
  \hfill
  \vspace{-0.2cm}
\caption{Performance on the semi-synthetic benchmarks. \method remains robust under high automorphism while maintaining competitive MRR and AUC in low-automorphism regimes.}
  \label{fig:5-syn-results}
  \vspace{-0.7cm}
\end{figure*}

\section{Empirical Evaluation}
\label{sec:exp}
In our empirical analysis, we first show the significance of the proposed designs
in \method on both synthetic and real-world graphs with low-to-high \EAR\xspace.
Specifically, we consider the following research questions: \textbf{(RQ1)} In a synthetic setting where the data generation process and automorphism level are controllable, do \method distinguish the automorphic structure? \textbf{(RQ2)} As the automorphism level increases, to what extent do the \method improve the performance? \textbf{(RQ3)} On complex real-world medium- and large-scale graphs with varying levels of automorphism, to what extent do \method and other baselines improve GNN performance? 

\subsection{Experimental Setup}
\textbf{Datasets.} We evaluate EO-GNN on seven real-world link prediction benchmarks from the Planetoid and OGB collections: Cora, Citeseer, Pubmed, ogbl-collab, ogbl-ddi, ogbl-ppa, and ogbl-citation2~\citep{Hu2021ogblscal}. 
To study the effect of structural automorphism, we further construct semi-synthetic datasets with controllable \EAR\xspace values from 0.1 to 0.9, following~\citep{lim2023expressive}. 
For the Planetoid and synthetic datasets, all methods use identical random splits (80\%/15\%/5\%), while official OGB splits are adopted for fair comparison with prior work. \\
\textbf{Training Protocol and Metrics.}
We report Mean Reciprocal Rank (MRR), Area Under the ROC Curve (AUC) and Hits@K depending on the benchmark setting, following the standard evaluation protocol established for each benchmark~\citep{li_evaluating_nodate-1}. We primarily emphasize MRR due to its sensitivity to structural discrimination quality in link prediction. Although the reported metric varies by dataset, within each dataset all methods are evaluated under identical splits and the same metric, so comparisons are always apples-to-apples. All reported results are averaged over five random runs using identical training, validation and testing protocols across all methods. Detailed hyperparameter configurations and training settings are provided in \Cref{subsec:app-choice-metrics}. \\
\noindent \textbf{Baseline models.} 
We compare EO-GNN with representative baselines spanning five categories:
(1)~\textit{heuristic link predictors}, including Common Neighbor (CN), Adamic--Adar (AA), and Resource Allocation (RA)~\citep{libennowell2007link,adamic2003friends,zhou2009predicting};
(2)~\textit{classical GNNs}, including GCN, GAT, GIN, GraphSAGE, and MixHop~\citep{Kipf2016SemiSupervisedCW,Velickovic2017GraphAN,Xu2018HowPA,Hamilton2017InductiveRL,abuelhaija2019mixhophigherordergraphconvolutional};
(3)~\textit{pairwise and structure-aware GNNs}, including SEAL, NBFNet, Neo-GNN, BUDDY, and NCN(C)~\citep{zhang_link_2018,3540261.3542517,yun2021neognns,chamberlain_graph_2023,wang_neural_2023};
(4)~\textit{graph-agnostic models}, such as LINKX and MLP~\citep{Lim2021LargeSL}; and
(5)~\textit{GNNs with positional encodings}, including GCN-DW, GCN-Lap, and GCN-RF~\citep{perozzi_deepwalk_2014,Ito2025LearningLP,Rampek2022RecipeFA}.
The real-world evaluation (\Cref{tab:main_results}) covers the heuristic, classical-GNN (GCN, GAT, GIN, GraphSAGE), and pairwise/structure-aware baselines, whereas the synthetic benchmark (\Cref{tab:app-syn-cora-results}) additionally includes MixHop, the graph-agnostic models, and the positional-encoding variants.
Additional dataset generation details and results are provided in \Cref{app:real,app:synthetic}.
\subsection{(RQ1) Effectiveness of \method}
\label{subsec:synthetic benchmark}
Figure~\ref{fig:syn-cora-auc} reports the mean test AUC (with standard deviation) over five runs of the top-6 performing models on the synthetic benchmark, selected for visual clarity; the complete comparison against all baselines is reported in \Cref{tab:app-syn-cora-results}. All baselines exhibit clear performance degradation as \EAR\xspace increases, especially classical GNNs under highly symmetric settings (\EAR\xspace $\geq 0.7$). Positional encoding methods, including GCN-DW, GCN-RF, and GCN-Lap, improve robustness over vanilla GNNs by up to 3.71\% in AUC, indicating that auxiliary structural signals partially mitigate symmetry ambiguity. Pairwise-feature methods remain relatively stable across all automorphism regimes, but often sacrifice performance under low-\EAR\xspace settings. In contrast, \method (Red) consistently achieves the best AUC and MRR, reaching up to 99.7\% with the lowest variance across all automorphism levels, demonstrating strong robustness under increasing structural automorphism.
\subsection{(RQ2) Ablation Study: Significance of Design Principles}
\label{subsec:ablation study}
\begin{figure*}[t!]
\vspace{-0.2cm}
  \centering
  \begin{subfigure}[t]{0.43\textwidth}
    \centering
    \includegraphics[width=\textwidth, clip]{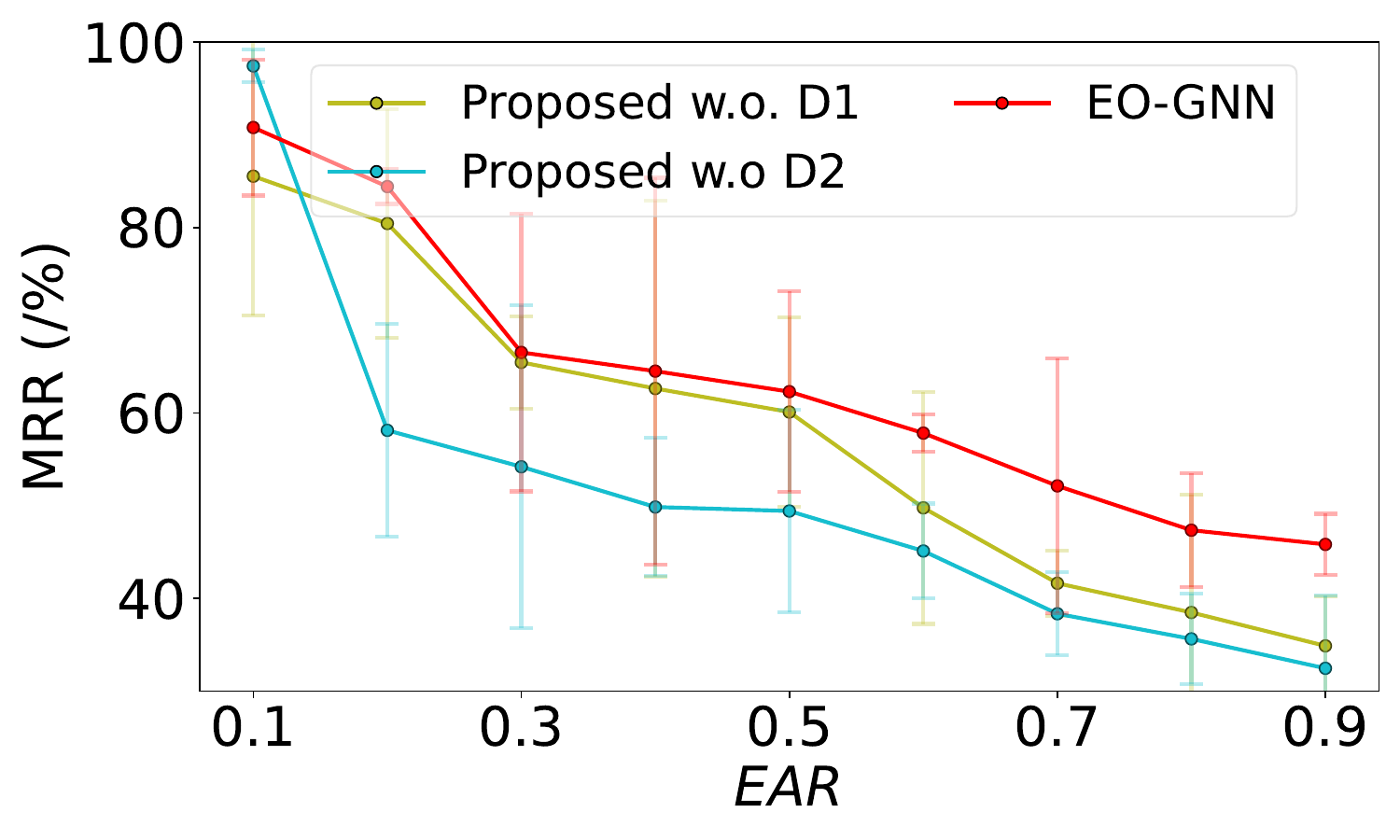}
    \vspace{-0.7cm}
    \caption{\texttt{syn-cora} (MRR)}
    \label{fig:ablation-syn-cora-mrr}
  \end{subfigure}
  \begin{subfigure}[t]{0.43\textwidth}
    \centering
    \includegraphics[width=\textwidth, clip]{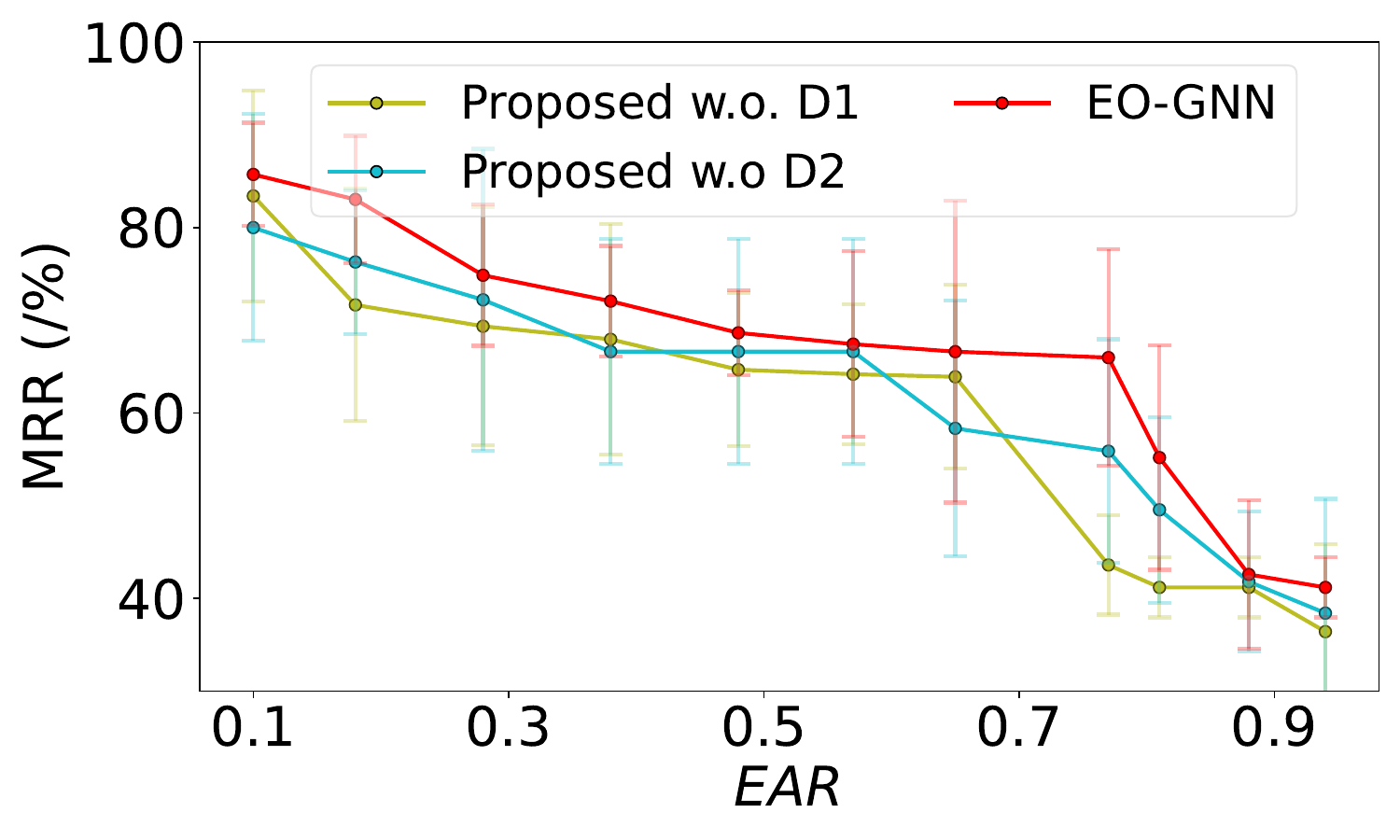}
    \label{fig:ablation-syn-citeseer-mrr}
    \vspace{-0.7cm}
    \caption{\texttt{syn-citeseer} (MRR)}
  \end{subfigure}
  \vspace{-0.2cm}
  \caption{Impact of D1 and D2 Modules in EO-GNN. When \EAR\xspace$\geq$\num{0.5}, both D1 and D2 contribute substantially to the performance gains, with D2 provides slightly stronger improvements on Cora.}
  \label{fig:ablation-syn-results}
  \vspace{-0.71cm}
\end{figure*} 

We evaluate the effectiveness of the proposed designs on \texttt{syn-cora} and
\texttt{syn-citeseer} through the ablation studies shown in
\Cref{fig:ablation-syn-results}. Specifically, we consider two variants of
\method: (1) removing the \Designone (D1,
\Cref{equ:dropout}), and (2) removing the \Designtwo (D2,
\Cref{equ:subgraph-orbit aware aggregation}). These variants are compared
against the full model.

\textit{D1: \Designone.}
As illustrated in
Fig.~\ref{fig:ablation-syn-results}-a and
Fig.~\ref{fig:ablation-syn-results}-b,
D1 (Green) substantially improves both the expressiveness and robustness of link
prediction across different automorphism levels, with particularly strong gains
in highly automorphic regimes. On Syn-Cora, D1 achieves
up to a 22.30\% improvement at $\EAR\xspace = 0.8$. Similarly, on
Syn-Citeseer, a clear inflection point emerges around $\EAR\xspace = 0.81$,
where the performance gap reaches 42.36\% in MRR. \\
\textit{D2: \Designtwo.} The results further demonstrate that D2 (Blue) plays an important role in maintaining robustness under highly automorphic settings, and its impact is more correlated with the underlying dataset characteristics. Removing D2 leads to a gradual degradation in both performance and stability as $\EAR\xspace$ increases, specifically, removing D2 reduces the average MRR by 26.28\% on Syn-Cora and by 10.09\% on Syn-Citeseer. Moreover, Syn-Citeseer appears more sensitive to high automorphism levels. Additional results and analyses are provided in \Cref{sec:app:ablation_syn_details}.
\setlength{\tabcolsep}{1pt}
\begin{table}[h]
    \scriptsize
    \centering
    \caption{%
    Real data: mean accuracy $\pm$ stdev over different data splits. Best model per benchmark highlighted in YellowGreen. OOM indicates that the algorithm requires over 40Gb of GPU memory or more than 24 hours. We highlight the two best-performing GNN models. }
    \vspace{-0.2cm}
    \label{tab:main_results}
    \begin{adjustbox}{width=0.9\textwidth}
    \begin{tabular}{lccccccc}
    \toprule
    \textbf{Dataset} & \textbf{PPA} & \textbf{Pubmed} & \textbf{Collab} & \textbf{DDI} & \textbf{Cora} & \textbf{Citeseer} & \textbf{Citation2} \\
    \textbf{\EAR} & \textbf{0.001} & \textbf{0.216} & \textbf{0.505} & \textbf{0.002} & \textbf{0.139} & \textbf{0.326} & \textbf{0.311} \\
    \textbf{Metric} & HR@100 & MRR & HR@50 & HR@20 & MRR & MRR & MRR \\
    \midrule
    CN & $27.65 {\scriptstyle \pm 0.00}$ & $14.66 {\scriptstyle \pm 0.06}$ & $61.37 {\scriptstyle \pm 0.00}$ & $17.73 {\scriptstyle \pm 0.00}$ & $32.88 {\scriptstyle \pm 0.09}$ & $21.13 {\scriptstyle \pm 0.02}$ & $74.30 {\scriptstyle \pm 0.00}$ \\
    AA & $32.45 {\scriptstyle \pm 0.00}$ & $19.87 {\scriptstyle \pm 0.30}$ & $64.17 {\scriptstyle \pm 0.00}$ & $18.61 {\scriptstyle \pm 0.00}$ & $47.33 {\scriptstyle \pm 0.09}$ & $24.61 {\scriptstyle \pm 0.11}$ & $75.96 {\scriptstyle \pm 0.00}$ \\
    RA & $49.33 {\scriptstyle \pm 0.00}$ & $19.16 {\scriptstyle \pm 0.27}$ & $63.81 {\scriptstyle \pm 0.00}$ & $6.23 {\scriptstyle \pm 0.00}$ & $47.17 {\scriptstyle \pm 0.11}$ & $23.94 {\scriptstyle \pm 0.16}$ & $76.04 {\scriptstyle \pm 0.00}$ \\
    \midrule
    GCN & $29.57 {\scriptstyle \pm 2.90}$ & $14.55 {\scriptstyle \pm 2.41}$ & $46.25 {\scriptstyle \pm 1.60}$ & $65.04 {\scriptstyle \pm 0.76}$ & $34.19 {\scriptstyle \pm 6.23}$ & $44.37 {\scriptstyle \pm 6.06}$ & $84.74 {\scriptstyle \pm 0.07}$ \\
    GIN & OOM & $15.96 {\scriptstyle \pm 2.21}$ & $48.84 {\scriptstyle \pm 0.54}$ & $67.05 {\scriptstyle \pm 0.61}$ & $42.47 {\scriptstyle \pm 7.02}$ & $29.48 {\scriptstyle \pm 6.21}$ & OOM \\
    SAGE & $25.80 {\scriptstyle \pm 1.94}$ & $11.34 {\scriptstyle \pm 2.41}$ & $48.10 {\scriptstyle \pm 0.81}$ & $63.69 {\scriptstyle \pm 1.45}$ & $30.81 {\scriptstyle \pm 8.07}$ & $44.48 {\scriptstyle \pm 8.40}$ & $82.60 {\scriptstyle \pm 0.36}$ \\
    GAT & OOM & $4.85 {\scriptstyle \pm 0.91}$ & $48.33 {\scriptstyle \pm 0.61}$ & $34.51 {\scriptstyle \pm 22.49}$ & $36.96 {\scriptstyle \pm 4.27}$ & $45.69 {\scriptstyle \pm 7.19}$ & OOM \\
    \midrule
    SEAL & $48.80 {\scriptstyle \pm 3.16}$ & \cellcolor{YellowGreen!20}$49.02 {\scriptstyle \pm 13.91}$ & $64.74 {\scriptstyle \pm 0.43}$ & $30.56 {\scriptstyle \pm 3.86}$ & $37.81 {\scriptstyle \pm 9.93}$ & $39.36 {\scriptstyle \pm 4.99}$ & $87.67 {\scriptstyle \pm 0.32}$ \\
    NBFNet & OOM & $19.46 {\scriptstyle \pm 2.42}$ & OOM & $44.00 {\scriptstyle \pm 0.58}$ & $41.48 {\scriptstyle \pm 5.11}$ & $38.17 {\scriptstyle \pm 3.06}$ & OOM \\
    Neo-GNN & $49.13 {\scriptstyle \pm 0.60}$ & \cellcolor{YellowGreen!40}$31.44 {\scriptstyle \pm 3.85}$ & $57.52 {\scriptstyle \pm 0.37}$ & $63.57 {\scriptstyle \pm 3.52}$ & $41.48 {\scriptstyle \pm 5.11}$ & $53.97 {\scriptstyle \pm 5.88}$ & $87.26 {\scriptstyle \pm 0.84}$ \\
    BUDDY & $49.85 {\scriptstyle \pm 0.20}$ & $19.46 {\scriptstyle \pm 2.42}$ & \cellcolor{YellowGreen!20}$65.94 {\scriptstyle \pm 0.58}$ & $78.51 {\scriptstyle \pm 1.36}$ & $30.78 {\scriptstyle \pm 5.55}$ & $22.84 {\scriptstyle \pm 0.36}$ & $87.56 {\scriptstyle \pm 0.11}$ \\
    NCN & \cellcolor{YellowGreen!40}$61.19 {\scriptstyle \pm 0.85}$ & $25.92 {\scriptstyle \pm 4.33}$ & $64.76 {\scriptstyle \pm 0.87}$ & \cellcolor{YellowGreen!20}$82.32 {\scriptstyle \pm 6.10}$ & \cellcolor{YellowGreen!20}$45.76 {\scriptstyle \pm 6.39}$ & \cellcolor{YellowGreen!20}$54.97 {\scriptstyle \pm 6.03}$ & \cellcolor{YellowGreen!30}$88.09 {\scriptstyle \pm 0.06}$ \\
    NCNC & \cellcolor{YellowGreen!20}$61.42 {\scriptstyle \pm 0.73}$ & $20.31 {\scriptstyle \pm 6.51}$ & \cellcolor{YellowGreen!40}$66.61 {\scriptstyle \pm 0.71}$ & \cellcolor{YellowGreen!40}$84.11 {\scriptstyle \pm 3.67}$ & \cellcolor{YellowGreen!40}$48.68 {\scriptstyle \pm 8.60}$ & \cellcolor{YellowGreen!40}$64.03 {\scriptstyle \pm 3.67}$ & \cellcolor{YellowGreen!40}$89.12 {\scriptstyle \pm 0.40}$ \\
    \midrule
    \method & \cellcolor{YellowGreen!70}$61.58_{\pm 0.45}$ & \cellcolor{YellowGreen!70}$68.51 {\scriptstyle \pm 0.74}$ & \cellcolor{YellowGreen!70}$70.95 {\scriptstyle \pm 0.81}$ & \cellcolor{YellowGreen!70}$86.95 {\scriptstyle \pm 1.63}$ & \cellcolor{YellowGreen!70}$50.02 {\scriptstyle \pm 5.83}$ & \cellcolor{YellowGreen!70}$72.88 {\scriptstyle \pm 1.45}$ & \cellcolor{YellowGreen!70}$90.15 {\scriptstyle \pm 0.06}$ \\ 
    \textbf{Improve.}
    & $+0.16\uparrow$
    & $+19.49\uparrow$
    & $+4.34\uparrow$
    & $+2.84\uparrow$
    & $+1.34\uparrow$
    & $+8.85\uparrow$
    & $+1.03\uparrow$ \\
    \bottomrule
\end{tabular}
\end{adjustbox}
\end{table}

\subsection{(RQ3) Evaluation on Real-World Graphs}
\label{subsec:eval_real_world}
\noindent \textbf{Significance of \method.}  \Cref{tab:main_results} shows that \method consistently achieves the strongest overall performance across both Planetoid and OGB link prediction benchmarks under multiple evaluation metrics. More importantly, the gains become increasingly pronounced on datasets with larger EAR values, such as Pubmed, Citeseer, and Collab, indicating that \method is particularly effective under strong structural automorphism and automorphism-induced ambiguity. In particular, on Pubmed, SEAL which also incorporates symmetry-aware structural modeling already outperforms conventional baselines by a large margin, while \method further improves upon SEAL by 19.49 MRR points(28.44\%). In contrast, datasets with near-zero \EAR\xspace values, such as PPA and DDI, exhibit comparatively smaller gains. We additionally observe that pairwise and subgraph-based methods generally achieve stronger robustness than classical message-passing GNNs, highlighting the importance of richer structural representations for link prediction.  
Nevertheless, as graph size increases, the estimation quality of subtree orbits in \Cref{algo:wl_forward} gradually decreases, causing D2 to classify more nodes as structurally equivalent and partially reducing its discriminative capability. Despite this limitation, \method consistently maintains strong performance with low variance across all benchmarks.

\section{Conclusion}
We highlighted the limitations of node-level symmetry frameworks for link prediction and introduced an edge-centric view via the concept of \emph{edge orbit}. We presented \method,
an automorphism-aware GNN with theoretically grounded components. Together, our
framework and design better align GNN inductive biases with graph symmetries,
improving both the quantitative theoretical framework and empirical performance in
link prediction.

\textbf{Limitations.} Our theoretical framework is principled and offers
theoretically grounded design insights. However, it assumes identical vertex
representations, following prior work. Although this assumption does not strictly
hold in real-world graphs, our experiments indicate that the theoretical findings
remain valid even when node features differ. Extending this analysis to
incorporate semantic embeddings (and understanding their interaction with
structural signals) remains an important direction for future theoretical
development.
\section*{Acknowledgments}
This work was supported in part by the National Science Foundation under Grants No. IIS-2212143 and IIS-2504090, and in part by the Federal Ministry of Research, Technology and Space (BMFTR), Germany, under award number 01IS23066.

\subsection*{AI use statement}
In accordance with the venue's AI policy, we disclose that generative AI tools were used solely to aid and polish writing, specifically grammar correction, improving clarity of exposition, and minor LaTeX formatting. No research content or code was produced with the aid of AI tools, and we take full responsibility for the final content of this work.

\subsection*{Ethics statement}
This work is methodological, studying the expressive power of graph neural
networks for link prediction. All experiments use publicly available benchmark
datasets (the Planetoid citation networks and the Open Graph Benchmark), which
contain no personally identifiable or sensitive information and are standard in
the community under their respective licenses. The research involves no human
subjects, and we foresee no direct risks relating to discrimination, fairness,
privacy, or security beyond the general dual-use considerations common to
machine-learning research. The authors declare no conflicts of interest.

\subsection*{Reproducibility statement}
We have taken several steps to ensure reproducibility. Our model and
experimental code are available as an anonymous repository, linked in the
contributions of \Cref{sec:design}. All theoretical claims state their
assumptions explicitly and are accompanied by complete proofs in
\Cref{app:proofs}. The construction of our synthetic benchmarks with
controllable automorphism (\EAR) is described in detail in \Cref{app:synthetic},
and the real-world datasets, their statistics, and the splits used are described
in \Cref{app:real}. Full hyperparameter settings and training protocols for all
methods are reported in \Cref{subsec:app-choice-metrics}.

\bibliography{iclr2027_conference}
\bibliographystyle{iclr2027_conference}

\newpage
\appendix
\section{Nomenclature}
\label{app:dfn}
We summarize the main symbols used in this work and their definitions below:
\begin{table}[h!]
     \caption{Major symbols and definitions.} %
     \label{tab:dfn}
     {\small
     \resizebox{0.8\textwidth}{!}{
    \begin{tabular}{ l@{\hspace{1cm}}p{11cm}}
         \toprule
         \textbf{Symbols} &  \textbf{Definitions}\\
         \midrule
         $\graph = (\vertexSet, \edgeSet)$ & graph $\graph$ with nodeset $\vertexSet$, edgeset $\edgeSet$ \\
         $\matA$ & $n \times n$ adjacency matrix of $\graph$ \\
         $\matX$ & $n \times F$ node feature matrix of $\graph$ \\
         $\V{x}_v$ & $F$-dimensional feature vector for node $v$ \\
         $\matL$ & unnormalized graph Laplacian matrix \\ 
         \midrule
         $\mathcal{G} / \text{Aut}(\mathcal{G}) $ & the quotient of the automorphism when it acts on the graph $\modify{\mathcal{G}}$ \\
         $ \text{Aut}(\mathcal{G})$ & automorphism on graph $\mathcal{G}$  \\
         $\pi$ & vertex permutation; $\pi \in \mathrm{Aut}^k(\mathcal{G})$ denotes a $k$-hop graph automorphism \\
         \midrule
         $N(v)$ & general type of neighbors of node $v$ in graph $\graph$ \\
         $\neighNoSelfLoop(v)$ & general type of neighbors of node $v$ in $\graph$  \textit{without self-loops} (i.e., excluding $v$)  \\
         $N_i(v),\neighNoSelfLoop_i(v)$ & $i$-hop/step neighbors of node $v$ in $\graph$ (at exactly distance $i$) maybe-with/without self-loops, resp.\\
         $\edgeSet_2$ & set of pairs of nodes $(u,v)$ with shortest distance between them being 2 \\
         $d, d_{\mathrm{max}}$ & node degree and maximum node degree across all nodes $v \in \vertexSet$, resp. \\
        \midrule
         $K$ & the number of rounds in the neighborhood aggregation stage \\
         $\matW$ & learnable weight matrix for GNN model \\
         $\rho$ & non-linear activation function \\
         $[ \cdot  \vert \cdot ]$& vector concatenation operator \\
         \midrule
         \texttt{AGGR} & function that aggregates node feature representations within a neighborhood \\
         \texttt{COMBINE} & function that combines feature representations from different neighborhoods \\
         $\mathcal{P}(\mathcal{V})$ & Power set of vertex set $\mathcal{V}$\\
        $\mathcal{W}$ & Weight space of the parameterized neural network; $w \in \mathcal{W}$ \\
        $\phi$ & Graph encoder \\
         \midrule
        $\mathcal{O}(v)$ & orbit of vertex $v$ under $\operatorname{Aut}(\mathcal{G})$ \\
        $O$ & $O \in \mathbb{Z}^{\vert \mathcal{V}\vert}$ WL-Hash value from  \Cref{algo:wl_forward}  \\
        $\mathbf{O}$ & $\mathbf{O} \in \mathbb{Z}^{\vert \mathcal{V}\vert\times d_o}$ embedded node feature based on the hash label from  \Cref{algo:wl_forward} as node feature \\
        $\vert \mathcal{O}_u \vert$ & orbit size of vertex $u$: number of the vertices sharing the same role as $u$ \\
        \modify{$\mathcal{O}_{\mathcal{S}}(u)\,(\!=\!\mathcal{O}_u)$} & \modify{subgraph orbit of $u$: its $k$-equivalence class under a GNN (\Cref{dfn:subgraph-orbit})} \\
        \modify{$\mathcal{O}_{\mathcal{E}}(e)$} & \modify{WL-induced edge-orbit signature of $e=\{u,v\}$: the multiset $\multiset{\mathcal{O}_{\mathcal{S}}(u),\mathcal{O}_{\mathcal{S}}(v)}$ (\Cref{def:link-orbit})} \\
        \modify{$[e]_{\mathcal{E}}$} & \modify{edge orbit of $e$: the edges sharing its signature, i.e.\ its $\sim^{\mathcal{E}}$ equivalence class (\Cref{dfn:automorphism-ratio})} \\
        $\mathcal{O}(\cdot)$ & Complexity \\
        $D$ & Decoder; typically implemented as a dot product \\
        $\mathcal{M}_w$ & The full link-prediction model for weight set $w$, i.e., the GNN encoder $\phi_w$ composed with a decoder $D$; $D$'s functional form is left unspecified \\
        $\mathcal{S}^{(k)}(u)$ & $k$-hop enclosing subgraph centered at node $u$ \\
        $\mathcal{S}^{(k)}(u)_{\epsilon}$ & $k$-hop enclosing subgraph of $u$ with $\epsilon$-level structural perturbation \\
        \modify{$\epsilon$} & \modify{Degree of structural alteration; an embedded distance in the representation space, i.e.\ the desired closeness to the unperturbed graph} \\
        $\alpha_\mathcal{V}$ & Proposed vertex-level automorphism measure (see Appendix~\Cref{app:automorphism}) \\
        $\alpha_\mathcal{E}$ & Proposed edge-level automorphism measure (see Appendix~\Cref{app:automorphism})\\
        \midrule
        $\widetilde{A} $ & Adjacency matrix after random dropout \\
        $\delta$        & Permissible distance from $a$ for input \\
        $\sigma$ & bound on failure probability; a guarantee is stated to hold with probability at least $1-\sigma$ (\Cref{theo:orbits-aware-dropout}) \\
        $\tau$ & standard deviation of the Gaussian perturbation noise used in D2 (\Cref{dfn:role-biased-aggregation}) \\
        $\mathbf{r}^{(0)}$        & Embedding vector from WL-Hash $\mathbf{o}$ for input in iteration $0$\\
        \midrule
               $\mathbf{E}_e \quad \quad$ & weight of graph adjacency matrix \\
        $\mathbf{X}^{(K)}$ & final embedding of GNN \\
        $\mathbf{H}_c$ & common neighbor as feature vector, $\mathbf{H}_c(i, j) = 1$ if $(i, j ) \in \mathcal{E}$ \\
        $\mathbf{R}_v^{(0)}$ & embedding from embedded WL-hash Label \\
        $\mathbf{H}_e$ & embedding from \\
        $\matW_f$ & weight $\matW_f \in\mathbb{R}^{p\times 1}$ in the decoder \\
        $p$ & feature dimension in the final embedding \\
         \bottomrule 
    \end{tabular}
    }
    }
    \end{table}
\newpage
\section{Extended Preliminary}
\begin{definition}[Graph Quotient]
\label{equ:quotient}
The \emph{quotient} of a graph $\mathcal{G}$ is the partition of its vertex set
into automorphic orbits under the action of $\operatorname{Aut}(\mathcal{G})$:
$\mathcal{G} \mathbin{/} \operatorname{Aut}(\mathcal{G})
= \{ \mathcal{O}(v) \mid v \in \mathcal{V}_\mathcal{G} \}.$
\end{definition} 
\noindent \textbf{Subgraph Orbit $\neq$ Orbit.} The difference between \emph{Subgraph Orbit} and standard \emph{Orbit} lies in the group action they are based on: an orbit is induced by is a \emph{global} automorphism (i.e., an adjacency-preserving permutation of the entire node set), whereas subgraph orbit in link prediction is a local generalization of automorphism based on the fact that the GNN operates on a small $k$-hop neighborhood. In  \Cref{fig:motive}, nodes b and d belong to the same subgraph orbit under a 1-layer GNN, as their 1-hop enclosing subgraphs are structurally identical. 
\noindent \textit{Motivation: It ensures that structural equivalence is rigorously characterized through the isomorphism of GNN-induced local enclosing subgraphs. } \\
 
\section{Extended Related Work}
\label{app:related work}
\noindent \textbf{MPNNs for Link Prediction.}
Link prediction is commonly performed using embeddings generated by GCN-based encoders followed by simple decoders such as the dot
product~\citep{Velickovic2017GraphAN, Xu2018HowPA, Hamilton2017InductiveRL}.
However, \citet{zhang_link_2018} showed that these vanilla GCNs cannot
distinguish automorphic nodes due to their inherent permutation invariance.
To address this limitation, \citet{Zhang2020LabelingTA} introduced node labeling features that encode relative distances between target nodes and their neighborhoods, enhancing structural discrimination.
Subsequent work incorporated manually designed structural priors, such as the Jaccard index~\citep{chamberlain_graph_2023, yun2021neognns}, while
\citet{wang_neural_2023} proposed soft completion of common-neighbor features to mitigate distribution shifts and structural holes. Recent work, including the introduction of higher-order shortest path methods and PageRank attention-based frameworks, has attempted to facilitate link prediction by utilizing manually extracted local structures \citep{10.1145/3637528.3672025}. Our work, in contrast, aims to provide a principled solution within a permutation invariance
framework. 

\noindent \textbf{Expressivity Comparison.} The closest method to ours is $k_\phi$-$k_\rho$-$m$-WL. However, our approach
differs in three key aspects: (i)~\textit{Motivation.} The
$k_\phi$-$k_\rho$-$m$-WL framework provides a unified view for hierarchically
comparing the aggregation and update designs of existing methods relative to
1-WL. In contrast, our work defines model expressiveness by explicitly
identifying, quantifying and measuring indistinguishable edges in standard GNNs
in a model-agnostic manner. (ii)~\textit{Objective.} While
$k_\phi$-$k_\rho$-$m$-WL aims to compare the expressiveness of existing
approaches for link prediction, our framework studies graph automorphisms in
real-world graphs and analyzes their impact on link prediction across different
categories of GNNs. (iii)~\textit{Measure and synthetic benchmark.} The
$k_\phi$-$k_\rho$-$m$-WL framework proposes a vertex-level measure to compare
graph automorphisms, which leads to a mismatch with the edge-level nature of the
problem we study. In contrast, our proposed measure is an empirical estimator
derived directly from our theoretical framework. This estimator is consistently
aligned with our method design, our synthetic benchmark and our theoretical
analysis.

\noindent \textbf{Dropout Comparison.} \textit{Dropout}~\citep{hinton2012improvingneuralnetworkspreventing} was originally introduced to mitigate overfitting by randomly deactivating hidden units during training. Subsequent analysis showed that dropout implicitly induces a regularization term that is first-order equivalent to an $L_2$ penalty after feature scaling~\citep{Wager2013DropoutTA}. In graph learning, dropout primarily has three variants. \textit{DropEdge}~\citep{rong2020dropedge} extends dropout by randomly removing edges during message passing, effectively alleviating the over-smoothing problem in GNNs~\citep{li2018deeper, chen2020simple}. \textit{DropNode}~\citep{you2021graphcontrastivelearningaugmentations} further generalizes this idea by randomly discarding vertices and their incident edges, serving as a powerful augmentation method that enhances model generalizability, transferability and robustness. \textit{DropPath}~\citep{huang2016deep} stochastically masks random-walk–based paths within the graph, reducing redundancy in local substructures. More recently, \textit{AdaEdge}~\citep{fan2021adaptively} optimizes the graph topology in a learnable fashion by adaptively applying dropout according to local homophily ratios during training.
We empirically compare \Designone \xspace with other dropout variants on the synthetic benchmark in \Cref{subsec:ablation study}.

\begin{figure*}[t]
    \centering
    \begin{minipage}{\textwidth}
        \centering
        \vspace{0.5em}
        \label{fig:dropout-comparison-syn-results}
    \end{minipage}
    \vspace{1em} 
    \begin{subfigure}[t]{0.48\textwidth}
        \centering
        \includegraphics[width=\textwidth, clip]{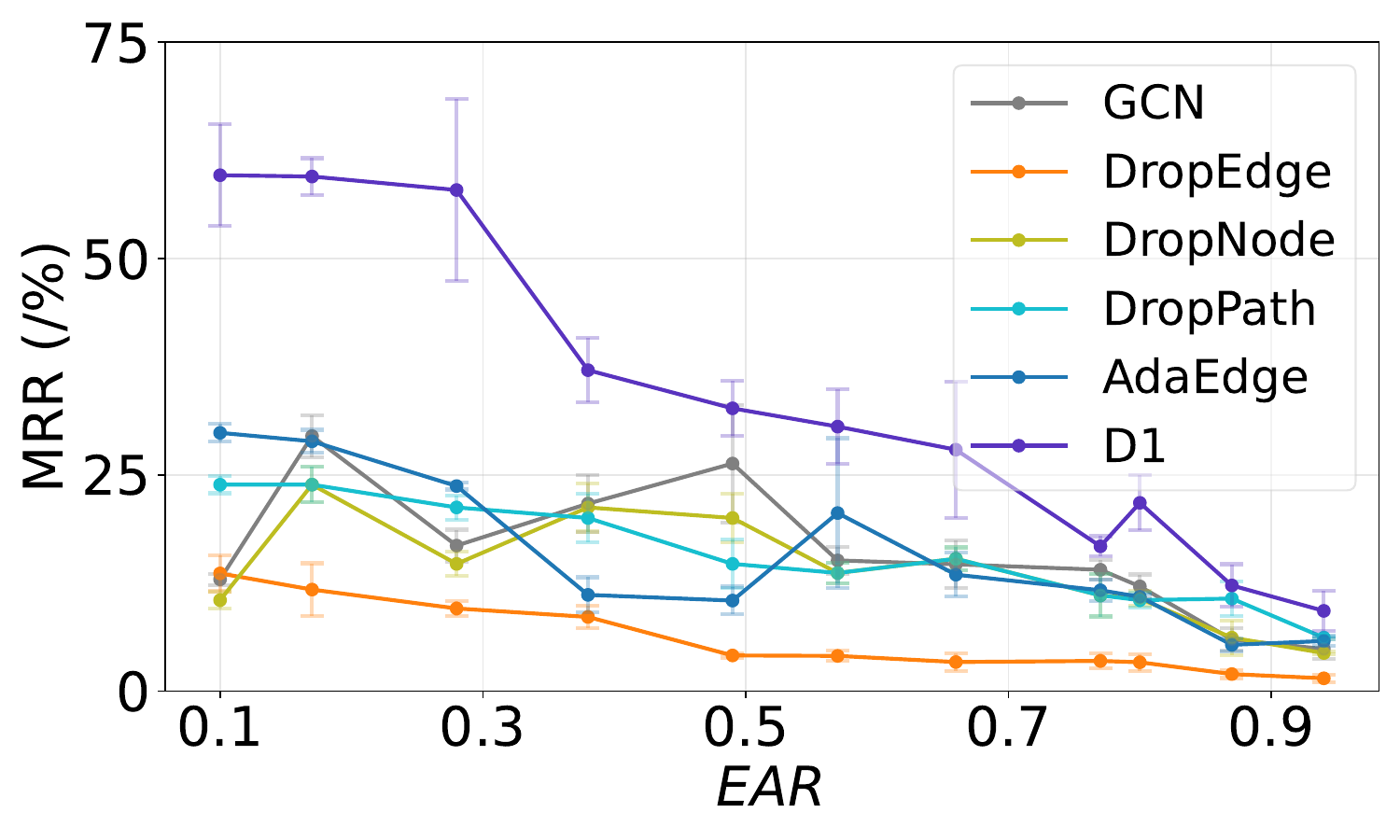}
        \caption{\texttt{syn-citeseer} (MRR)}
        \label{fig:comparison-dropout-auc}
    \end{subfigure}
    \begin{subfigure}[t]{0.48\textwidth}
        \centering
        \includegraphics[width=\textwidth, clip]
        {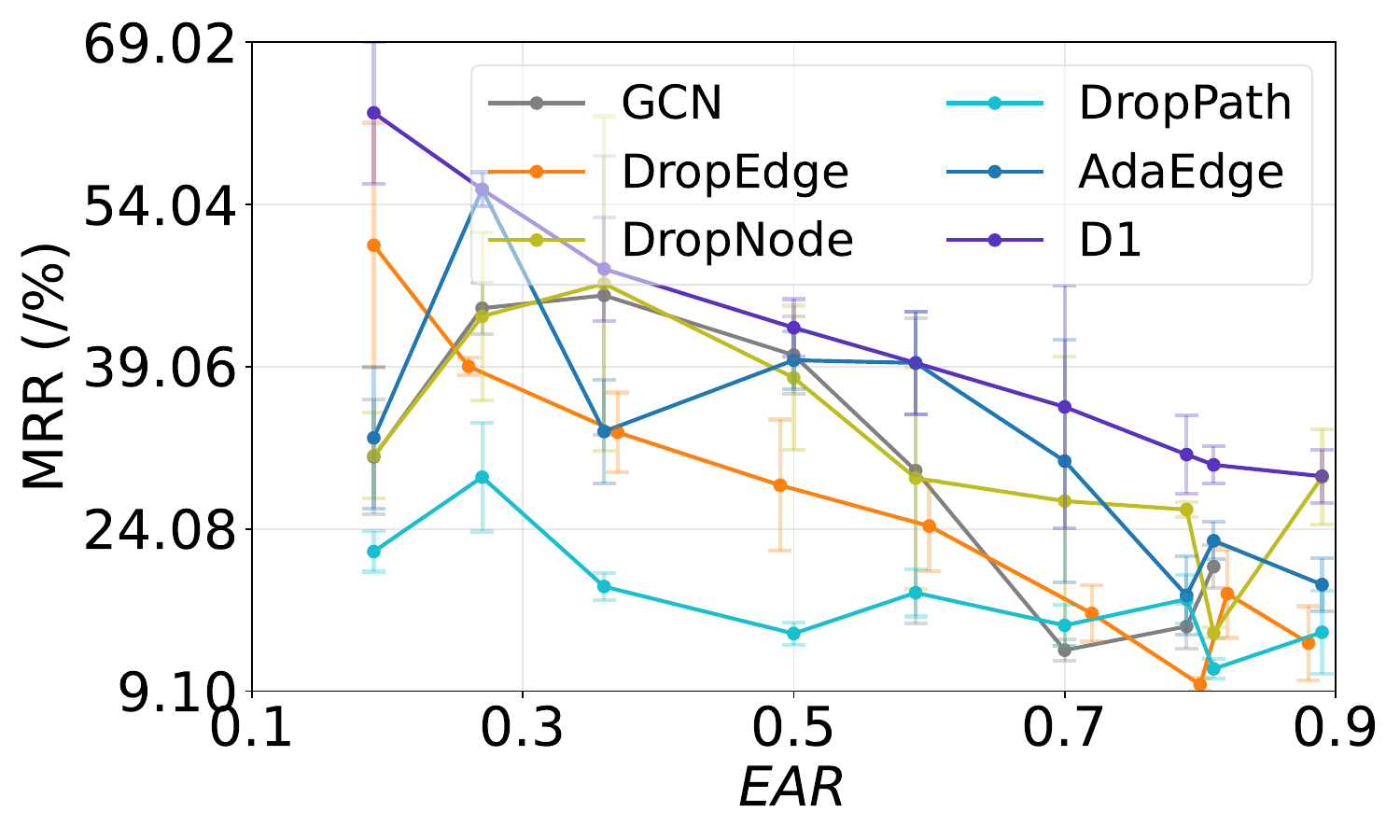}
        \caption{\texttt{syn-cora} (MRR)}
        \label{fig:comparison-dropout-mrr}
    \end{subfigure}
        \captionof{figure}{
        \textbf{Empirical Comparison of Dropout Variants on Synthetic Benchmarks. }
        Ablation study comparing our proposed method, $\text{\Designone}$ (D1), against three popular dropout approaches ($\text{DropEdge}$, $\text{DropNode}$, $\text{DropPath}$,  $\text{AdaEdge}$ and $\text{DropPath}$) on synthetic datasets $\textit{syn-citeseer}$ and $\textit{syn-cora}$. We show that $\text{\Designone}$ consistently achieves higher MRR than the other variants. This demonstrates that leveraging the WL test's automorphism prior allows $\text{\Designone}$ to identify and prune automorphic (structurally identical) edges more effectively. It enhances the model's expressiveness to distinguish non-isomorphic local structures. (See \Cref{sec:design} for method details and \Cref{subsec:ablation study} for other empirical evaluation) 
        }
\end{figure*}

\noindent \textbf{Summary}  
While prior works utilized dropout to address various issues in GNN. We are the first to leverage the automorphism prior derived from the WL test to optimize the dropout process. This design choice uniquely enables us to mitigate degradations in GNN model expressiveness, which we both theoretically justify and empirically demonstrate on diverse datasets. We note that our design philosophy is similar to the adaptive nature of $\text{AdaEdge}$, but our method specifically mitigates automorphism as opposed to optimizing for homophily labels.

\section{Proofs and Discussions of Theorems}
\label{app:proofs}

\subsection{Proof for Edge Orbit Equivariance}
\label{appendix:edge-orbit-equivariance}
\begingroup
\begin{Proposition}[Edge-Orbit Equivariance]
\label{prop:edge-orbit-equivariance}
Let $\mathcal{G}$ be an undirected graph with automorphism group
$\operatorname{Aut}(\mathcal{G})$ (\Cref{dfn:graph-automorphism}), and let
$\mathcal{O}_{\mathcal{E}}(e) = \multiset{\mathcal{O}_{\mathcal{S}}(u),\,\mathcal{O}_{\mathcal{S}}(v)}$
be the WL-induced edge orbit of $e = \{u, v\}$ (\Cref{def:link-orbit}). Then the
edge-orbit assignment is \emph{invariant} under automorphisms: for every
$\phi \in \operatorname{Aut}(\mathcal{G})$ and every $e = \{u,v\}$,
\[
\mathcal{O}_{\mathcal{E}}(\phi(u), \phi(v)) = \mathcal{O}_{\mathcal{E}}(u, v).
\]
For a permutation $\phi \notin \operatorname{Aut}(\mathcal{G})$ the equality need
not hold, since $\phi$ may neither preserve subgraph orbits nor map edges to
edges.
\end{Proposition}

\begin{proof}
The argument rests on the fact that automorphisms preserve subgraph orbits. Fix
$\phi \in \operatorname{Aut}(\mathcal{G})$ and a node $x \in \mathcal{V}$. Because
$\phi$ is adjacency-preserving, it maps the $k$-hop rooted subgraph
$\mathcal{S}^{(k)}(x)$ isomorphically onto $\mathcal{S}^{(k)}(\phi(x))$; a
message-passing GNN is invariant to such rooted-subgraph isomorphisms, so
$\phi_w\big(\mathcal{S}^{(k)}(x)\big) = \phi_w\big(\mathcal{S}^{(k)}(\phi(x))\big)$
for all $w$. By \Cref{dfn:subgraph-orbit} this gives $x \sim \phi(x)$, i.e.
\begin{equation}
\label{eq:subgraph-orbit-invariance}
\mathcal{O}_{\mathcal{S}}(\phi(x)) = \mathcal{O}_{\mathcal{S}}(x)
\qquad \text{for all } x \in \mathcal{V}.
\end{equation}
Since $\mathcal{O}_{\mathcal{E}}$ is the multiset of the endpoints' subgraph
orbits, applying \Cref{eq:subgraph-orbit-invariance} to both endpoints of
$e = \{u,v\}$ yields
\[
\mathcal{O}_{\mathcal{E}}(\phi(u), \phi(v))
= \multiset{\mathcal{O}_{\mathcal{S}}(\phi(u)),\,\mathcal{O}_{\mathcal{S}}(\phi(v))}
= \multiset{\mathcal{O}_{\mathcal{S}}(u),\,\mathcal{O}_{\mathcal{S}}(v)}
= \mathcal{O}_{\mathcal{E}}(u, v),
\]
which proves the claim. Conversely, if $\phi \notin \operatorname{Aut}(\mathcal{G})$
then $\phi$ need not preserve the enclosing subgraphs, so
\Cref{eq:subgraph-orbit-invariance} can fail and the two multisets may differ;
indeed $\phi(u),\phi(v)$ may not even form an edge. Hence the assignment is
equivariant exactly on $\operatorname{Aut}(\mathcal{G})$.
\end{proof}
\endgroup

\subsection{Detailed Analysis of Theorem 4.1}
\label{subsec:theorem 1}
We begin by introducing key concepts to formalize subgraph orbit-distinguishability. We then define the Node Automorphism Problem and show that the proposed D1 provides a solution.
\begin{definition}[$\epsilon$-sufficient]
\label{dfn:sufficient alteration}
We define an alteration as \emph{$\epsilon$-sufficient} if, for all subgraph $\mathcal{S}^{(k)} \in \mathcal{G}$, there exists an altered version $\mathcal{S}_\epsilon^{(k)}$, such that 
$\langle \phi_w(\mathcal{S}_\epsilon^{(k)}(u)), \phi_w(\mathcal{S}^{(k)}(u))\rangle) > 0$ under GNN encoder $\phi_w$ and decoder $D$. $\epsilon$ denotes the magnitude of change in the embedding space resulting from a structural perturbation in the graph.
\end{definition}
\begin{definition}[Node Automorphism Problem]
\label{dfn:automorphic-node-problem} 
A GNN is said to be \emph{$\epsilon$-complete} if there exists an \emph{$\epsilon$-sufficient alteration} such that for all ${u \in \mathcal{V}}$, the following holds: {\small
\begin{equation}
D\left(\phi_w(\mathcal{S}_\epsilon^{(k)}(u)), \phi_w(\mathcal{S}^{(k)}(u)) \right) \geq \epsilon
\end{equation}}
Automorphic nodes are defined as nodes that are $\epsilon$-incomplete under the action of $(\phi, D)_w$.
\end{definition}
We reformulate the Theorem as follows:
\begin{proof}
\label{app:proof-thm1}
For any $\epsilon > 0$ with probability $p_\epsilon$, we expect one bounded $\mathbb {L}$ exists, such that for all $l\geq \mathbb{L}$, $ \mathcal{S'} = \mathcal{B(S)} $
\begin{equation}
    \left(\phi_{w_1}(\mathcal{S}), \ldots, \phi_{w_\ell}(\mathcal{S})\right) \neq \left(\phi_{w_1}(\mathcal{S'}), \ldots, \phi_{w_\ell}(  \mathcal{S'})\right)
\end{equation}
holds with probability $1-\sigma$ , where random variables $w \sim \mathcal{W}$ .
It particularly implies that one closure exists $\mathcal{B(S)}_\epsilon$  with $\Pr(\mathcal{S} \in \mathcal{B(S)}_\epsilon) =  \Pr_{\mathcal{S}, \mathcal{S'}}=1 - (1-p)^{|\mathcal{E}|}> 0$, such that for all subgraph orbits  $\phi_\mathcal{B}(\mathcal{S})  \neq \phi_\mathcal{B}(\mathcal{S'})$
\\ 
We need $ \Pr\left( \exists i \in \{1, \ldots, \ell\} : \mathcal{S} \in \mathcal{S'} \right) \geq 1 - \sigma$, hence
\begin{equation}
     1 - (1 - \Pr_{\mathcal{S, S'}})^\ell \geq 1 - \sigma 
\end{equation}
must hold. Solving for $\ell$, we find that 
\begin{equation}
    \mathbb{L} = \left\lceil \frac{\log(1/\sigma)}{\log\left( \frac{1}{1 - \Pr_{\mathcal{S},\mathcal{S'}}} \right)} \right\rceil  
    \label{equ: L1}
\end{equation}

is sufficient to guarantee that there will be at least one $\mathcal{S}$ in $\mathcal{S'}$ with probability at least $1 - \sigma$, implying $\phi_{w}(\mathcal{S}) \neq \phi_{w}(\mathcal{S}')$, which closes the proof.
\end{proof}
We further simplify the   \Cref{equ: L1} to provide a qualitative analysis for $\mathbb{L} $ w.r.t $p$ and $|\mathcal{E}|$:
\begin{equation}
    \mathbb{L} = \left\lceil \frac{\log(1/\sigma)}{\log\left( \frac{1}{1 - \Pr_{\mathcal{S},\mathcal{S'}}} \right)} \right\rceil
\end{equation}

Given that
\begin{equation}
    \Pr_{\mathcal{S},\mathcal{S'}} = 1 - (1-p)^{|\mathcal{E}|},
\end{equation}
we substitute into the definition of $\mathbb{L}$:
\begin{equation}
\begin{aligned}
    \mathbb{L} &= \left\lceil \frac{\log(1/\sigma)}{\log\left( \frac{1}{1 - (1 - (1-p)^{|\mathcal{E}|})} \right)} \right\rceil \\
    &= \left\lceil \frac{\log(1/\sigma)}{\log\left( \frac{1}{(1-p)^{|\mathcal{E}|}} \right)} \right\rceil.
\end{aligned}
\end{equation}

setting in $\log(1/x) = -\log(x)$, we have
\begin{equation}
    \log\left( \frac{1}{(1-p)^{|\mathcal{E}|}} \right) = -\log\left( (1-p)^{|\mathcal{E}|} \right).
\end{equation}

Applying the power rule for logarithms $\log(x^n) = n \log(x)$, we obtain
\begin{equation}
    -\log\left( (1-p)^{|\mathcal{E}|} \right) = -|\mathcal{E}| \log(1-p).
\end{equation}

Thus,
\begin{equation}
    \mathbb{L} = \left\lceil \frac{\log(1/\sigma)}{-|\mathcal{E}|\log(1-p)} \right\rceil.
\end{equation}

Since $\log(1/\sigma) = -\log(\sigma)$, it further simplifies to
\begin{equation}
    \mathbb{L} = \left\lceil \frac{-\log(\sigma)}{-|\mathcal{E}|\log(1-p)} \right\rceil = \left\lceil \frac{\log(\sigma)}{|\mathcal{E}|\log(1-p)} \right\rceil.
\end{equation}

\paragraph{Final simplified form:}
\begin{equation}
    \boxed{ \mathbb{L} = \left\lceil \frac{\log(\sigma)}{|\mathcal{E}|\log(1-p)} \right\rceil }
\end{equation}

\paragraph{Qualitative Analysis:}
Since $\log(1-p) < 0$ for $p \in (0,1)$ and $\log(\sigma) < 0$ for $\sigma \in (0,1)$, the ratio $\frac{\log(\sigma)}{|\mathcal{E}|\log(1-p)}$ is positive. 

Furthermore, as $p$ increases $\mathbb{L}$ decreases; but when $|\mathcal{E}|$ is fixed and the graph becomes larger,  Suppose $|\mathcal{E}|$ is fixed, but the graph becomes larger (i.e., the number of nodes increases).  
Since
\begin{equation}
    (1-p)^{|\mathcal{E}|} \approx 1 - p|\mathcal{E}|,
\end{equation}
for small $p$, we apply first-order Taylor approximation, then have 
\begin{equation}
    \Pr_{\mathcal{S},\mathcal{S'}} \approx p|\mathcal{E}|.
\end{equation}
Thus, to keep $\Pr_{\mathcal{S},\mathcal{S'}}$ approximately constant as the graph grows, $p$ should satisfy
\begin{equation}
    p \propto \frac{1}{|\mathcal{E}|}.
\end{equation}
\emph{the per-edge probability $p$} should be increased proportionally to maintain a stable success probability and avoid increasing $\mathbb{L}$ excessively.


\subsection{Detailed Analysis of Theorem 4.2}
\label{app:app-automorphism-dominant}
\begin{theorem}

Let $ G $ be a lobster graph containing an automorphism-dominant node $ v_a \) and let $ \phi(\cdot) $ denote the embedding function of a $ K \)-layer GNN. In standard GNNs, the average embedding distance between $ v_a $ and its automorphic neighbors is greater than that between $ v_a $ and its non-automorphic neighbors:
\begin{equation}
\label{app:embedding-gap}
\begin{aligned}
&\mathbb{E}_{u \in \mathcal{N}(v)} \left[ \| \phi(u) - \phi(v) \|^2 \mid \mathcal{O}_u \neq \mathcal{O}_{v} \right]  \\
&\quad - \mathbb{E}_{u \in \mathcal{N}(v_a)} \left[ \| \phi(u) - \phi(v_a) \|^2 \mid \mathcal{O}_u = \mathcal{O}_{v_a} \right]  > \delta
\end{aligned}
\end{equation}
for some $ \delta > 0 \). A GNN equipped with D2 reduces this gap caomparing to standard GNN without.
\end{theorem}

We start with a concrete example: lobster graph illustrated in \Cref{fig:lobster}
\begin{lemma}
Let $ \mathcal{G} = (\mathcal{V}, \mathcal{E}) $ be an undirected graph with $ |\mathcal{V}| =10  $ nodes and let $ \mathbf{X} \in \mathbb{R}^{10 \times 1} $ be the node feature matrix where all node features are identical, i.e.,
\[
\mathbf{X} = \mathbf{1}_N \cdot \mathbf{x}_0^\top
\]
for some fixed vector $ \mathbf{x}_0 \in \mathbb{R}^d \). Let $ \hat{\mathbf{A}} \in \mathbb{R}^{N \times N} $ be the normalized adjacency matrix used in a one-layer GCN. Then the output embedding $ \mathbf{H} \in \mathbb{R}^{N \times d'} $ satisfies:
\[
\mathbf{H} = \hat{\mathbf{A}} \mathbf{X} \mathbf{W} = \mathbf{c} \cdot (\mathbf{x}_0^\top \mathbf{W}),
\]
where $ \mathbf{c} \in \mathbb{R}^{N} $ is a vector with entries $ c_i = \sum_{j=1}^N \hat{A}_{ij} \) and $ \mathbf{W} \in \mathbb{R}^{d \times d'} $ is the trainable weight matrix.
\end{lemma}

\begin{proof}
By the standard GCN layer definition, we have:
\[
\mathbf{H} = \hat{\mathbf{A}} \mathbf{X} \mathbf{W}.
\]
We replace $ \mathbf{X} = \mathbf{1}_N \cdot \mathbf{x}_0^\top \), then:
\[
\hat{\mathbf{A}} \mathbf{X} = \hat{\mathbf{A}} (\mathbf{1}_N \cdot \mathbf{x}_0^\top) = (\hat{\mathbf{A}} \mathbf{1}_N) \cdot \mathbf{x}_0^\top.
\]
We represent $\mathbf{c}$ in matrix form  $ \mathbf{c} = \hat{\mathbf{A}} \mathbf{1}_N \). Then:
\[
\hat{\mathbf{A}} \mathbf{X} = \mathbf{c} \cdot \mathbf{x}_0^\top.
\]
Therefore:
\[
\mathbf{H} = (\hat{\mathbf{A}} \mathbf{X}) \mathbf{W} = (\mathbf{c} \cdot \mathbf{x}_0^\top) \mathbf{W} = \mathbf{c} \cdot (\mathbf{x}_0^\top \mathbf{W}).
\]
\end{proof}
We see that three vertices are automorphic; a special case is when two nodes $ u, v \in \mathcal{V} $ have identical rows in $ \hat{\mathbf{A}} $ in the lobster graph. We characterize such intuition as $ \hat{A}_{u,:} = \hat{A}_{v,:} \). Assuming we have no dropout, injected noise and all weights $\mathbf{W}$ are shared, we have the following deviation, substituting $ \mathbf{X} = \mathbf{1}_N \cdot \mathbf{x}_0^\top \):
\[
\mathbf{H} = \hat{\mathbf{A}} (\mathbf{1}_N \cdot \mathbf{x}_0^\top) \mathbf{W}.
\]
\[
\mathbf{H} = (\hat{\mathbf{A}} \mathbf{1}_N) \cdot (\mathbf{x}_0^\top \mathbf{W}).
\]
Define $ \mathbf{z} = \hat{\mathbf{A}} \mathbf{1}_N \in \mathbb{R}^{N \times 1} $ and $ \mathbf{v} = \mathbf{x}_0^\top \mathbf{W} \in \mathbb{R}^{1 \times d'} \), $
\mathbf{H} = \mathbf{z} \cdot \mathbf{v}$
which implies that each row of $ \mathbf{h} $ is given by:
\[
\mathbf{h}_i = z_i \cdot \mathbf{v} = \sum_{j=1}^N \hat{A}_{ij} \mathbf{v} = k \mathbf{D}_i \mathbf{v}.
\]
We now compute the two terms in  ~\Cref{app:embedding-gap} to illustrate the embedding difference between automorphic and non-automorphic neighbors of an automorphism-dominant node.
\begin{equation}
\begin{aligned}
&\mathbb{E}_{u \in \mathcal{N}(v)} \left[ \| \phi(u) - \phi(v) \|^2 \mid \mathcal{O}_u \neq \mathcal{O}_{v} \right]  \\
&\quad  = \frac{k (\mathbf{D}_u - \mathbf{D}_v)\mathbf{v}}{N_{nauto}} \\
&\quad  = 2/7 = 0.2857
\end{aligned}
\end{equation}
where $ \mathbf{D}_u $ and $ \mathbf{D}_v $ represent the degrees of the neighboring and center nodes respectively, $ \mathbf{v} $ is the learned weight vector and $ N_{\text{non-auto}} = 7 $ is the number of non-automorphic neighbors. This reflects the average squared embedding distance due to degree gaps, most automorphic vertices have identical degrees except two tail nodes.

Next, we calculate the expected distance among automorphic neighbors:

\begin{equation}
\begin{aligned}
&\mathbb{E}_{u \in \mathcal{N}(v_a)} \left[ \| \phi(u) - \phi(v_a) \|^2 \mid \mathcal{O}_u = \mathcal{O}_{v} \right]  \\
&\quad  = \frac{k (\mathbf{D}_u - \mathbf{D}_v)\mathbf{v}}{N_{nauto}} \\
&\quad  = 1+3+3+3/4 \\
&\quad  = 2.5
\end{aligned}
\end{equation}

where the degree difference is 3 for automorphic neighbors (e.g., between vertex 1 and other symmetric vertices) and 1 for the tail node.
These results indicate that standard GNNs assign larger embedding differences to automorphic edges compared to non-automorphic ones—contradicting structural automorphism. This directly supports ~\Cref{thm:automorphism-dominant}.

We have now 
\begin{equation}
\begin{aligned}
&\mathbb{E}_{u \in \mathcal{N}(v_a)} \left[ \| \phi(u) - \phi(v_a) \|^2 \mid \mathcal{O}_u = \mathcal{O}_{v} \right]  < 0.2857\\
&\quad  = \frac{k (\mathbf{D}_u - \mathbf{D}_v)\mathbf{v}}{N_{nauto}} \\
&\quad  = 1w_1+3w_2+3w_3+0.75w_4  < 0.2857 \\
\end{aligned}
\end{equation}
We provide one valid solution to the inequality using negative weights:
\[
w_1 = 0.1, \quad w_2 = 0.05, \quad w_3 = -0.1, \quad w_4 = 0.4
\]
Substituting into the left-hand side of the inequality:
\[
1w_1 + 3w_2 + 3w_3 + 0.75w_4 
\]
\[= 1(0.1) + 3(0.05) + 3(-0.1) + 0.75(0.4) \]
\[= 0.1 + 0.15 - 0.3 + 0.3 = 0.25 < 0.2857\]
Thus, this set of weights satisfies the inequality.


\section{Preliminaries in Detail}
\label{app:automorphism}

\section{Synthetic Datasets: Details}
\label{app:synthetic}
In this section, we present our data generation process, dataset details, the sources of our baselines and the corresponding hyperparameters.

\subsection{Detailed Results on Semi-Synthetic Benchmarks}
\label{app:synthetic-results}

\textbf{Automorphic Real Synthetic Graph Construction} We utilize a real world graph $\mathcal{H}$ (Cora and Citeseer). Then we form a larger graph $\mathcal{H}^2$ that contains two disjoint copies of $\mathcal{H}$, along with 1000 uniformly-randomly added edges (both between (inter) and within (intra) copies of $\mathcal{H}$). Without the random edges, each node in one copy of $\mathcal{H}$ is automorphic to the corresponding node in the other copy, as 
increasing inter, intra probability and number of edges, $\EAR\xspace$ increases. We use three parameters inter $p_i$, intra $p_s$, introduced automorphic nodes $a$ and number of perturbed edges $\vert \mathcal{E} \vert$ to control the $\EAR\xspace$. The utilized parameters for Syn-Cora \& Syn-Citeseer are provided in \Cref{tab:A-synthetic-stats}.

\begin{figure}[h]
    \centering
        \includegraphics[width=0.28\linewidth]{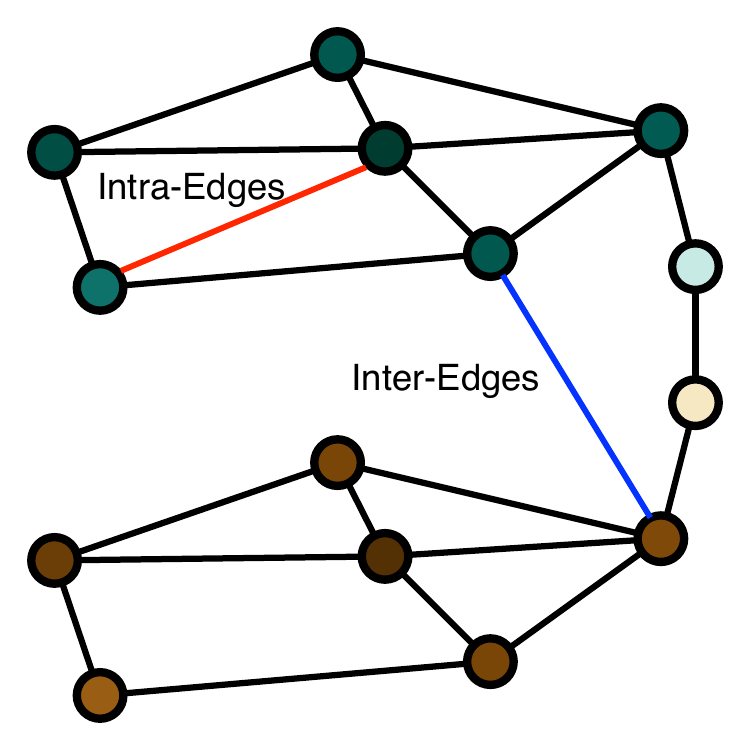}
        \caption{Syn-Cora \& Syn-Citeseer}
    \caption{Symmetric automorphic graph construction to change the automorphism in real world graph.}
    \label{fig:edge_rewiring_tri}
\end{figure}

\begin{table}[H]
	\centering
	\caption{Statistics for Synthetic Datasets}
	\label{tab:A-synthetic-stats}
	{\footnotesize
	\begin{tabular}{>{\raggedright\arraybackslash}p{2.5cm} 
	                >{\raggedright\arraybackslash}p{2.5cm} 
	                >{\raggedright\arraybackslash}p{2.5cm} 
	                >{\raggedright\arraybackslash}p{2.5cm}}
		\toprule 
		\textbf{Benchmark Name} & \texttt{syn-cora} & \texttt{syn-citeseer} \\
		\midrule
		\textbf{\# Nodes} & 5416 & 6654 \\
		\textbf{\# Edges} & 10556 to 17549 & 9104 to 23098 \\
		\textbf{automorphism} $\EAR\xspace$  & [0, 0.1, $\ldots$, 0.9] & [0, 0.1, $\ldots$, 0.9] \\
		\textbf{Degree Range} & 1 to 168 & 1 to 107 \\
		\textbf{Average Degree} & 3.89 to 6.48 & 2.73 to 6.94 \\
		\midrule
		\textbf{Inter Prob.} & 0.1 & 0.1 \\ 
		\textbf{Intra Prob.} & 0.5 & 0.5 \\
		\textbf{Number of Edges} & {[0.2, 1, 4, 7, 12, 18, 20, 28]$\times$250} & {[0.2, 1, 2, 3, 4, 5, 7, 8, 10, 14]$\times$1000} \\ 
        \midrule
		\textbf{$\EAR\xspace$} & [0.19, 0.30, 0.37, 0.51, 0.60, 0.71, 0.81, 0.83, 0.90] & [0.10, 0.18, 0.28, 0.38, 0.48, 0.57, 0.65, 0.77, 0.81, 0.88, 0.94] \\
		\bottomrule
	\end{tabular}
	}
\end{table}


\begin{table}[h]
    \centering
    \caption{\textsc{syn-cora}: Mean metrics (MRR) and standard deviation for each method on synthetic datasets with varying \metrics ratios $\EAR\xspace$. }
    \label{tab:app-syn-cora-results}
    \resizebox{\linewidth}{!}{\scriptsize
    \begin{tabular}{lcccccccccc}
        \toprule
        Method & 0.19 & 0.30 & 0.37 & 0.51 & 0.60 & 0.71 & 0.81 & 0.83 & 0.90 \\
        \midrule
        GCN & $99.41 {\scriptstyle \pm .23}$ & $99.39 {\scriptstyle \pm .09}$ & $99.06 {\scriptstyle \pm .22}$ & $98.51 {\scriptstyle \pm .35}$ & $98.01 {\scriptstyle \pm .45}$ & $97.27 {\scriptstyle \pm .42}$ & $96.45 {\scriptstyle \pm .57}$ & $96.29 {\scriptstyle \pm .67}$ & $94.07 {\scriptstyle \pm .86}$ \\
        GAT & $99.28 {\scriptstyle \pm .16}$ & $98.97 {\scriptstyle \pm .23}$ & $98.99 {\scriptstyle \pm .13}$ & $98.75 {\scriptstyle \pm .20}$ & $98.48 {\scriptstyle \pm .25}$ & $97.98 {\scriptstyle \pm .22}$ & $97.72 {\scriptstyle \pm .50}$ & $97.51 {\scriptstyle \pm .50}$ & $96.42 {\scriptstyle \pm .39}$ \\
        GIN & $91.83 {\scriptstyle \pm .70}$ & $95.76 {\scriptstyle \pm .78}$ & $95.53 {\scriptstyle \pm .59}$ & $94.91 {\scriptstyle \pm .90}$ & $84.81 {\scriptstyle \pm .14}$ & $82.49 {\scriptstyle \pm .14}$ & $70.24 {\scriptstyle \pm .78}$ & $72.01 {\scriptstyle \pm .96}$ & $65.24 {\scriptstyle \pm .38}$ \\
        GraphSAGE & $97.91 {\scriptstyle \pm .62}$ & $97.96 {\scriptstyle \pm .57}$ & $97.96 {\scriptstyle \pm .60}$ & $97.05 {\scriptstyle \pm .85}$ & $95.93 {\scriptstyle \pm .93}$ & $93.38 {\scriptstyle \pm .99}$ & $90.27 {\scriptstyle \pm 1.93}$ & $90.37 {\scriptstyle \pm 2.24}$ & $85.82 {\scriptstyle \pm 2.03}$ \\
        MixHopGCN & $99.55 {\scriptstyle \pm .17}$ & $99.28 {\scriptstyle \pm .09}$ & $98.99 {\scriptstyle \pm .13}$ & $98.75 {\scriptstyle \pm .20}$ & $98.48 {\scriptstyle \pm .25}$ & $97.98 {\scriptstyle \pm .22}$ & $97.72 {\scriptstyle \pm .50}$ & $97.51 {\scriptstyle \pm .50}$ & $96.42 {\scriptstyle \pm .39}$ \\
        ChebGCN & $98.41 {\scriptstyle \pm .13}$ & $98.19 {\scriptstyle \pm .22}$ & $97.82 {\scriptstyle \pm .21}$ & $96.56 {\scriptstyle \pm .79}$ & $96.76 {\scriptstyle \pm .32}$ & $94.09 {\scriptstyle \pm .56}$ & $91.18 {\scriptstyle \pm 1.02}$ & $86.24 {\scriptstyle \pm 1.77}$ & $84.88 {\scriptstyle \pm 1.42}$ \\
        GCN-DW & $99.89 {\scriptstyle \pm .04}$ & $99.66 {\scriptstyle \pm .06}$ & $99.55 {\scriptstyle \pm .07}$ & $99.51 {\scriptstyle \pm .05}$ & $99.19 {\scriptstyle \pm .10}$ & $98.81 {\scriptstyle \pm .06}$ & $98.34 {\scriptstyle \pm .21}$ & $98.15 {\scriptstyle \pm .10}$ & $97.54 {\scriptstyle \pm .24}$ \\
        GCN-RF & $99.81 {\scriptstyle \pm .07}$ & $99.81 {\scriptstyle \pm .03}$ & $99.60 {\scriptstyle \pm .05}$ & $99.67 {\scriptstyle \pm .08}$ & $98.92 {\scriptstyle \pm .07}$ & $98.84 {\scriptstyle \pm .15}$ & $98.43 {\scriptstyle \pm .15}$ & $98.42 {\scriptstyle \pm .18}$ & $97.85 {\scriptstyle \pm .12}$ \\
        GCN-LAP     & $99.46 {\scriptstyle \pm .16}$ & $99.18 {\scriptstyle \pm .18}$ & $98.99 {\scriptstyle \pm .27}$ & $98.77 {\scriptstyle \pm .27}$ & $97.83 {\scriptstyle \pm .41}$ & $97.16 {\scriptstyle \pm .38}$ & $96.91 {\scriptstyle \pm .29}$ & $96.54 {\scriptstyle \pm .73}$ & $95.39 {\scriptstyle \pm .48}$ \\
        BUDDY & $95.34 {\scriptstyle \pm .17}$ & $94.00 {\scriptstyle \pm .23}$ & $95.03 {\scriptstyle \pm .17}$ & $95.34 {\scriptstyle \pm .04}$ & $94.82 {\scriptstyle \pm .18}$ & $94.63 {\scriptstyle \pm .30}$ & $94.06 {\scriptstyle \pm .18}$ & $93.94 {\scriptstyle \pm .10}$ & $93.80 {\scriptstyle \pm .10}$ \\
        LINKX & $97.30 {\scriptstyle \pm .32}$ & $96.99 {\scriptstyle \pm .15}$ & $96.81 {\scriptstyle \pm .18}$ & $94.28 {\scriptstyle \pm .15}$ & $92.82 {\scriptstyle \pm .13}$ & $89.63 {\scriptstyle \pm .29}$ & $86.69 {\scriptstyle \pm .15}$ & $85.47 {\scriptstyle \pm .12}$ & $83.48 {\scriptstyle \pm .18}$ \\
        Proposed w.o. D1   & $99.89 {\scriptstyle \pm .08}$ & $99.83 {\scriptstyle \pm .10}$ & $99.77 {\scriptstyle \pm .11}$ & $99.76 {\scriptstyle \pm .07}$ & $99.76 {\scriptstyle \pm .05}$ & $99.74 {\scriptstyle \pm .07}$ & $99.67 {\scriptstyle \pm .12}$ & $99.63 {\scriptstyle \pm .07}$ & $99.62 {\scriptstyle \pm .22}$ \\
        Proposed w.o D2    & $99.90 {\scriptstyle \pm .05}$ & $99.80 {\scriptstyle \pm .05}$ & $99.79 {\scriptstyle \pm .07}$ & $99.78 {\scriptstyle \pm .12}$ & $99.75 {\scriptstyle \pm .03}$ & $99.74 {\scriptstyle \pm .07}$ & $99.73 {\scriptstyle \pm .04}$ & $99.71 {\scriptstyle \pm .03}$ & $99.90 {\scriptstyle \pm .05}$ \\
        Proposed          & $\mathbf{99.94} {\scriptstyle \pm .03}$ & $\mathbf{99.86} {\scriptstyle \pm .03}$ & $\mathbf{99.83} {\scriptstyle \pm .07}$ & $\mathbf{99.81} {\scriptstyle \pm .00}$ & $\mathbf{99.80} {\scriptstyle \pm .05}$ & $\mathbf{99.78} {\scriptstyle \pm .04}$ & $\mathbf{99.76} {\scriptstyle \pm .03}$ & $\mathbf{99.74} {\scriptstyle \pm .07}$ & $\mathbf{99.93} {\scriptstyle \pm .08}$ \\
        \bottomrule
    \end{tabular}}
\end{table}

\begin{table}[h]
    \centering
    \caption{\textsc{syn-cora}: Mean Metrics (MRR) and standard deviation for each method on synthetic datasets with varying \metrics ratios $\EAR\xspace$. }
    \label{tab:app-syn-cora-mrr}
    \resizebox{\linewidth}{!}{\scriptsize
    \begin{tabular}{lcccccccccc}
        \toprule
        Method & 0.19 & 0.30 & 0.37 & 0.51 & 0.60 & 0.71 & 0.81 & 0.83 & 0.90 \\
        \midrule
        GCN & $42.28 {\scriptstyle \pm 11.15}$ & $34.09 {\scriptstyle \pm 8.75}$ & $27.06 {\scriptstyle \pm 5.83}$ & $21.29 {\scriptstyle \pm 6.17}$ & $20.83 {\scriptstyle \pm 3.29}$ & $18.37 {\scriptstyle \pm 3.90}$ & $16.50 {\scriptstyle \pm 3.50}$ & $14.78 {\scriptstyle \pm 4.48}$ & $12.82 {\scriptstyle \pm 3.91}$ \\
        GAT & $30.45 {\scriptstyle \pm 6.74}$ & $25.24 {\scriptstyle \pm 7.63}$ & $21.77 {\scriptstyle \pm 7.00}$ & $20.28 {\scriptstyle \pm 5.60}$ & $19.74 {\scriptstyle \pm 9.83}$ & $17.06 {\scriptstyle \pm 6.54}$ & $15.60 {\scriptstyle \pm 3.41}$ & $14.80 {\scriptstyle \pm 4.85}$ & $13.28 {\scriptstyle \pm 3.10}$ \\
        GIN & $15.80 {\scriptstyle \pm 2.59}$ & $15.57 {\scriptstyle \pm 5.25}$ & $14.43 {\scriptstyle \pm 3.99}$ & $10.88 {\scriptstyle \pm 4.81}$ & $9.19 {\scriptstyle \pm 3.93}$ & $6.87 {\scriptstyle \pm 4.00}$ & $4.42 {\scriptstyle \pm 3.89}$ & $4.06 {\scriptstyle \pm 3.72}$ & $3.07 {\scriptstyle \pm 3.30}$ \\
        GraphSAGE & $24.79 {\scriptstyle \pm 8.27}$ & $24.20 {\scriptstyle \pm 4.74}$ & $17.53 {\scriptstyle \pm 5.70}$ & $16.30 {\scriptstyle \pm 5.24}$ & $15.45 {\scriptstyle \pm 3.57}$ & $13.66 {\scriptstyle \pm 4.13}$ & $12.26 {\scriptstyle \pm 5.25}$ & $11.01 {\scriptstyle \pm 3.77}$ & $8.94 {\scriptstyle \pm 2.54}$ \\
        MixHopGCN & $65.87 {\scriptstyle \pm 14.64}$ & $58.62 {\scriptstyle \pm 11.24}$ & $17.65 {\scriptstyle \pm 29.31}$ & $0.12 {\scriptstyle \pm 0.00}$ & $0.12 {\scriptstyle \pm 0.00}$ & $0.12 {\scriptstyle \pm 0.00}$ & $0.12 {\scriptstyle \pm 0.00}$ & $0.12 {\scriptstyle \pm 0.00}$ & $0.12 {\scriptstyle \pm 0.00}$ \\
        ChebGCN & $28.71 {\scriptstyle \pm 4.29}$ & $24.71 {\scriptstyle \pm 5.39}$ & $21.38 {\scriptstyle \pm 5.21}$ & $21.01 {\scriptstyle \pm 4.68}$ & $19.82 {\scriptstyle \pm 5.16}$ & $16.75 {\scriptstyle \pm 4.98}$ & $15.35 {\scriptstyle \pm 3.71}$ & $10.26 {\scriptstyle \pm 2.32}$ & $9.77 {\scriptstyle \pm 2.48}$ \\
        GCN-DW & $72.87 {\scriptstyle \pm 12.57}$ & $43.98 {\scriptstyle \pm 7.79}$ & $29.52 {\scriptstyle \pm 4.48}$ & $40.14 {\scriptstyle \pm 5.64}$ & $53.73 {\scriptstyle \pm 11.90}$ & $35.98 {\scriptstyle \pm 3.61}$ & $42.41 {\scriptstyle \pm 6.86}$ & $\mathbf{50.71} {\scriptstyle \pm 9.77}$ & $21.16 {\scriptstyle \pm 3.57}$ \\
        GCN-RF & $52.38 {\scriptstyle \pm 13.90}$ & $48.97 {\scriptstyle \pm 6.06}$ & $44.79 {\scriptstyle \pm 10.38}$ & $61.81 {\scriptstyle \pm 9.20}$ & $27.26 {\scriptstyle \pm 3.79}$ & $29.51 {\scriptstyle \pm 7.91}$ & $44.57 {\scriptstyle \pm 11.12}$ & $24.89 {\scriptstyle \pm 4.14}$ & $34.98 {\scriptstyle \pm 10.87}$ \\
        GCN-LAP & $34.06 {\scriptstyle \pm 9.46}$ & $34.24 {\scriptstyle \pm 9.45}$ & $23.08 {\scriptstyle \pm 5.45}$ & $33.95 {\scriptstyle \pm 7.70}$ & $37.39 {\scriptstyle \pm 7.80}$ & $29.56 {\scriptstyle \pm 2.86}$ & $30.49 {\scriptstyle \pm 4.94}$ & $33.04 {\scriptstyle \pm 7.21}$ & $15.52 {\scriptstyle \pm 2.65}$ \\
        BUDDY & $33.41 {\scriptstyle \pm 5.10}$ & $23.80 {\scriptstyle \pm 3.71}$ & $32.42 {\scriptstyle \pm 1.96}$ & $23.09 {\scriptstyle \pm 3.11}$ & $19.55 {\scriptstyle \pm 1.05}$ & $27.87 {\scriptstyle \pm 7.83}$ & $31.36 {\scriptstyle \pm 4.70}$ & $27.51 {\scriptstyle \pm 1.46}$ & $18.71 {\scriptstyle \pm 1.04}$ \\
        LINKX & $32.99 {\scriptstyle \pm 9.10}$ & $23.91 {\scriptstyle \pm 0.00}$ & $19.90 {\scriptstyle \pm 0.00}$ & $18.60 {\scriptstyle \pm 0.00}$ & $17.44 {\scriptstyle \pm 0.00}$ & $15.08 {\scriptstyle \pm 0.00}$ & $14.99 {\scriptstyle \pm 0.00}$ & $11.67 {\scriptstyle \pm 3.83}$ & $10.86 {\scriptstyle \pm 0.00}$ \\
        Proposed w.o. D1 & $85.53 {\scriptstyle \pm 14.99}$ & $80.43 {\scriptstyle \pm 12.34}$ & $65.46 {\scriptstyle \pm 4.98}$ & $62.64 {\scriptstyle \pm 20.27}$ & $60.10 {\scriptstyle \pm 10.21}$ & $49.77 {\scriptstyle \pm 12.52}$ & $41.64 {\scriptstyle \pm 3.52}$ & $38.49 {\scriptstyle \pm 12.70}$ & $34.89 {\scriptstyle \pm 5.32}$ \\
        Proposed w.o D2 & $\mathbf{97.42} {\scriptstyle \pm 1.74}$ & $58.13 {\scriptstyle \pm 11.50}$ & $54.20 {\scriptstyle \pm 17.44}$ & $49.86 {\scriptstyle \pm 7.44}$ & $49.43 {\scriptstyle \pm 10.93}$ & $45.12 {\scriptstyle \pm 5.11}$ & $38.35 {\scriptstyle \pm 4.47}$ & $35.64 {\scriptstyle \pm 4.88}$ & $32.47 {\scriptstyle \pm 7.89}$ \\
        Proposed & $90.78 {\scriptstyle \pm 7.34}$ & $\mathbf{84.41} {\scriptstyle \pm 1.87}$ & $\mathbf{66.53} {\scriptstyle \pm 14.99}$ & $\mathbf{64.51} {\scriptstyle \pm 20.85}$ & $\mathbf{62.30} {\scriptstyle \pm 10.83}$ & $\mathbf{57.83} {\scriptstyle \pm 2.04}$ & $\mathbf{52.14} {\scriptstyle \pm 13.77}$ & $47.36 {\scriptstyle \pm 6.13}$ & $\mathbf{45.83} {\scriptstyle \pm 3.30}$ \\
        \bottomrule
    \end{tabular}}
\end{table}

\begin{table}[h]
    \centering
    \caption{\textsc{syn-citeseer}: Mean metrics (AUC) and standard deviation for each method on synthetic datasets with varying \metrics ratios $\EAR\xspace$ (1/2, $\EAR\xspace \in [0.10, 0.57]$). The best method per column is bolded.}
    \label{tab:app-syn-citeseer-results}
    \resizebox{\linewidth}{!}{\scriptsize
    \begin{tabular}{lcccccc}
        \toprule
         $\EAR\xspace$ & 0.10 & 0.18 & 0.28 & 0.38 & 0.48 & 0.57 \\
        \midrule
        GCN & $99.65 {\scriptstyle \pm 0.08}$ & $99.50 {\scriptstyle \pm 0.14}$ & $99.42 {\scriptstyle \pm 0.16}$ & $98.92 {\scriptstyle \pm 0.24}$ & $98.61 {\scriptstyle \pm 0.33}$ & $98.03 {\scriptstyle \pm 0.22}$ \\
        GAT & $99.42 {\scriptstyle \pm 0.14}$ & $98.94 {\scriptstyle \pm 0.32}$ & $99.17 {\scriptstyle \pm 0.14}$ & $98.28 {\scriptstyle \pm 0.71}$ & $98.31 {\scriptstyle \pm 0.74}$ & $97.89 {\scriptstyle \pm 1.07}$ \\
        GIN & $95.43 {\scriptstyle \pm 1.10}$ & $95.40 {\scriptstyle \pm 0.71}$ & $94.69 {\scriptstyle \pm 1.14}$ & $92.73 {\scriptstyle \pm 0.79}$ & $90.53 {\scriptstyle \pm 1.34}$ & $89.78 {\scriptstyle \pm 1.16}$ \\
        GraphSAGE & $96.19 {\scriptstyle \pm 2.00}$ & $95.99 {\scriptstyle \pm 1.77}$ & $94.18 {\scriptstyle \pm 1.73}$ & $93.26 {\scriptstyle \pm 3.49}$ & $91.69 {\scriptstyle \pm 2.63}$ & $91.81 {\scriptstyle \pm 2.91}$ \\
        MixHopGCN & $99.69 {\scriptstyle \pm 0.06}$ & $94.41 {\scriptstyle \pm 1.60}$ & $99.03 {\scriptstyle \pm 0.25}$ & $98.77 {\scriptstyle \pm 0.07}$ & $98.22 {\scriptstyle \pm 0.27}$ & $97.99 {\scriptstyle \pm 0.23}$ \\
        ChebGCN & $97.59 {\scriptstyle \pm 0.46}$ & $97.76 {\scriptstyle \pm 0.41}$ & $97.20 {\scriptstyle \pm 0.44}$ & $94.90 {\scriptstyle \pm 0.58}$ & $92.50 {\scriptstyle \pm 1.31}$ & $88.75 {\scriptstyle \pm 1.86}$ \\
        GCN-DW & $99.86 {\scriptstyle \pm 0.06}$ & $99.80 {\scriptstyle \pm 0.08}$ & $99.60 {\scriptstyle \pm 0.10}$ & $99.37 {\scriptstyle \pm 0.12}$ & $99.09 {\scriptstyle \pm 0.16}$ & $98.58 {\scriptstyle \pm 0.23}$ \\
        GCN-RF & $99.87 {\scriptstyle \pm 0.04}$ & $99.69 {\scriptstyle \pm 0.06}$ & $99.75 {\scriptstyle \pm 0.06}$ & $99.60 {\scriptstyle \pm 0.08}$ & $99.39 {\scriptstyle \pm 0.09}$ & $99.27 {\scriptstyle \pm 0.11}$ \\
        GCN-LAP & $99.79 {\scriptstyle \pm 0.07}$ & $99.65 {\scriptstyle \pm 0.09}$ & $99.35 {\scriptstyle \pm 0.19}$ & $99.04 {\scriptstyle \pm 0.13}$ & $98.78 {\scriptstyle \pm 0.14}$ & $98.23 {\scriptstyle \pm 0.23}$ \\
        BUDDY & $95.20 {\scriptstyle \pm 0.39}$ & $95.83 {\scriptstyle \pm 0.18}$ & $95.84 {\scriptstyle \pm 0.26}$ & $94.37 {\scriptstyle \pm 0.53}$ & $95.10 {\scriptstyle \pm 0.51}$ & $95.79 {\scriptstyle \pm 0.48}$ \\
        LINKX & $97.30 {\scriptstyle \pm 0.08}$ & $96.99 {\scriptstyle \pm 0.14}$ & $96.81 {\scriptstyle \pm 0.16}$ & $94.28 {\scriptstyle \pm 0.24}$ & $92.82 {\scriptstyle \pm 0.33}$ & $89.63 {\scriptstyle \pm 0.22}$ \\
        Proposed & $\mathbf{99.94} {\scriptstyle \pm 0.03}$ & $\mathbf{99.86} {\scriptstyle \pm 0.03}$ & $\mathbf{99.83} {\scriptstyle \pm 0.07}$ & $\mathbf{99.81} {\scriptstyle \pm 0.00}$ & $\mathbf{99.80} {\scriptstyle \pm 0.05}$ & $\mathbf{99.78} {\scriptstyle \pm 0.04}$ \\
        \bottomrule
    \end{tabular}}
\end{table}

\begin{table}[h]
    \centering
    \caption{\textsc{syn-citeseer}: Mean metrics (AUC) and standard deviation for each method on synthetic datasets with varying \metrics ratios $\EAR\xspace$ (2/2, $\EAR\xspace \in [0.65, 0.94]$). The best method per column is bolded.}
    \label{tab:app-syn-citeseer-results-b}
    \resizebox{\linewidth}{!}{\scriptsize
    \begin{tabular}{lccccc}
        \toprule
         $\EAR\xspace$ & 0.65 & 0.77 & 0.81 & 0.88 & 0.94 \\
        \midrule
        GCN & $97.65 {\scriptstyle \pm 0.47}$ & $96.32 {\scriptstyle \pm 0.56}$ & $96.49 {\scriptstyle \pm 0.74}$ & $95.03 {\scriptstyle \pm 1.60}$ & $91.81 {\scriptstyle \pm 2.05}$ \\
        GAT & $97.81 {\scriptstyle \pm 0.41}$ & $96.56 {\scriptstyle \pm 1.26}$ & $95.87 {\scriptstyle \pm 1.41}$ & $93.11 {\scriptstyle \pm 4.54}$ & $91.93 {\scriptstyle \pm 2.09}$ \\
        GIN & $86.46 {\scriptstyle \pm 2.00}$ & $78.42 {\scriptstyle \pm 1.10}$ & $82.29 {\scriptstyle \pm 1.19}$ & $75.34 {\scriptstyle \pm 1.54}$ & $61.25 {\scriptstyle \pm 1.04}$ \\
        GraphSAGE & $86.72 {\scriptstyle \pm 4.49}$ & $85.62 {\scriptstyle \pm 5.20}$ & $82.41 {\scriptstyle \pm 6.02}$ & $77.59 {\scriptstyle \pm 3.97}$ & $74.68 {\scriptstyle \pm 5.96}$ \\
        MixHopGCN & $96.28 {\scriptstyle \pm 0.44}$ & $93.72 {\scriptstyle \pm 0.76}$ & $94.55 {\scriptstyle \pm 0.52}$ & $91.21 {\scriptstyle \pm 1.95}$ & $88.35 {\scriptstyle \pm 3.18}$ \\
        ChebGCN & $84.72 {\scriptstyle \pm 1.89}$ & $79.21 {\scriptstyle \pm 1.42}$ & $76.92 {\scriptstyle \pm 1.49}$ & $72.16 {\scriptstyle \pm 2.07}$ & $68.81 {\scriptstyle \pm 2.76}$ \\
        GCN-DW & $98.45 {\scriptstyle \pm 0.17}$ & $97.92 {\scriptstyle \pm 0.20}$ & $97.83 {\scriptstyle \pm 0.23}$ & $97.36 {\scriptstyle \pm 0.24}$ & $96.23 {\scriptstyle \pm 0.24}$ \\
        GCN-RF & $99.27 {\scriptstyle \pm 0.13}$ & $98.46 {\scriptstyle \pm 0.15}$ & $98.48 {\scriptstyle \pm 0.19}$ & $98.16 {\scriptstyle \pm 0.20}$ & $97.01 {\scriptstyle \pm 0.37}$ \\
        GCN-LAP & $98.29 {\scriptstyle \pm 0.29}$ & $97.81 {\scriptstyle \pm 0.25}$ & $97.51 {\scriptstyle \pm 0.24}$ & $97.16 {\scriptstyle \pm 0.25}$ & $95.80 {\scriptstyle \pm 0.45}$ \\
        BUDDY & $96.03 {\scriptstyle \pm 0.32}$ & $94.86 {\scriptstyle \pm 0.42}$ & $94.70 {\scriptstyle \pm 0.31}$ & $95.50 {\scriptstyle \pm 0.39}$ & $95.57 {\scriptstyle \pm 0.41}$ \\
        LINKX & $86.69 {\scriptstyle \pm 0.47}$ & $85.47 {\scriptstyle \pm 0.56}$ & $83.48 {\scriptstyle \pm 0.74}$ & $81.28 {\scriptstyle \pm 1.42}$ & $79.68 {\scriptstyle \pm 1.49}$ \\
        Proposed & $\mathbf{99.76} {\scriptstyle \pm 0.03}$ & $\mathbf{99.74} {\scriptstyle \pm 0.07}$ & $99.74 {\scriptstyle \pm 0.075}$ & $99.74 {\scriptstyle \pm 0.0775}$ & $\mathbf{99.90} {\scriptstyle \pm 0.08}$ \\
        \bottomrule
    \end{tabular}}
\end{table}

\begin{table}[h]
    \centering
    \caption{\textsc{syn-citeseer}: Mean MRR and standard deviation for each method on synthetic datasets with varying \metrics ratios $\EAR\xspace$ (1/2, $\EAR\xspace \in [0.10, 0.57]$). The best method per column is bolded.}
    \label{tab:app-syn-citeseer-mrr}
    \resizebox{\linewidth}{!}{\scriptsize
    \begin{tabular}{lcccccc}
        \toprule
         $\EAR\xspace$ & 0.10 & 0.18 & 0.28 & 0.38 & 0.48 & 0.57 \\
        \midrule
        GCN & $49.84 {\scriptstyle \pm 4.78}$ & $48.61 {\scriptstyle \pm 13.29}$ & $48.34 {\scriptstyle \pm 10.77}$ & $43.70 {\scriptstyle \pm 10.66}$ & $39.74 {\scriptstyle \pm 9.66}$ & $27.99 {\scriptstyle \pm 7.70}$ \\
        GAT & $31.93 {\scriptstyle \pm 8.23}$ & $27.72 {\scriptstyle \pm 7.61}$ & $25.82 {\scriptstyle \pm 4.77}$ & $24.61 {\scriptstyle \pm 6.22}$ & $23.06 {\scriptstyle \pm 3.30}$ & $21.82 {\scriptstyle \pm 8.21}$ \\
        GIN & $20.37 {\scriptstyle \pm 3.99}$ & $18.09 {\scriptstyle \pm 4.48}$ & $15.77 {\scriptstyle \pm 2.23}$ & $14.17 {\scriptstyle \pm 3.56}$ & $11.37 {\scriptstyle \pm 2.29}$ & $10.79 {\scriptstyle \pm 2.36}$ \\
        GraphSAGE & $24.01 {\scriptstyle \pm 15.47}$ & $20.28 {\scriptstyle \pm 12.22}$ & $20.13 {\scriptstyle \pm 10.96}$ & $16.45 {\scriptstyle \pm 4.91}$ & $15.82 {\scriptstyle \pm 5.18}$ & $12.57 {\scriptstyle \pm 5.13}$ \\
        MixHopGCN & $70.99 {\scriptstyle \pm 27.71}$ & $52.17 {\scriptstyle \pm 4.79}$ & $51.84 {\scriptstyle \pm 9.30}$ & $43.61 {\scriptstyle \pm 6.49}$ & $39.02 {\scriptstyle \pm 2.52}$ & $36.87 {\scriptstyle \pm 10.14}$ \\
        ChebGCN & $28.40 {\scriptstyle \pm 6.49}$ & $27.45 {\scriptstyle \pm 5.14}$ & $27.07 {\scriptstyle \pm 10.19}$ & $23.05 {\scriptstyle \pm 3.27}$ & $22.76 {\scriptstyle \pm 4.13}$ & $15.23 {\scriptstyle \pm 4.23}$ \\
        GCN-DW & $65.92 {\scriptstyle \pm 9.59}$ & $53.50 {\scriptstyle \pm 9.29}$ & $\mathbf{75.96} {\scriptstyle \pm 10.76}$ & $61.99 {\scriptstyle \pm 8.48}$ & $50.64 {\scriptstyle \pm 10.31}$ & $32.26 {\scriptstyle \pm 3.79}$ \\
        GCN-RF & $60.49 {\scriptstyle \pm 9.57}$ & $49.02 {\scriptstyle \pm 12.14}$ & $59.79 {\scriptstyle \pm 12.07}$ & $58.83 {\scriptstyle \pm 11.46}$ & $59.68 {\scriptstyle \pm 11.20}$ & $65.63 {\scriptstyle \pm 13.42}$ \\
        GCN-LAP & $63.75 {\scriptstyle \pm 8.76}$ & $64.01 {\scriptstyle \pm 8.85}$ & $74.99 {\scriptstyle \pm 8.55}$ & $60.41 {\scriptstyle \pm 9.69}$ & $49.12 {\scriptstyle \pm 7.91}$ & $30.90 {\scriptstyle \pm 4.64}$ \\
        BUDDY & $29.06 {\scriptstyle \pm 7.08}$ & $31.40 {\scriptstyle \pm 4.33}$ & $26.07 {\scriptstyle \pm 8.50}$ & $37.07 {\scriptstyle \pm 5.59}$ & $40.97 {\scriptstyle \pm 4.09}$ & $38.80 {\scriptstyle \pm 1.98}$ \\
        LINKX & $61.82 {\scriptstyle \pm 5.53}$ & $56.80 {\scriptstyle \pm 7.94}$ & $51.21 {\scriptstyle \pm 6.15}$ & $50.44 {\scriptstyle \pm 15.06}$ & $44.59 {\scriptstyle \pm 15.02}$ & $39.72 {\scriptstyle \pm 9.74}$ \\
        Proposed & $\mathbf{85.72} {\scriptstyle \pm 5.57}$ & $\mathbf{83.02} {\scriptstyle \pm 6.87}$ & $74.85 {\scriptstyle \pm 7.60}$ & $\mathbf{72.06} {\scriptstyle \pm 5.95}$ & $\mathbf{68.64} {\scriptstyle \pm 4.59}$ & $\mathbf{67.42} {\scriptstyle \pm 10.02}$ \\
        \bottomrule
    \end{tabular}}
\end{table}

\begin{table}[h]
    \centering
    \caption{\textsc{syn-citeseer}: Mean MRR and standard deviation for each method on synthetic datasets with varying \metrics ratios $\EAR\xspace$ (2/2, $\EAR\xspace \in [0.65, 0.94]$). The best method per column is bolded.}
    \label{tab:app-syn-citeseer-mrr-b}
    \resizebox{\linewidth}{!}{\scriptsize
    \begin{tabular}{lccccc}
        \toprule
         $\EAR\xspace$ & 0.65 & 0.77 & 0.81 & 0.88 & 0.94 \\
        \midrule
        GCN & $21.32 {\scriptstyle \pm 4.65}$ & $20.78 {\scriptstyle \pm 3.08}$ & $18.08 {\scriptstyle \pm 3.87}$ & $13.18 {\scriptstyle \pm 2.00}$ & $10.56 {\scriptstyle \pm 2.93}$ \\
        GAT & $21.16 {\scriptstyle \pm 4.89}$ & $19.10 {\scriptstyle \pm 4.48}$ & $14.75 {\scriptstyle \pm 2.74}$ & $13.61 {\scriptstyle \pm 5.15}$ & $10.22 {\scriptstyle \pm 1.75}$ \\
        GIN & $10.45 {\scriptstyle \pm 1.82}$ & $9.61 {\scriptstyle \pm 3.08}$ & $7.51 {\scriptstyle \pm 3.36}$ & $5.91 {\scriptstyle \pm 2.72}$ & $2.82 {\scriptstyle \pm 2.89}$ \\
        GraphSAGE & $10.61 {\scriptstyle \pm 5.20}$ & $10.56 {\scriptstyle \pm 5.06}$ & $10.34 {\scriptstyle \pm 3.02}$ & $6.89 {\scriptstyle \pm 2.12}$ & $5.16 {\scriptstyle \pm 2.18}$ \\
        MixHopGCN & $36.12 {\scriptstyle \pm 4.80}$ & $35.69 {\scriptstyle \pm 9.47}$ & $34.84 {\scriptstyle \pm 6.12}$ & $25.08 {\scriptstyle \pm 11.08}$ & $14.32 {\scriptstyle \pm 2.42}$ \\
        ChebGCN & $12.33 {\scriptstyle \pm 3.00}$ & $8.83 {\scriptstyle \pm 2.23}$ & $7.57 {\scriptstyle \pm 2.49}$ & $7.09 {\scriptstyle \pm 0.94}$ & $4.81 {\scriptstyle \pm 0.62}$ \\
        GCN-DW & $37.06 {\scriptstyle \pm 4.30}$ & $35.14 {\scriptstyle \pm 4.28}$ & $26.22 {\scriptstyle \pm 2.46}$ & $29.03 {\scriptstyle \pm 3.65}$ & $25.11 {\scriptstyle \pm 5.44}$ \\
        GCN-RF & $54.62 {\scriptstyle \pm 9.97}$ & $30.14 {\scriptstyle \pm 6.63}$ & $34.67 {\scriptstyle \pm 4.98}$ & $41.84 {\scriptstyle \pm 5.38}$ & $18.56 {\scriptstyle \pm 3.56}$ \\
        GCN-LAP & $37.15 {\scriptstyle \pm 6.31}$ & $38.19 {\scriptstyle \pm 6.80}$ & $24.42 {\scriptstyle \pm 2.50}$ & $26.08 {\scriptstyle \pm 4.46}$ & $23.36 {\scriptstyle \pm 4.97}$ \\
        BUDDY & $22.80 {\scriptstyle \pm 4.00}$ & $37.20 {\scriptstyle \pm 3.34}$ & $32.22 {\scriptstyle \pm 3.14}$ & $35.62 {\scriptstyle \pm 5.89}$ & $34.86 {\scriptstyle \pm 5.28}$ \\
        LINKX & $36.21 {\scriptstyle \pm 4.30}$ & $32.68 {\scriptstyle \pm 8.60}$ & $32.09 {\scriptstyle \pm 12.39}$ & $18.60 {\scriptstyle \pm 2.41}$ & $12.34 {\scriptstyle \pm 4.56}$ \\
        Proposed & $\mathbf{66.61} {\scriptstyle \pm 16.27}$ & $\mathbf{65.97} {\scriptstyle \pm 11.65}$ & $\mathbf{55.19} {\scriptstyle \pm 12.10}$ & $\mathbf{42.59} {\scriptstyle \pm 8.04}$ & $\mathbf{41.19} {\scriptstyle \pm 3.23}$ \\
        \bottomrule
    \end{tabular}}
\end{table}

\newpage

\section{Ablation Study in Details}
\label{sec:app:ablation_syn_details}
\begin{figure}[h]
    \centering
    \begin{subfigure}[h]{0.35\textwidth}
        \centering
        \includegraphics[width=\textwidth, clip]{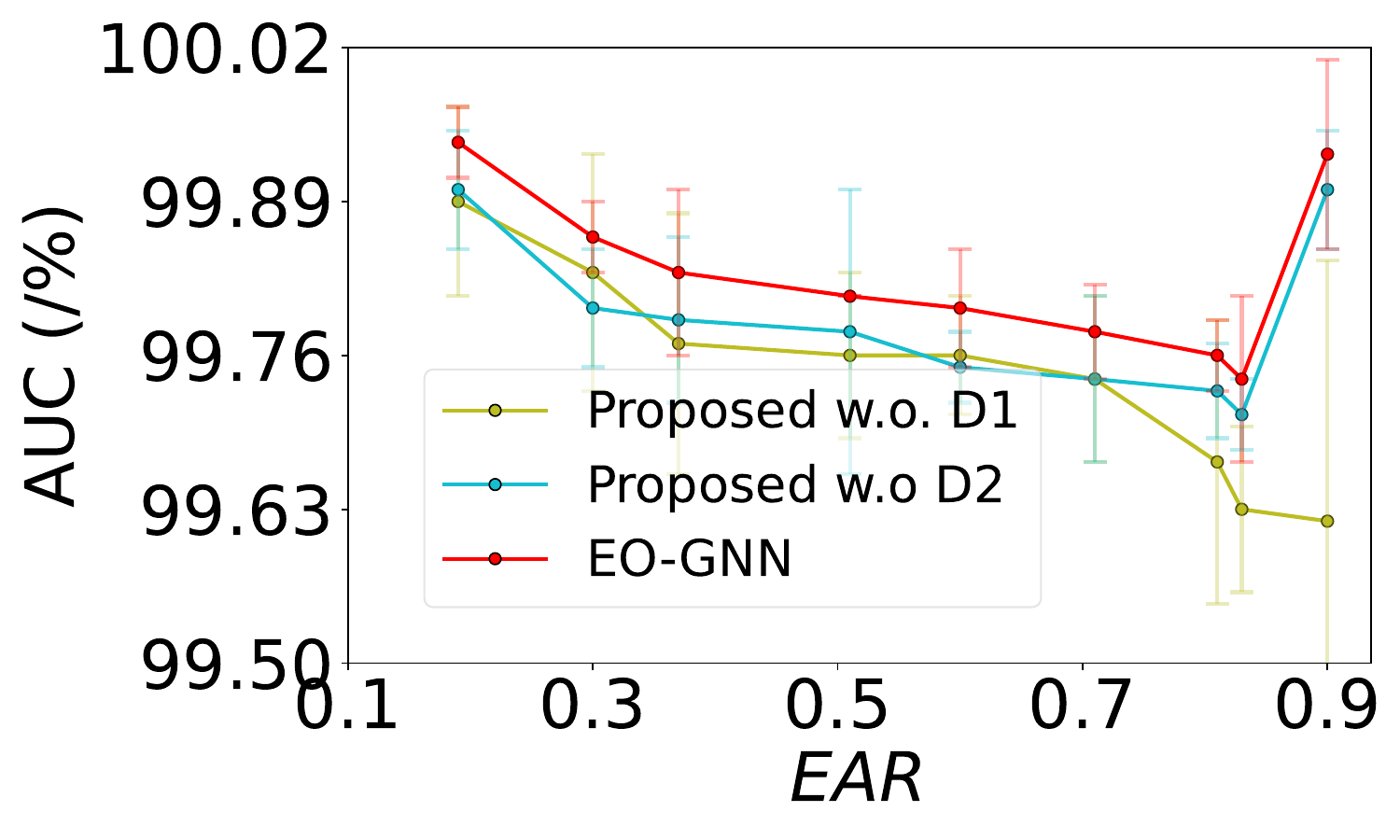}
        \caption{\scriptsize\texttt{syn-cora} (~\Cref{tab:app-syn-cora-results}). }
        \label{fig:app-syn-cora-auc} 
    \end{subfigure}
    \begin{subfigure}[h]{0.35\textwidth}
        \centering
        \includegraphics[width=\textwidth, clip]{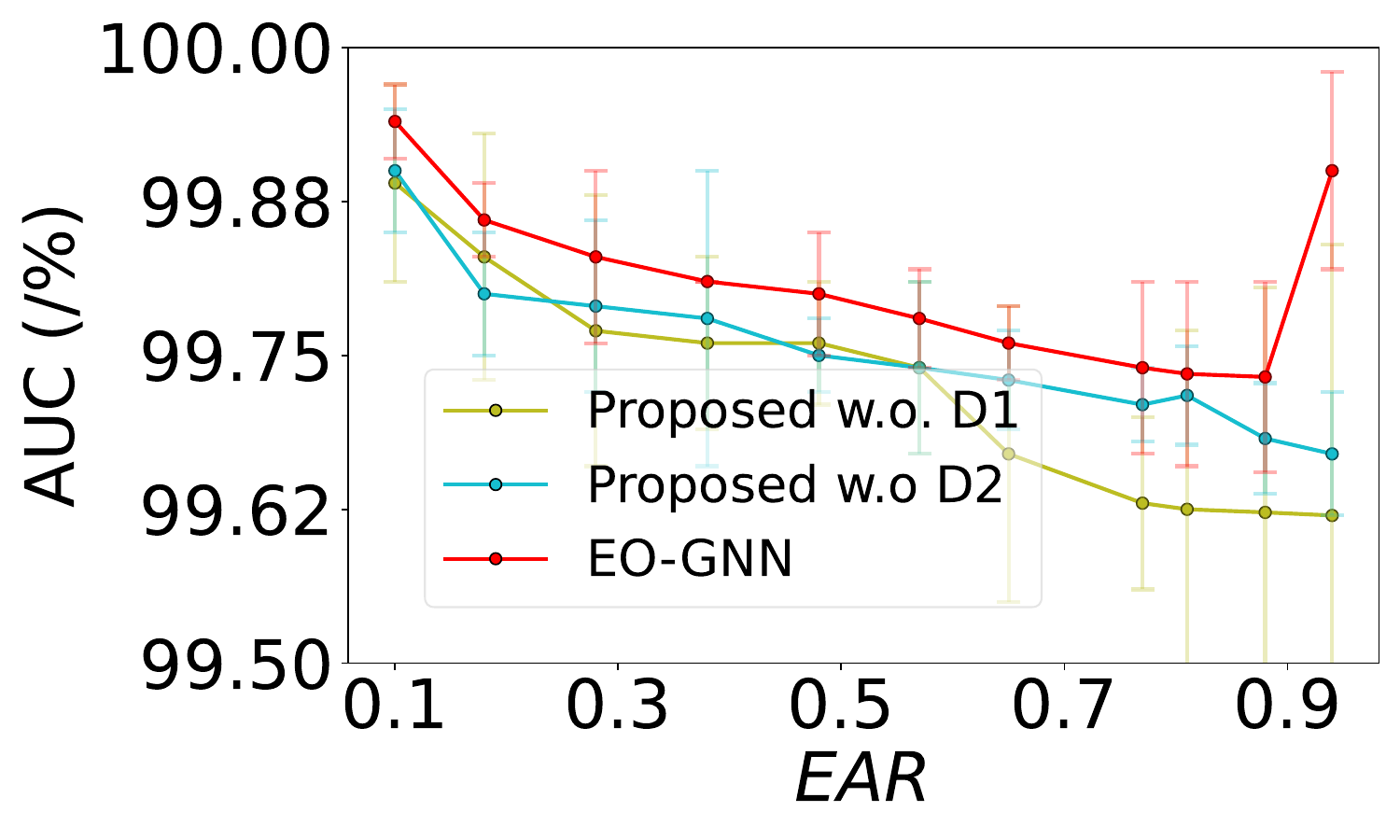}
        \caption{\scriptsize\texttt{syn-citeseer} (~\Cref{tab:app-syn-citeseer-results}).}
        \label{fig:app-syn-citeseer-auc}
    \end{subfigure}
    \begin{subfigure}[H]{0.35\textwidth}
        \centering
        \includegraphics[width=\textwidth, clip]{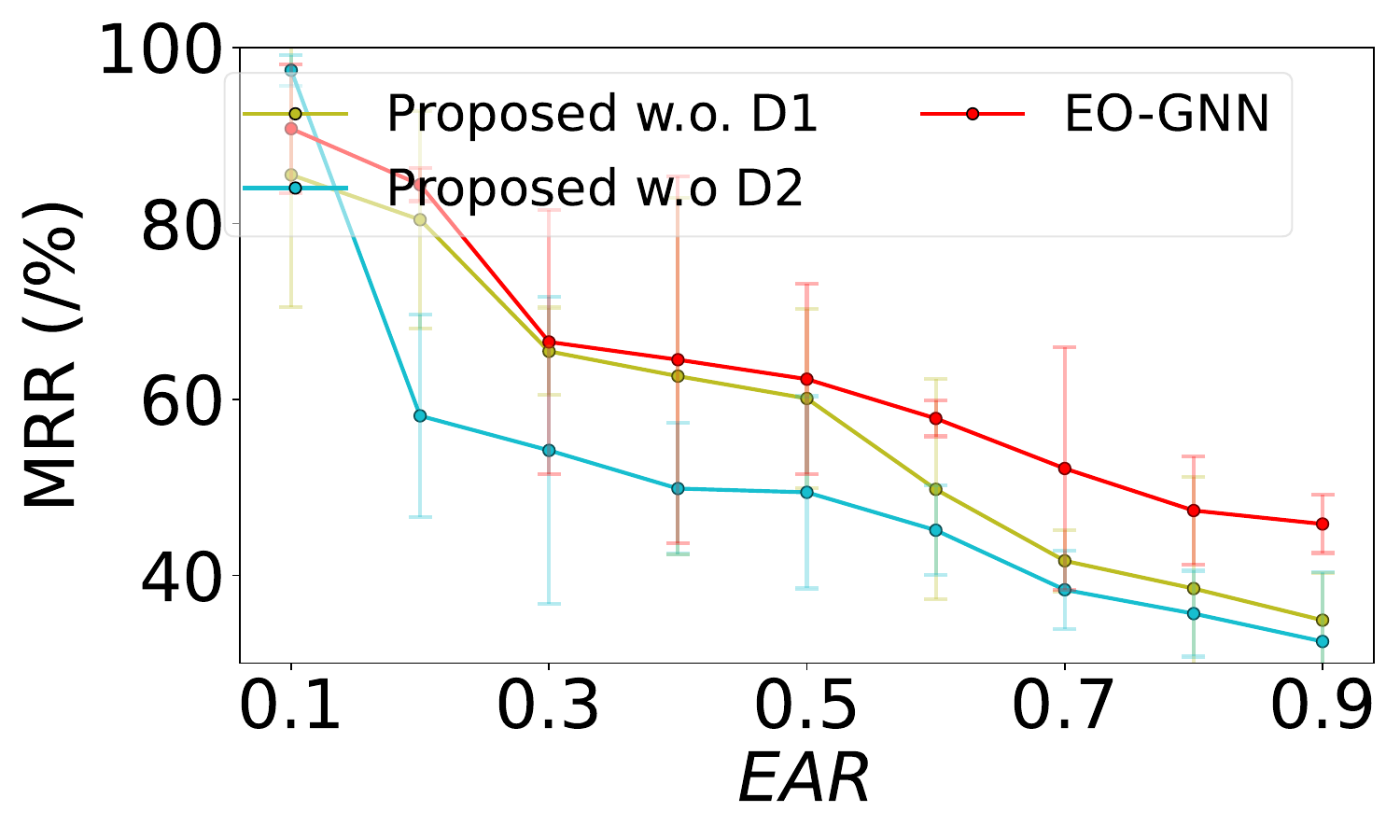}
        \caption{\scriptsize\texttt{syn-cora} (~\Cref{tab:app-syn-cora-mrr}).}
        \label{fig:app-syn-cora-mrr}
    \end{subfigure}
    \caption{Significance of Design Choices D1-D2 via ablation studies. Under high automorphism, D1 exhibits a larger performance gain compared to low automorphism, whereas D2 demonstrates greater robustness across varying levels of automorphism.}
    \label{fig:app_ablation-syn-results}
\end{figure}

\begin{table}[h]
    \centering
    \caption{Ablation study showing the effect of different components. \textcolor{red}{Mean metrics (MRR) and standard deviation} for each method on synthetic Cora with varying \metrics ratios $\EAR\xspace$. }
    \label{tab:ablation_1}
    \resizebox{\linewidth}{!}{\scriptsize
    \begin{tabular}{lccccccccc}
        \toprule
        Method & 0.19 & 0.30 & 0.37 & 0.51 & 0.60 & 0.71 & 0.81 & 0.83 & 0.90 \\
        \midrule
        Proposed w.o. D1   & $99.89 {\scriptstyle \pm 0.08}$ & $99.83 {\scriptstyle \pm 0.10}$ & $99.77 {\scriptstyle \pm 0.11}$ & $99.76 {\scriptstyle \pm 0.07}$ & $99.76 {\scriptstyle \pm 0.05}$ & $99.74 {\scriptstyle \pm 0.07}$ & $99.67 {\scriptstyle \pm 0.12}$ & $99.63 {\scriptstyle \pm 0.07}$ & $99.62 {\scriptstyle \pm 0.22}$ \\
        Proposed w.o D2    & \textcolor{red}{$99.90 {\scriptstyle \pm 0.05}$} & $99.80 {\scriptstyle \pm 0.05}$ & $99.79 {\scriptstyle \pm 0.07}$ & $99.78 {\scriptstyle \pm 0.12}$ & $99.75 {\scriptstyle \pm 0.03}$ & $99.74 {\scriptstyle \pm 0.07}$ & $99.73 {\scriptstyle \pm 0.04}$ & $99.71 {\scriptstyle \pm 0.03}$ & \textcolor{red}{$99.90 {\scriptstyle \pm 0.05}$} \\
        Proposed          & $99.94 {\scriptstyle \pm 0.03}$ & $99.86 {\scriptstyle \pm 0.03}$ & $99.83 {\scriptstyle \pm 0.07}$ & $99.81 {\scriptstyle \pm 0.00}$ & $99.80 {\scriptstyle \pm 0.05}$ & $99.78 {\scriptstyle \pm 0.04}$ & $99.76 {\scriptstyle \pm 0.03}$ & $99.74 {\scriptstyle \pm 0.07}$ & $99.93 {\scriptstyle \pm 0.08}$ \\
        \bottomrule
    \end{tabular}}
\end{table}

\begin{table}[h]
    \centering
    \caption{Ablation study showing the effect of different components. \textcolor{red}{Mean metrics (MRR) and standard deviation} for each method on synthetic Cora with varying \metrics ratios $\EAR\xspace$. }
    \label{tab:ablation_2}
    \resizebox{\linewidth}{!}{\scriptsize
    \begin{tabular}{lccccccccc}
        \toprule
        Method & 0.19 & 0.30 & 0.37 & 0.51 & 0.60 & 0.71 & 0.81 & 0.83 & 0.90 \\
        \midrule
         Proposed w.o. D1 & $85.53 {\scriptstyle \pm 14.99}$ & $80.43 {\scriptstyle \pm 12.34}$ & $65.46 {\scriptstyle \pm 4.98}$ & $62.64 {\scriptstyle \pm 20.27}$ & $60.10 {\scriptstyle \pm 10.21}$ & $49.77 {\scriptstyle \pm 12.52}$ & $41.64 {\scriptstyle \pm 3.52}$ & $38.49 {\scriptstyle \pm 12.70}$ & $34.89 {\scriptstyle \pm 5.32}$ \\
        Proposed w.o D2 & $97.42 {\scriptstyle \pm 1.74}$ & $58.13 {\scriptstyle \pm 11.50}$ & $54.20 {\scriptstyle \pm 17.44}$ & $49.86 {\scriptstyle \pm 7.44}$ & $49.43 {\scriptstyle \pm 10.93}$ & $45.12 {\scriptstyle \pm 5.11}$ & $38.35 {\scriptstyle \pm 4.47}$ & $35.64 {\scriptstyle \pm 4.88}$ & $32.47 {\scriptstyle \pm 7.89}$ \\
        Proposed & $90.78 {\scriptstyle \pm 7.34}$ & $84.41 {\scriptstyle \pm 1.87}$ & $66.53 {\scriptstyle \pm 14.99}$ & $64.51 {\scriptstyle \pm 20.85}$ & $62.30 {\scriptstyle \pm 10.83}$ & $57.83 {\scriptstyle \pm 2.04}$ & $52.14 {\scriptstyle \pm 13.77}$ & $47.36 {\scriptstyle \pm 6.13}$ & $45.83 {\scriptstyle \pm 3.30}$ \\
        \bottomrule
    \end{tabular}}
\end{table}

\begin{table}[h]
    \centering
    \caption{Ablation study showing the effect of different components. \textcolor{red}{Mean metrics (MRR) and standard deviation} for each method on synthetic Citeseer with varying  \metrics ratios $\EAR\xspace$. }
    \label{tab:ablation_3}
    \resizebox{\linewidth}{!}{\scriptsize
    \begin{tabular}{lcccccc}
        \toprule
         $\EAR\xspace$ & 0.10 & 0.18 & 0.28 & 0.38 & 0.48 & 0.57 \\
        \midrule
        Proposed w.o. D1  & $99.89 {\scriptstyle \pm 0.08}$ & $99.83 {\scriptstyle \pm 0.10}$ & $99.77 {\scriptstyle \pm 0.11}$ & $99.76 {\scriptstyle \pm 0.07}$ & $99.76 {\scriptstyle \pm 0.05}$ & $99.74 {\scriptstyle \pm 0.07}$ \\
        Proposed w.o D2 & $99.90 {\scriptstyle \pm 0.05}$ & $99.80 {\scriptstyle \pm 0.05}$ & $99.79 {\scriptstyle \pm 0.07}$ & $99.78 {\scriptstyle \pm 0.12}$ & $99.75 {\scriptstyle \pm 0.03}$ & $99.74 {\scriptstyle \pm 0.07}$ \\
        Proposed & $99.94 {\scriptstyle \pm 0.03}$ & $99.86 {\scriptstyle \pm 0.03}$ & $99.83 {\scriptstyle \pm 0.07}$ & $99.81 {\scriptstyle \pm 0.00}$ & $99.80 {\scriptstyle \pm 0.05}$ & $99.78 {\scriptstyle \pm 0.04}$ \\
        \bottomrule
    \end{tabular}}
    \vspace{0.3cm}
    \resizebox{\linewidth}{!}{\scriptsize
    \begin{tabular}{lccccc}
        \toprule
         $\EAR\xspace$ & 0.65 & 0.77 & 0.81 & 0.88 & 0.94 \\
        \midrule
        Proposed w.o. D1  & $99.67 {\scriptstyle \pm 0.12}$ & $99.63 {\scriptstyle \pm 0.07}$ & $99.63 {\scriptstyle \pm 0.15}$ & $99.63 {\scriptstyle \pm 0.18}$ & $99.62 {\scriptstyle \pm 0.22}$ \\
        Proposed w.o D2 & $99.73 {\scriptstyle \pm 0.04}$ & $99.71 {\scriptstyle \pm 0.03}$ & $99.75 {\scriptstyle \pm 0.04}$ & $99.84 {\scriptstyle \pm 0.05}$& $99.73 {\scriptstyle \pm 0.05}$ \\
        Proposed & $99.76 {\scriptstyle \pm 0.03}$ & $99.74 {\scriptstyle \pm 0.07}$ & $99.74 {\scriptstyle \pm 0.08}$& $99.73 {\scriptstyle \pm 0.08}$& $99.90 {\scriptstyle \pm 0.08}$ \\
        \bottomrule
    \end{tabular}}
\end{table}

\begin{table}[h]
    \centering
    \caption{Ablation study showing the effect of different components. \textcolor{red}{Mean MRR and standard deviation} for each method on synthetic Citeseer with varying  \metrics ratios $\EAR\xspace$. }
    \label{tab:ablation_4}
    \resizebox{\linewidth}{!}{\scriptsize
    \begin{tabular}{lcccccc}
        \toprule
         $\EAR\xspace$ & 0.10 & 0.18 & 0.28 & 0.38 & 0.48 & 0.57 \\
        \midrule
        Proposed w.o. D1  & $83.40 {\scriptstyle \pm 11.40}$ & $71.65 {\scriptstyle \pm 12.54}$ & $69.35 {\scriptstyle \pm 12.85}$ & $67.94 {\scriptstyle \pm 12.46}$ & $64.67 {\scriptstyle \pm 8.23}$ & $64.18 {\scriptstyle \pm 7.58}$ \\
        Proposed w.o D2   & $79.99 {\scriptstyle \pm 12.23}$ & $76.28 {\scriptstyle \pm 7.76}$ & $72.20 {\scriptstyle \pm 16.27}$ & $66.62 {\scriptstyle \pm 12.13}$& $66.62 {\scriptstyle \pm 12.13}$& $66.62 {\scriptstyle \pm 12.13}$\\
        Proposed & $85.72 {\scriptstyle \pm 5.57}$ & $83.02 {\scriptstyle \pm 6.87}$ & $74.85 {\scriptstyle \pm 7.60}$ & $72.06 {\scriptstyle \pm 5.95}$ & $68.64 {\scriptstyle \pm 4.59}$ & $67.42 {\scriptstyle \pm 10.02}$ \\
        \bottomrule
    \end{tabular}}
    \vspace{0.3cm}
    \resizebox{\linewidth}{!}{\scriptsize
    \begin{tabular}{lccccc}
        \toprule
         $\EAR\xspace$ & 0.65 & 0.77 & 0.81 & 0.88 & 0.94 \\
        \midrule
        Proposed w.o. D1  & $63.90 {\scriptstyle \pm 9.91}$ & $43.61 {\scriptstyle \pm 5.35}$ & $42.21 {\scriptstyle \pm 3.54}$ & $41.19 {\scriptstyle \pm 3.23}$ & $36.42 {\scriptstyle \pm 9.45}$ \\
        Proposed w.o D2   & $58.34 {\scriptstyle \pm 13.80}$ & $55.88 {\scriptstyle \pm 12.07}$ & $49.56 {\scriptstyle \pm 10.00}$ & $41.81 {\scriptstyle \pm 7.54}$ & $38.42 {\scriptstyle \pm 12.32}$ \\
        Proposed & $66.61 {\scriptstyle \pm 16.27}$ & $65.97 {\scriptstyle \pm 11.65}$ & $55.19 {\scriptstyle \pm 12.10}$ & $42.59 {\scriptstyle \pm 8.04}$ & $41.19 {\scriptstyle \pm 3.23}$ \\
        \bottomrule
    \end{tabular}}
\end{table}

\newpage

\section{Empirical Setup \& Hyperparameter Tuning}
\label{app:tuning}

\subsection{Setup} 
\paragraph{\method Implementation} 
For loss function, we calculate the contrastive learning loss between the predicted 
and the ground-truth edge connections. 
\begin{equation}
    \label{equ:loss}
    L(v) = - \sum_{v \in N(v)} \log \delta(\mathbf{r}_v, \mathbf{r}_u) - \sum_{w \notin N(v)} \log (1 - \delta(\mathbf{r}_v, \mathbf{r}_w)).
\end{equation}

\paragraph{Baseline Implementations} For all baselines besides MLP, we used the implementation from the public library pytorch geometric \cite{fey2019fastgraphrepresentationlearning}. 
\begin{itemize}
    \item \textbf{GCN}~\citep{Kipf2016SemiSupervisedCW}: \url{https://github.com/tkipf/gcn}
    \item \textbf{GraphSAGE}~\citep{you2021graphcontrastivelearningaugmentations}: \url{https://github.com/williamleif/graphsage-simple} (PyTorch implementation)
    \item \textbf{GIN} \citep{Xu2018HowPA}: \url{https://pytorch-geometric.readthedocs.io/en/latest/_modules/torch_geometric/nn/conv/gin_conv.html#GINConv}
    \item \textbf{GAT}~\citep{Velickovic2017GraphAN}: \url{https://github.com/PetarV-/GAT}. 
\end{itemize}

For \decoder, we used our own implementation of MLP with 1 to 3-hidden layers. We use the same loss function as \method for training MLP. 

\paragraph{Hardware Specifications} We run experiments on synthetic benchmarks using the HoreKa Green cluster, which consists of nodes equipped with dual-socket Intel Xeon Platinum 8368 CPUs, each with 76 cores (152 threads per node) and 512 GB of main memory. Each node features four NVIDIA A100 GPUs, each with 40 GB of GPU memory.  

\paragraph{Dataset Statistic}
Random splits allocate \textcolor{red}{80\%/15\%/5\%} of the edges for the training, validation and test sets, respectively. The Collab dataset permits the use of validation edges as input during testing. For all reported results in the tables, we perform target link removal during training and do not utilize validation edges in the training stage.

\subsection{Hyperparameter Tuning} 
\label{subsec:app-choice-metrics}
To avoid bias, we tuned the hyperparameters of each method (\method and baseline models) on each benchmark. 
Below we list the hyperparameters tested on each benchmark per model. 
As the hyperparameters defined by each baseline model differ significantly, we list the combinations of non-default command line arguments we tested, without explaining them in detail. %
For the adam optimizer, we perform a grid search over the
parameters reported in the original source code and tune the Adam optimizer
through the following ranges: learning rate ($10^{-2, -3, -4}$), dropout rate
(0, 0.1, 0.2, 0.3) and batch size ($2^{5, 7, \dots, 15}$). \\

\paragraph{Synthetic Benchmark Tuning} For \texttt{syn-cora}, we test the following command-line arguments for each baseline method: 
\begin{itemize}
    \item \textbf{\method-1 \& \method-2}: 
    \begin{itemize}
        \item Dimension of Feature Embedding $p$: 64
        \item Non-linearity Function $\rho$: ReLU
        \item Dropout Rate: $a\in \{0, 0.5\}$
    \end{itemize}
    We report the best performance, for $a=0$. 
    \item \textbf{GCN}~\cite{Kipf2016SemiSupervisedCW}: 
    \begin{itemize}
        \item \texttt{\lstinline{hidden1}}: $a\in \{16, 32, 64\}$
        \item \texttt{\lstinline{early_stopping}}: $b\in\{40, 100, 200\}$
        \item \texttt{\lstinline{epochs}}: 2000
    \end{itemize}
    We report the best performance, for $a=32, b=40$.
    \item \textbf{GraphSAGE}~\cite{Hamilton2017InductiveRL}: 
    \begin{itemize}
        \item \texttt{\lstinline{hid_units}}: $a\in \{64, 128\}$
        \item \texttt{\lstinline{lr}}: $b\in \{0.1, 0.7\}$
         \item \texttt{\lstinline{layers}}: $l\in \{1, 3\}$
         \item \texttt{\lstinline{dropout}}: $l\in \{1, 3\}$
        \item \texttt{\lstinline{epochs}}: 500
    \end{itemize}
    We report the performance with $layers=3, dropout=0.5, epochs=800$.
    
    \item \textbf{GAT}~\cite{Velickovic2017GraphAN}: 
    \begin{itemize}
        \item \texttt{\lstinline{hid_units}}: $a \in \{8, 16, 32, 64\}$
        \item \texttt{\lstinline{n_heads}}: $b\in\{1, 4, 8\}$ 
        \item \texttt{\lstinline{n_layers}}: $b\in\{1, 2, 3\}$ 
        \item \texttt{\lstinline{dropout}}: $l\in \{1, 3\}$
        \item \texttt{\lstinline{epochs}}: 500
    \end{itemize}
    We report the performance with $a=8, b=8$.
    \item \textbf{\decoder}
    \begin{itemize}
        \item Dimension of Feature Embedding $p$: 64
        \item Non-linearity Function $\rho$: ReLU
        \item Dropout Rate: 0.5
    \end{itemize}
\end{itemize}

\newpage

\section{WL Labeling Scheme with Perfect HASH}
\label{app:algo}
Consider a graph $\graph = (\vertexSet, \edgeSet)$. Let $\ell^{(0)}(v) = \ell(v)$ be the initial label assigned to each node $v \in \vertexSet$ (e.g., based on node features or degree). Let $H$ denote the number of WL iterations. Then, for each iteration $h = 1, \dots, H$, the WL labeling scheme updates node labels recursively by considering the multiset of neighboring labels from the previous iteration.

Formally, define the multiset of neighbors' labels for node $v$ at iteration $h$ as:
\[
\mathcal{N}^{(h)}(v) = \left\{ \ell^{(h)}(u) \mid u \in \neighNoSelfLoop(v) \right\}
\]
The updated label at iteration $h+1$ is computed as:
\begin{equation}
\ell^{(h+1)}(v) = \texttt{hash}\left( \ell^{(h)}(v), \mathcal{N}^{(h)}(v) \right)
\label{eq:wl-update}
\end{equation}
Following the original formulation~\citep{weisfeiler1968reduction}, we use a perfect hashing function, ensuring that two nodes receive the same label at iteration $h+1$ if and only if their own label and the multiset of their neighbors’ labels were identical at iteration $h$.

\begin{algorithm}[h]
\caption{\textcolor{royalgreen}{Weisfeiler-Lehman Forward Pass}}
\label{algo:wl_forward}
\SetNoFillComment
\textcolor{royalgreen}{\footnotesize
\SetKwInput{Input}{Input}\SetKwInput{Output}{Output}
\SetKwInput{HyperParams}{Hyper-parameters}\SetKwInput{TrainParams}{Train Parameters}
\Input{
Node feature matrix $\mathbf{X} = \mathbf{I}\in  \mathbb{Z}^{|\mathcal{V}| \times 1}$ \\
Adjacency matrix $\mathbf{A}$ \\
Hash function $\mathcal{H}$
}
\TrainParams{
Empty
}
\Output{
Updated node hash labels $h \in \mathbb{Z}^{|\mathcal{V}|}$
}
\Begin{
\tcc{Stage S1: Preprocessing}
\If{$\mathbf{A}$ is sparse}{
    Convert to $(s, d)$ format\;
}
\Else{
    Sort $\mathbf{A}$ by destination node\;
}
Initialize $\mathcal{N}\leftarrow$ $\text{EmptyList}_N$\;
\tcc{Stage S2: Neighbor Aggregation}
\ForEach{$(s, d) \in \mathbf{A}$}{
    Append $\mathbf{X} [s]$ to $\mathcal{N}[d]$\;
}
\tcc{Stage S3: Weisfeiler-Lehman Hash Computation}
Initialize $O \leftarrow \text{EmptyList}_|\mathcal{V}|$\;
\ForEach{$i \in [0, |\mathcal{V}|)$}{
    Compute unique hash $i$ using:\;
    \quad - Node feature $\mathbf{X}[i]$\;
    \quad - Sorted features of $\mathcal{N}[i]$\;
    \If{$ i \notin \mathcal{H}$}{
        Assign a new unique ID to $i$\;
    }
    Store hashed value in $O[i]$\;
}
\tcc{Stage S4: Edge Orbit Hash Computation}
\ForEach{$(u,v)\in\mathcal{E}$}{
    Compute edge hash using the multiset $\{O[u], O[v]\}$\;
    Store hashed edge orbit ID in $\mathcal{O}_{\mathcal{E}}[(u,v)]$\;
}
\Output{
Updated node hash labels $O \in \mathbb{Z}^{|\mathcal{V}|}$ \\
Edge orbit labels $\mathcal{O}_{\mathcal{E}}$
}
}}
\end{algorithm}

\begin{figure*}[h]
    \centering
    \resizebox{\textwidth}{!}{%
    \begin{minipage}{\textwidth}
        \centering
        \begin{subfigure}{0.3\textwidth}
            \centering
            \includegraphics[width=\linewidth]{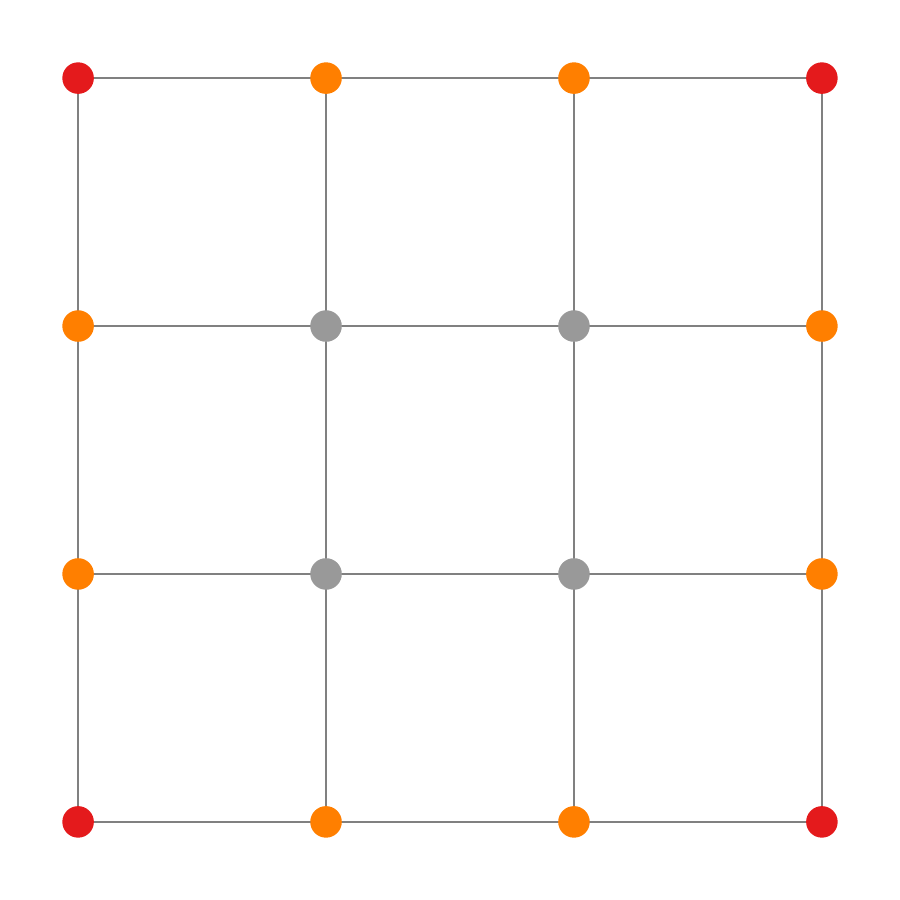}
            \caption{Square Grid Graph with 
            $\EAR\xspace$=1}
            \label{fig:square_grid_4}
        \end{subfigure}
        \hspace{0.4cm}
        \begin{subfigure}{0.3\textwidth}
            \centering
            \includegraphics[width=\linewidth]{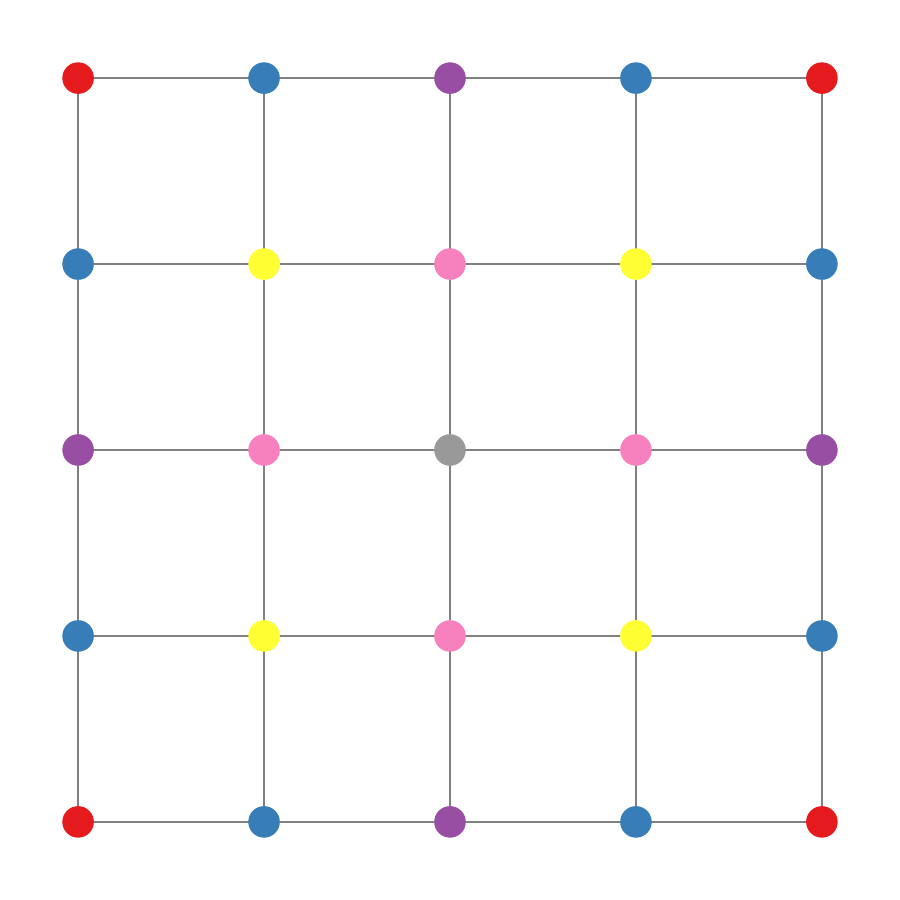}
            \caption{Square Grid Graph with 
            $\EAR\xspace=1$}
            \label{fig:square_grid_5}
        \end{subfigure}
        
        \begin{subfigure}{0.3\textwidth}
            \centering
            \includegraphics[width=\linewidth]{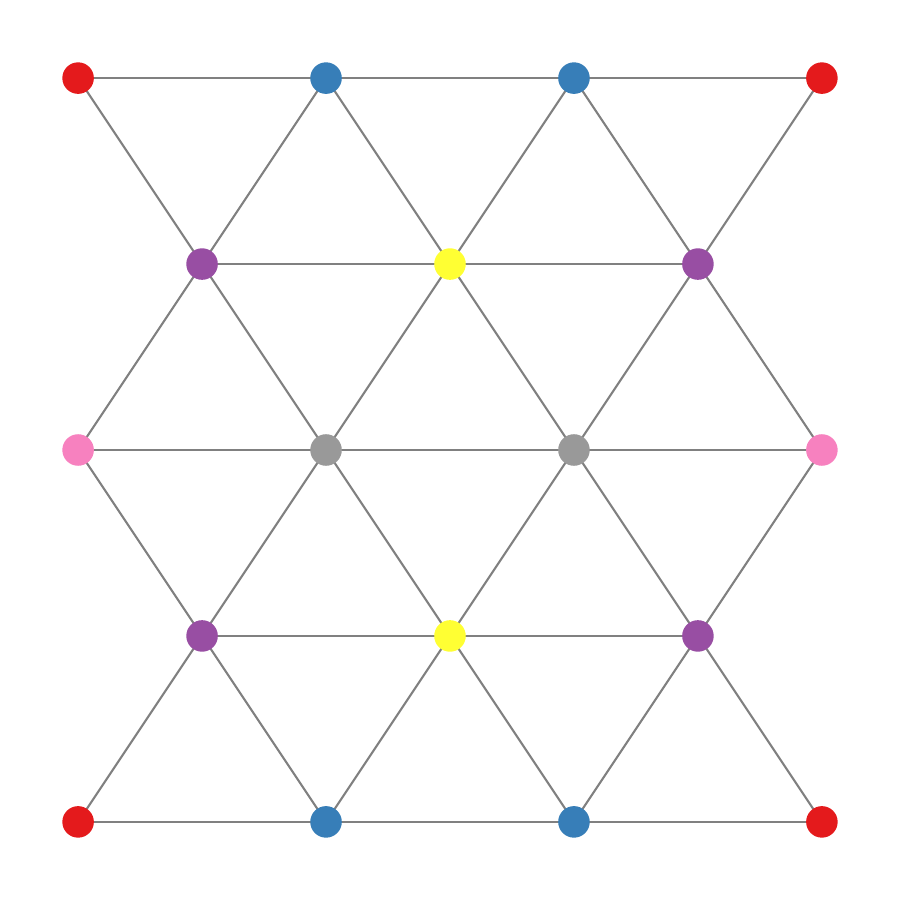}
            \caption{Triangular Graph with 
            $\EAR\xspace$=1}
            \label{fig:triangular_20}
        \end{subfigure}
        \hspace{0.4cm}
        \begin{subfigure}{0.3\textwidth}
            \centering
            \includegraphics[width=\linewidth]{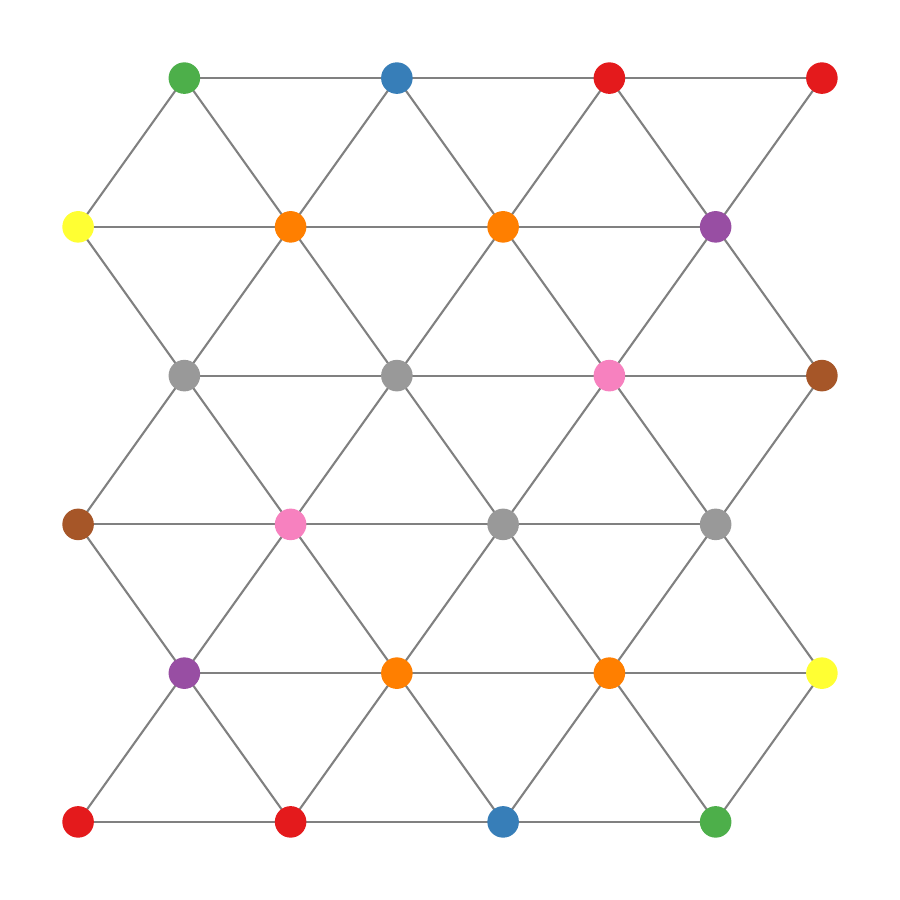}
            \caption{Triangular Graph with 
            $\EAR\xspace=1$}
            \label{fig:triangular_30}
        \end{subfigure}

        \begin{subfigure}{0.3\textwidth}
            \centering
            \includegraphics[width=\linewidth]{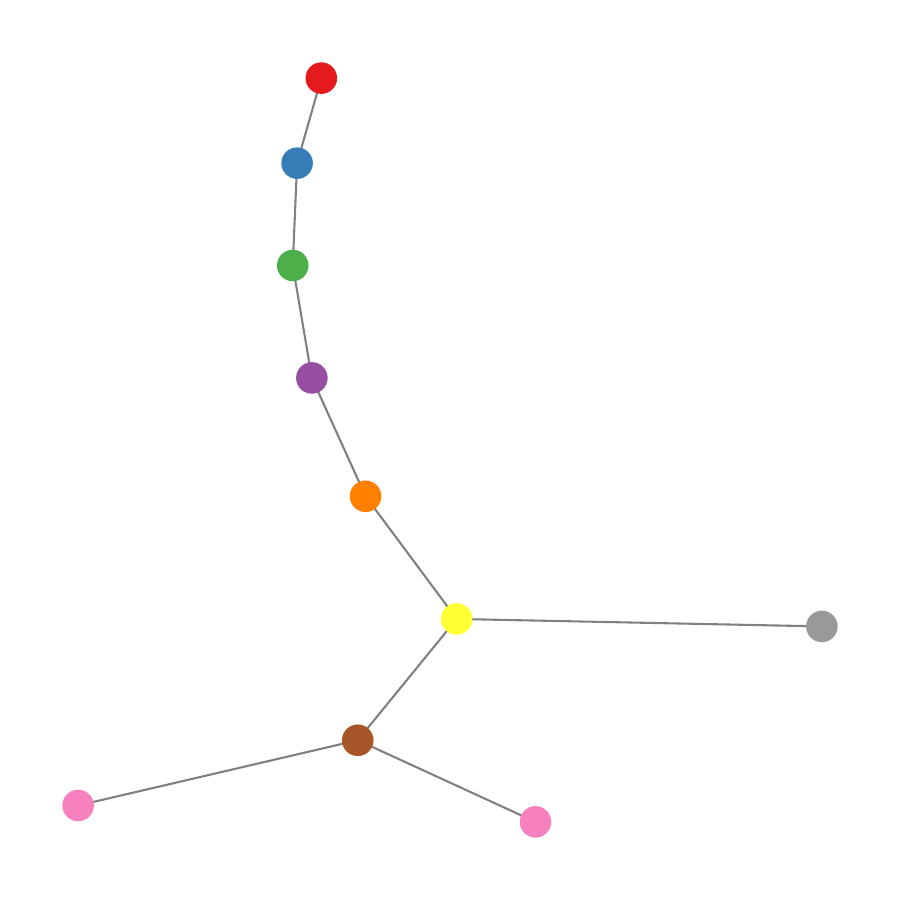}
            \caption{Tree Graph with 
            $\EAR\xspace$=2/9}
            \label{fig:tree_10}
        \end{subfigure}
        \hspace{0.4cm}
        \begin{subfigure}{0.3\textwidth}
            \centering
            \includegraphics[width=\linewidth]{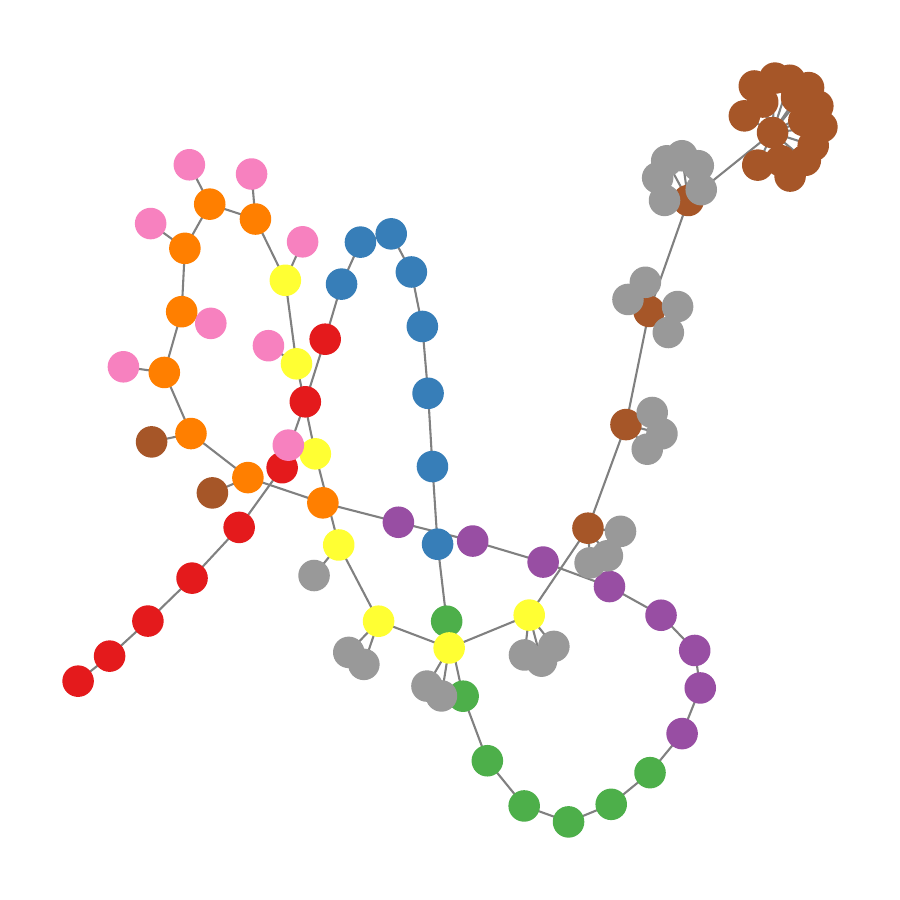}
            \caption{Tree Graph with 
            $\EAR\xspace$=0.53}
            \label{fig:tree_100}
        \end{subfigure}
        
    \end{minipage}
    } 
    \caption{Illustration of the automorphism test approximated by ~\Cref{algo:wl_forward}.}
        \label{fig:automorphic-test-visualization}
\end{figure*}

\newpage

\section{\method: Time Complexity in Detail}
\label{app:time complexity}

\begin{table*}[h!]
    \centering
    \caption{Statistics of standard benchmark graphs }
    {\small
    \begin{adjustbox}{width=0.8\textwidth}
    \begin{tabular}{lccccccc}
        \toprule
        & \textbf{Cora} & \textbf{Citeseer} & \textbf{Pubmed} & \textbf{Collab} & \textbf{PPA} & \textbf{Citation2} & \textbf{DDI} \\
        \midrule
       \textbf{Split Ratio} & \textcolor{red}{80/15/5} & \textcolor{red}{80/15/5} & \textcolor{red}{80/15/5} &92/4/4 & 70/20/10 & 98/1/1 & 80/10/10 \\
        \textbf{Split Scheme} & R & R & R & Time & Throughput & Time & Protein target \\
        \midrule
        \textbf{\#Nodes $|\mathcal{V}|$} & 2708 & 3327 & 19716 & 235868 & 576289 & 2927963 & 4267 \\
        \textbf{\#Edges $|\mathcal{E}|$} & 7392 & 6374 & 62056 & 1935264 & 42463862 & 60703760 & 2135822 \\
        \textbf{Avg Deg (G)} & 7.85 & 3.62 & 11.55 & 24.86 & 149.17 & 89.75 & \textbf{804.51} \\
        \textbf{Avg Deg (G2)} & 5.63 & 3.28 & 7.85 & 20.70 & 133.13 & 68.21 & \textbf{796} \\
        \textbf{Clustering} & 0.12 & 0.07 & 0.03 & \textbf{0.73} & 0.22 & 0.18 & 0.51 \\
        \textbf{Transitivity} & 0.06 & 0.09 & 0.04 & 0.36 & 0.22 & 0.06 & \textbf{0.47} \\
        \textbf{Deg Gini} & 0.45 & 0.50 & \textbf{0.63} & 0.55 & 0.55 & 0.57 & 0.47 \\
        \textbf{Coreness Gini} & 0.27 & 0.36 & 0.45 & 0.43 & \textbf{0.45} & 0.42 & 0.35 \\
        \textbf{Heterogeneity} & 0.14 & 0.11 & 0.23 & 0.21 & 0.13 & \textbf{0.31} & -0.08 \\
        \textbf{Power Law $\alpha$} & \textbf{1.89} & \textbf{2.13} & 1.97 & 1.89 & 1.33 & 1.39 & 1.21 \\
        \bottomrule
    \end{tabular}
    \end{adjustbox}
    }
    \label{tab:graph-stats}
\end{table*}

The 1-dimensional Weisfeiler-Lehman (1-WL) algorithm with $h$ iterations for precomputing subtree-based WL labels incurs a complexity of $\mathcal{O}(|\mathcal{E}| \cdot K)$~\cite{togninalli2019wasserstein}. The D1 augmentation, which involves randomly dropping edges and selecting subgraph orbits, introduces no computational overhead and is therefore considered to have $\mathcal{O}(1)$ complexity. 

For feature encoding, we apply a GNN with $L$ layers, each involving a dense transformation and neighborhood aggregation. This results in a total embedding complexity of $\mathcal{O}\left(L \cdot (|\mathcal{V}| \cdot p^2 + |\mathcal{E}| \cdot p)\right)$, where $p$ denotes the feature dimensionality. 
For each node with a performed linear transformation $\mathbf{H}^l\mathbf{W}^l$, where $\mathbf{H}^l \in \mathbb{R}^{\vert \mathcal{V}\vert \times p}$, $\mathbf{W}^l \in \mathbb{R}^{p \times p}$ refer to node and weight matrix in $l$ layer. The complexity is $\mathcal{O}\left(\vert \mathcal{V}  \vert \cdot p^2 \right)$. The edge wise aggregation is $\mathcal{O}\left(\vert \mathcal{E}  \vert \cdot p \right)$.
Combining both components, the total computational complexity of \method is:
\[
\mathcal{O}\left(|\mathcal{E}| \cdot K + L \cdot (|\mathcal{V}| \cdot p^2 + |\mathcal{E}| \cdot p)\right).
\]


\section{Real Datasets: Details}
\label{app:real}

\label{subsec:app-dataset_statistics}
Considerable work has demonstrated that local and global structural characterization's are more effective for LP. To translate the homophily assumption, local and global graph heuristics, small-world phenomenon and scale-free network properties into task-specific statistics, we provide the following graph metrics. 
\begin{enumerate}[left=0pt, labelsep=0.5em, itemsep=0em]
    \item \textbf{Graph Density}: Number of Nodes, Edges, Arithmetic Deg are used to measure the graph's size, density and sparsity. Average degree of each central node $v \in \mathcal{V}$ and its of 2-order neighborhood $\mathcal{N}_{v}$' average degree measures the graph's local connectivity.
    \item \textbf{Graph Locality}: We utilize two metrics to quantify the locality of one graph. 
    \textit{Transitivity}: Transitivity measures the fraction of all possible triangles in the graph. It quantifies the likelihood that if two nodes are connected to a common node, they will also be connected to each other. The formula for transitivity is given by:
    \begin{equation}
    T = \frac{3 \times \# \text{triangles}}{\# \text{triads}}
    \end{equation}
    where, the numerator represents the number of triangles in the graph; the denominator represents the number of possible triads (sets of three nodes that are connected by at least two edges). Transitivity gives an overall measure of how many triangles (closed 3-node subgraphs) exist relative to the total number of possible connections between three nodes in the graph.
    \textit{Average Clustering Coefficient}: It measures the fraction of possible triangles through that node that actually exist. It can be computed for a node $i$ as:
    \begin{equation}
        C_i = \frac{2 \times T(i)}{\text{deg}(i)(\text{deg}(i) - 1)}
    \end{equation}
    $T(i)$ is the number of triangles through node $i$ and $\text{deg}(i)$ is the degree of node $i$. The \textit{average clustering coefficient} is simply the average value of $C_i$ for all nodes in the graph. It gives a measure of how close the graph is to a complete clique, i.e., how often neighbors of a node are connected to each other.
    \item \textbf{Hierarchical level}: We leverage k-Core graph's fraction and degree distribution to calculate Gini and Coreness Gini. 
    \item \textbf{Scale-free}: If its node degree distribution $P(d)$ follows a power law $P(d) \sim d^{-\gamma}$, where $\gamma$ typically lies within the range $2 < \gamma < 3$. We approximate power law $\alpha$ based on the following estimator. \text{Citeseer} is scale-free networks.
    \begin{equation}
        \hat{\alpha} = 1 + N\left( \sum_{i=1}^{n} \log\left(\frac{d_i + 1}{d_{\min} + 1}\right) \right)^{-1}
\end{equation}
\end{enumerate}

\end{document}